\documentclass{article}

 \usepackage[preprint,nonatbib]{neurips_2020}

\usepackage[utf8]{inputenc} 
\usepackage[T1]{fontenc}    
\usepackage{hyperref}       

\usepackage{url}            
\usepackage{booktabs}       
\usepackage{amsfonts}       
\usepackage{nicefrac}       
\usepackage{microtype}      

\usepackage{epsf}
\usepackage{fancyhdr}
\usepackage{graphics}
\usepackage{graphicx}
\usepackage{psfrag}
\usepackage{color,xcolor}

\usepackage{algorithm}
\usepackage{algorithmic}

\usepackage{wrapfig}
\usepackage{float}

\usepackage{amsthm}
\usepackage{amsfonts}
\usepackage{amsmath}
\usepackage{amssymb,bbm}
\usepackage{multirow}
\usepackage{balance}

\usepackage[numbers]{natbib}

\usepackage{url}

\usepackage{comment}

 \usepackage{tabularx}%

\usepackage{enumitem}
\usepackage{booktabs}

\usepackage{bm,bbm}
\usepackage{mathtools}

\newtheorem{theorem}{Theorem}
\newtheorem{proposition}{Proposition}
\newtheorem{definition}{Definition}
\newtheorem{remark}{Remark}

\newcommand{\norm}[1]{\left\lVert#1\right\rVert}

\def\TM{\texttt{T}}

\def\TMX{\texttt{T}_{\!X}}

\def\TMZ{\texttt{T}_{\!Z}}

\def\FlowTGW{\textit{FlowAlign}} 
\def\DepthTGW{\textit{DepthAlign}} 
\def\FlowAlign{\textit{FlowAlign}}
\def\DepthAlign{\textit{DepthAlign}}

\def\GW{{\mathcal{GW}}}

\def\FTGW{{\mathcal{A}_{f}}}

\def\RFTGW{{\widehat{\mathcal{A}}_{f}}}

\def\RDTGW{{\widehat{\mathcal{A}}_{d}}}

\def\FA{{\mathcal{A}_{f}}}
\def\SFA{{\mathcal{A}_{f}^2}}
\def\RFFA{{\widehat{\mathcal{A}}_{f}}}
\def\SRFFA{{\widehat{\mathcal{A}}_{f}^2}}

\def\DA{{\mathcal{A}_{d}}}

\def\RFDA{{\widehat{\mathcal{A}}_{d}}}

\def\O{\mathcal{O}}

\def\RR{\mathbb{R}}

\def\Pp{\mathcal{P}}
\def\Ss{\mathcal{S}}

\def\Tt{\mathcal{T}}
\def\Mm{\mathcal{M}}

\title{Flow-based Alignment Approaches \\ for Probability Measures in Different Spaces}

\author{
  Tam Le\thanks{Equal contribution.} \\
  RIKEN AIP, Japan\\
  \texttt{tam.le@riken.jp} \\
  \And
  Nhat Ho$^{*}$ \\
  University of California, Berkeley \\
  \texttt{minhnhat@berkeley.edu} \\
  \AND
  Makoto Yamada \\
  Kyoto University, Japan \& RIKEN AIP \\
  \texttt{myamada@i.kyoto-u.ac.jp }
}

\begin{document}

\maketitle

\begin{abstract}
Gromov-Wasserstein (GW) is a powerful tool to compare probability measures whose supports are in different metric spaces. GW suffers however from a computational drawback since it requires to solve a complex non-convex quadratic program. We consider in this work a specific family of cost metrics, namely \textit{tree metrics} for a space of supports of each probability measure, and aim for developing efficient and scalable discrepancies between the probability measures. By leveraging a tree structure, we propose to align \textit{flows} from a root to each support instead of pair-wise tree metrics of supports, i.e., flows from a support to another, in GW. Consequently, we propose a novel discrepancy, named Flow-based Alignment (\FlowAlign), by matching the flows of the probability measures. We show that \FlowAlign~shares a similar structure as a univariate optimal transport distance. Therefore, \FlowAlign~is fast for computation and scalable for large-scale applications. By further exploring tree structures, we propose a variant of \FlowAlign, named Depth-based Alignment (\DepthAlign), by aligning the flows hierarchically along each depth level of the tree structures. Theoretically, we prove that both \FlowAlign~and \DepthAlign~are pseudo-distances. Moreover, we also derive tree-sliced variants, computed by averaging the corresponding \FlowAlign~/ \DepthAlign~using random tree metrics, built adaptively in spaces of supports. Empirically, we test our proposed discrepancies against other baselines on some benchmark tasks.
\end{abstract}




\section{Introduction}
Optimal transport (OT) theory provides a powerful set of tools to compare probability measures. OT has recently gained considerable interests in machine learning community~\cite{Cuturi-2013-Sinkhorn, genevay2016stochastic, NIPS2019_9130, NIPS2019_8703, muzellec2018generalizing, pmlr-v97-paty19a, perrot2016mapping, NIPS2019_8872}, and played an increasingly important role in several research areas, such as computer graphics~\cite{bonneel2016wasserstein, lavenant2018dynamical, solomon2015convolutional, solomon2019optimal}, domain adaptation~\cite{bhushan2018deepjdot, courty2017joint, courty2016optimal, pmlr-v89-redko19a}, and deep generative models~\cite{Arjovsky-2017-Wasserstein,  pmlr-v84-genevay18a, Gulrajani-2017-Improved, kolouri2018sliced, NIPS2019_8318, wu2019sliced}.

When probability measures are discrete and their supports are in the same space, OT distance can be recasted as a linear programming, which can be solved by standard interior-point method algorithms. However, these algorithms are not efficient when the number of supports is large. In order to account for the scalability of the OT distance, Cuturi \cite{Cuturi-2013-Sinkhorn} initiated a new research line by regularizing the OT with the entropy of the transport plans. Several efficient algorithms have been recently proposed to solve the entropic OT~\cite{Altschuler-2017-Near, altschuler2019massively, Dvurechensky-2018-Computational, Lin-2019-Efficient}.


When probability measures are discrete and their supports are in different spaces, the classical OT distance is no longer valid to measure their discrepancy. In his seminal work, M{\'e}moli \cite{Memoli_Gromov} introduced Gromov-Wasserstein (GW) distance to compare probability measures whose supports are in different metric spaces. Due to its flexibility, the GW distance has been used in several applications, including quantum chemistry~\cite{Peyre_Gromov}, computer graphics~\cite{solomon2016entropic}, cross-lingual embeddings~\cite{Alvarez-2018-Gromov, grave2019unsupervised}, graph partitioning and matching~\cite{Xu-2019-Scalable, Xu-2019-Gromov}, and deep generative models~\cite{bunne2019learning}. However, the GW is a complex non-convex quadratic program and NP-hard for arbitrary inputs~\cite{peyre2019computational} (\S10.6.3). Therefore, its computation is very costly, which hinders applications in large-scale settings where the number of supports is large. 

Reposing on the entropic regularization idea from OT, Peyr{\'e} et al. \cite{Peyre_Gromov} proposed an entropic GW discrepancy. The entropic GW can be efficiently solved by the Sinkhorn algorithm under certain cases of regularization parameter and a specific family of loss functions. Nevertheless, entropic GW requires the regularization to be sufficiently large for a fast computation, which leads to a poor approximation of GW. Following the direction of leveraging entropic regularization, Xu et al. \cite{Xu-2019-Scalable, Xu-2019-Gromov} proposed algorithmic approaches to further speed up GW for graph data. Another approach for scaling up the computation of GW is sliced GW~\cite{vayer2019sliced}, which relies on a one-dimensional projection of supports of the probability measures. Consequently, similar to sliced-Wasserstein, sliced GW albeit fast limits its capacity to capture high-dimensional structure in a distribution of supports \cite{le2019tree, pmlr-v97-liutkus19a}. Additionally, sliced GW can be \textit{only} either applied for discrete measures with the same number of supports and uniform weights, or required an artifact zero-padding for probability measures having different number of supports~\cite{vayer2019sliced}.

\textbf{Contributions.} In this work, we consider a particular family of cost metrics, namely \textit{tree metrics} for a space of supports of each probability measure, and aim for developing efficient and scalable discrepancies for probability measures in different spaces. Although it is well-known that one can leverage tree metrics to speed up a computation of arbitrary metrics \cite{bartal1996probabilistic, bartal1998approximating, charikar1998approximating, fakcharoenphol2004tight, indyk2001algorithmic}, our goal is rather to sample tree metrics for spaces of supports, and use them as cost metrics, similar to tree-sliced-Wasserstein (TSW)~\cite{le2019tree}. However, different to TSW, one may \textit{not} apply this idea straightforwardly by only using tree metrics as cost metrics for GW to develop scalable discrepancy for the probability measures in different tree metric spaces. Therefore, by exploiting a tree structure, we propose to align \textit{flows} from a root to each support instead of pair-wise tree metrics of supports, i.e., flows from a support to another, in GW for the probability measures. Consequently, we propose a novel discrepancy, named Flow-based Alignment (\FlowAlign), by matching the flows of the probability measures. \FlowAlign~is fast for computation and scalable for large-scale applications due to sharing a similar structure as a univariate OT. For further exploring tree structures, we propose to align the flows hierarchically along each depth level of the tree structures, named Depth-based Alignment (\DepthAlign). Theoretically, we prove that both \FlowAlign~and \DepthAlign~are pseudo-metrics, i.e., they are symmetric and satisfy the triangle inequality. Furthermore, we derive tree-sliced variants, computed by averaging the corresponding \FlowAlign~/ \DepthAlign~using random tree metrics, sampled by a fast adaptive method, e.g., clustering-based tree metric sampling~\cite{le2019tree} (\S4). 

\textbf{Organization.} 
The paper is organized as follows: we review tree metrics and GW in \S\ref{sec:TGW}. We propose two novel discrepancies: \FlowAlign~and \DepthAlign~ for probability measures in different tree metric spaces in \S\ref{sec:flow_tree_Gromov} and \S\ref{sec:depth_tree_Gromov} respectively. In \S\ref{sec:discussion}, we derive their tree-sliced variants for practical applications, and then evaluate the proposed discrepancies against other baselines on some benchmark tasks in \S\ref{sec:experiment} before concluding in \S\ref{sec:discuss}. 

\textbf{Notation.} We denote $[n] = \{1, 2, \ldots, n \}$, $\forall n \in \mathbb{N}_{+}$. For $x \in \RR^{d}$, let $\norm{x}_1$ be the $\ell_1$-norm of $x$, and $\delta_x$ be the Dirac function at $x$. For probability measure $\mu$, denote $|\mu|$ for the number of supports of $\mu$.

\section{Reminders on Tree Metric and Gromov-Wasserstein}
\label{sec:TGW}

In this section, we briefly review tree metric space and GW between probability measures in different tree metric spaces.

\subsection{Tree metric space}
\label{subsec:prelim}

Given a tree $\Tt$, let $d_{\TM}$ be a tree metric on $\Tt$. The tree metric $d_{\TM}$ between two nodes in $\Tt$ is equal to a length of the (unique) path between them \cite{semple2003phylogenetics} (\S7, p.145--182). Given node $x \in \Tt$, let $\Gamma(x)$ be the set of nodes in the subtree of $\Tt$ rooted at $x$, i.e., $\Gamma(x) = \left\{z \in \Tt \mid x \in \Pp(r, z) \right\}$ where $\Pp(r, z)$ is the (unique) path between root $r$ and node $z$ in $\Tt$, $\Ss(x)$ be the set of child nodes of $x$, and $\left|\Ss(\cdot) \right|$ is the cardinality of set $\Ss(\cdot)$. Given edge $e$, we write $u_e$ and $v_e$ for the nodes that are respectively at a shallower (closer to $r$) and deeper (further away from $r$) level of edge $e$, and $w_e$ be the non-negative length of that edge. We illustrate those notions in Fig.~\ref{fg:TreeMetric}.


 \begin{figure}
  \begin{center}
    \includegraphics[width=0.25\textwidth]{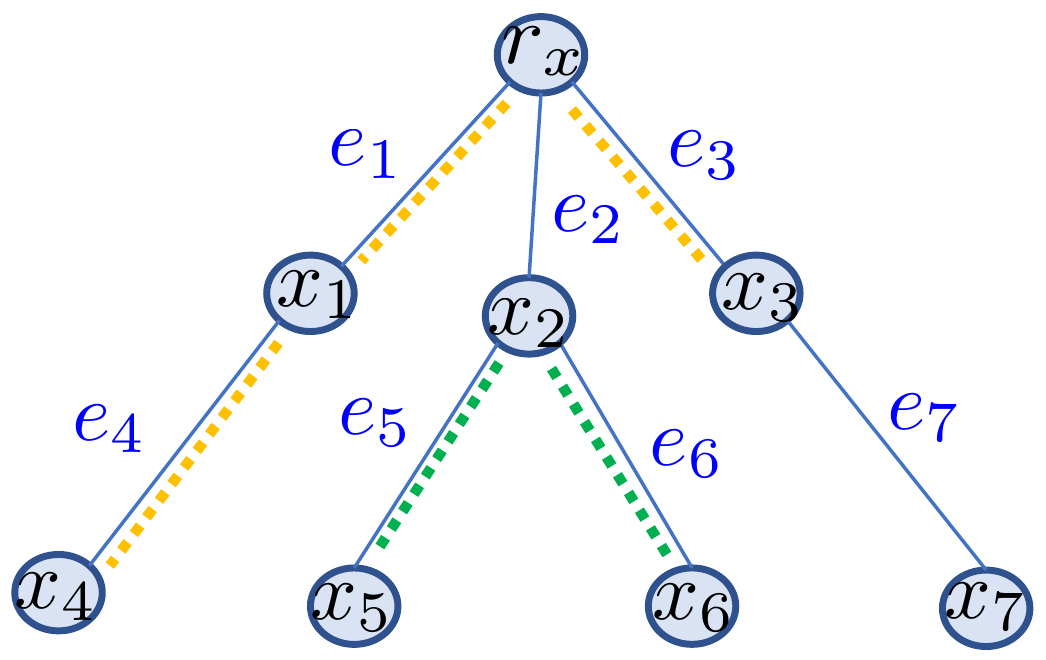}
  \end{center}
  \vspace{-10pt}
  \caption{An illustration for a tree metric space. $x_2$ is at depth level $2$. $\Pp(x_3, x_4)$ contains $e_3, e_1, e_4$ (the orange dot path), $\Gamma(x_2) = \left\{x_2, x_5, x_6 \right\}$ (the green dot subtree), and $\Ss(r_x) = \left\{x_1, x_2, x_3 \right\}$. For edge $e_5$, $v_{e_5} = x_5$ and $u_{e_5} = x_2$.}
  \label{fg:TreeMetric}
 \vspace{-5pt}
\end{figure}

Throughout the paper, we are given two probability measures $\mu = \sum_{i \in [k]} a_i \delta_{x_i}$ and $\nu = \sum_{j \in [k']} b_j \delta_{z_j}$ whose supports $x_i \mid_{i \in [k]}$ and $z_j \mid_{j \in [k']}$ are in different tree metric spaces $(\Tt_X, d_{\TMX})$ and $(\Tt_Z, d_{\TMZ})$ respectively; $a_i, b_j \in \RR_{+}, \forall i \in [k], j \in [k']$ such that $\sum_{i \in [k]} a_i = \sum_{j \in [k']} b_j = 1$. Our goal is to define discrepancies between these probability measures.

\subsection{Gromov-Wasserstein}
In \cite{Memoli_Gromov}, M{\'e}moli defined GW between $\mu$ and $\nu$ in different tree metric spaces as follow:
\begin{equation}
   \GW^2(\mu, \nu) :=  \min_{T \in \Pi(\mu, \nu)} \sum_{i,j,i', j'} \left| d_{\TMX}(x_i, x_{i'}) - d_{\TMZ}(z_j, z_{j'} )\right|^2 T_{ij} T_{i'j'},   
 \label{equ:GW}
\end{equation}
where $\Pi(\mu, \nu) : = \left\{T \in \RR_{+}^{k \times k'} \mid \sum_{j \in [k']} T_{ij} = a_i, \sum_{i \in [k]} T_{ij} = b_j, \forall i \in [k], j \in [k'] \right\}$ is a set of the transport plans between $\mu$ and $\nu$. Intuitively, GW aligns pair-wise tree metrics of supports $d_{\TM_X}(x_i, x_{i'})\mid_{i, i'}$ and $d_{\TM_Z}(z_j, z_{j'}) \mid_{j, j'}$ for $\mu$ and $\nu$.

However, one may not scale up GW by straightforwardly using tree metrics as cost metrics as in Eq.~\eqref{equ:GW} like TSW~\citep{le2019tree}. Therefore, we propose to leverage tree structure to align \textit{flows} from a root to each support instead of pair-wise tree metrics of supports, i.e., flows from a support to another, in GW to develop scalable discrepancy for the probability measures. Consequently, we propose two novel discrepancies for probability measures in different tree metric spaces: \FlowAlign~and \DepthAlign, detailed in \S\ref{sec:flow_tree_Gromov} and \S\ref{sec:depth_tree_Gromov} respectively. 


\section{Flow-based Alignment Discrepancy}
\label{sec:flow_tree_Gromov}
In this section, we propose a novel, efficient and scalable discrepancy, named Flow-based Alignment (\FlowAlign), for probability measures in different tree metric spaces. 

\subsection{Definition of \FlowAlign}
\label{subsec:definition_flow_Gromov}
Different from GW for probability measures in different tree metric spaces, \FlowAlign~exploits tree structures for the alignment.   
\begin{definition}
\label{def:flow_based}
The Flow-based Alignment discrepancy between $\mu$ and $\nu$ is defined as follows:
\begin{equation}
\label{equ:FTGW}
\hspace{-6pt} \SFA(\mu, \nu) := \min_{r_x, r_z, T \in \Pi(\mu, \nu)} \sum_{i, j} \left|d_{\TMX}(r_x, x_i) - d_{\TMZ}(r_z, z_j) \right|^2 T_{ij}.
\end{equation}
  \vspace{-16pt}
\end{definition}
Intuitively, \FlowAlign~considers the matching for \textit{flows} from a root to each support for probability measures based on (i) the flow lengths (i.e., tree metrics from a root to each support), and (ii) the flow masses (i.e., weights on supports corresponding to the flows). Moreover, \FlowAlign~also takes into account the root alignment for corresponding tree structures of tree metric spaces since the \textit{flows} depend on which node in the tree structure has a role as the tree root. Therefore, instead of matching \textit{pairs of supports} as in GW for probability measures in different spaces, \FlowAlign~exploits tree structures of the tree metric spaces to align \textit{both tree root and supports} for the probability measures. One should distinguish \FlowAlign~from tree-(sliced)-Wasserstein which directly matches \textit{supports} for probability measures in the \textit{same} tree metric space.

\begin{theorem}\label{thm:FlowTGW}
\FlowAlign~is a pseudo-distance. It satisfies symmetry and the triangle inequality.
\vspace{-5pt}
\end{theorem}

See the supplementary (\S A) for the proof of Theorem~\ref{thm:FlowTGW}. When $\SFA(\mu, \nu) = 0$, we can find roots $r_{x}^{*}$ and $r_{z}^{*}$ such that $\tilde{\mu}^{*} \equiv \tilde{\nu}^{*}$ where $\tilde{\mu}^{*} = \sum_{i} a_{i} \delta_{d_{\TMX}(r_{\!x}^{*}, x_i)}$ and $\tilde{\nu}^{*} = \sum_j b_j \delta_{d_{\TMZ}(r_{\!z}^{*}, z_j)}$. It demonstrates that $\mu$ and $\nu$ have similar weights on supports (i.e., flow masses) while the tree metrics of their supports to the corresponding root $r_{x}^{*}$ or $r_{z}^{*}$ (i.e., flow lengths) are identical. When the deepest levels of trees $\Tt_{X}$ and $\Tt_{Z}$ are equal to two, we have the following relation between \FlowAlign~and GW.
\begin{proposition}
\label{prop:flow_to_Gromov}
If the deepest levels of trees $\Tt_{X}$ and $\Tt_{Z}$ are 2, then $\GW(\mu, \nu) \leq 2 \FA(\mu, \nu)$.
\end{proposition}
See the supplementary (\S A) for the proof of Proposition~\ref{prop:flow_to_Gromov}. In practical applications without priori knowledge about tree structures for probability measures, $\Tt_{X}$ and $\Tt_{Z}$ are sampled from support data points, e.g., by clustering-based tree metric sampling~\cite{le2019tree}. We argue in the supplementary (\S B) that the farthest-point clustering within the clustering-based tree metric sampling ensures that \FlowAlign~is invariant to rotation and translation.

\subsection{Efficient computation for \FlowAlign}
\label{subsec:compute_flow_tree_Gromov}
A naive implementation for \FlowAlign~$\FA$ has a complexity $\O(N^3 \log N)$ where $N$ is the number of nodes in tree, if one exhaustively searches the optimal pair of roots for $\Tt_X$ and $\Tt_Z$\footnote{More details about (aligned-root) \FlowAlign~complexity are given in \S\ref{subsec:align_flow_tree_gromov}, and in the supplementary (\S C).}. In this section, we present an efficient computation approach which reduces this complexity into nearly $\O(N^2)$.

Consider $\FA$ between $\mu, \nu$ in $\Tt_X, \Tt_Z$ rooted at $r_x, r_z$ respectively. When one changes into the new root $\bar{r}_z$ for tree $\Tt_Z$, as illustrated in Fig.~\ref{fg:DP_FlowTGW}, there are two cases that can happen:

\begin{wrapfigure}{r}{0.46\textwidth} 
  \vspace{-20pt}
  \begin{center}
    \includegraphics[width=0.46\textwidth]{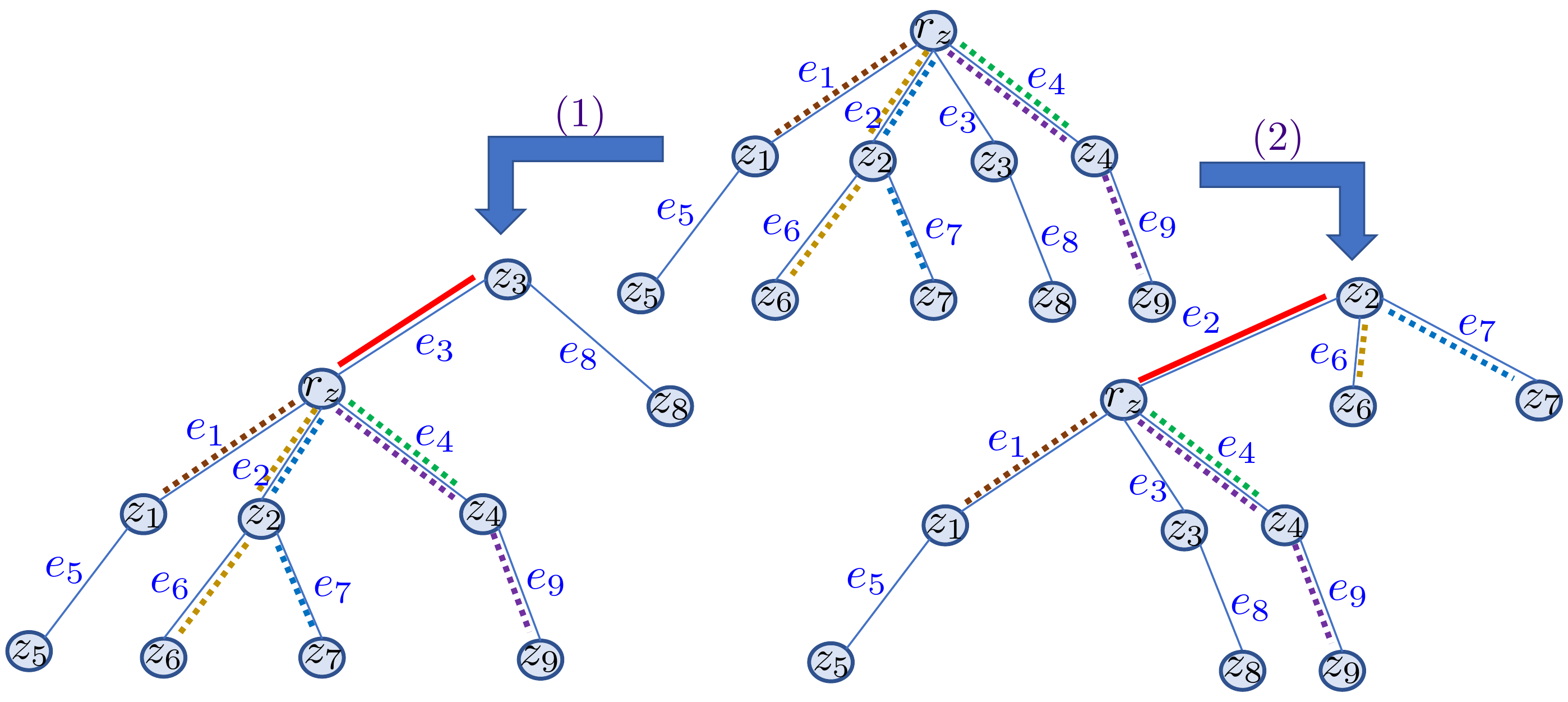}
  \end{center}
  \vspace{-14pt}
  \caption{An illustration for an efficient computation for \FlowTGW. Given $\nu = b_1\delta_{z_1} + b_2\delta_{z_6} + b_3\delta_{z_7} + b_4\delta_{z_4} + b_5\delta_{z_9}$, when the new root $\bar{r}_z = z_3$ ($z_3$ is in the subtree rooted at $z_3$, and not containing any supports of $\nu$), the order of $d_{\TM}(r_z, z_i) \mid_{z_i \in \nu}$ is the same as that of $d_{\TM}(z_3, z_i) \mid_{z_i \in \nu}$, and $d_{\TM}(z_3, z_i) = d_{\TM}(r_z, z_i) + d_{\TM}(r_z, z_3), \forall z_i \in \nu$ (\textbf{Case 1}: the left-bottom tree). When the new root $\bar{r}_z = z_2$ ($z_2$ is in the subtree rooted at $z_2$, and containing supports $\Omega_{\nu} = \{z_6, z_7\}$ in $\nu$), the order of $d_{\TM}(r_z, z_i) \mid_{z_i \in \nu, z_i \notin \Omega_{\nu}}$ is the same as that of $d_{\TM}(z_3, z_i) \mid_{z_i \in \nu, z_i \notin \Omega_{\nu}}$, and $d_{\TM}(z_2, z_i) = d_{\TM}(r_z, z_i) + d_{\TM}(r_z, z_3), \forall z_i \in \nu, z_i \notin \Omega_{\nu}$ (\textbf{Case 2}: the right-bottom tree).}
  \label{fg:DP_FlowTGW}
  \vspace{-35pt}
\end{wrapfigure}

\textbf{Case 1}: $\bar{r}_z$ is in the subtree rooted at a node in $\Ss(r_z)$, which does not contain any supports in $\nu$, illustrated in the left-bottom tree of Fig.~\ref{fg:DP_FlowTGW}. Then, $\forall z_i$ in $\nu$, we have $d_{\Tt_Z}(\bar{r}_z, z_i) = d_{\Tt_Z}(r_z, z_i) + d_{\Tt_Z}(\bar{r}_z, r_z)$. Therefore, the path-length order is preserved.

\textbf{Case 2}: $\bar{r}_z$ is in the subtree rooted at a node in $\Ss(r_z)$, containing some supports in $\nu$, denoted as $\Omega_{\nu}$, illustrated in the right-bottom tree of Fig.~\ref{fg:DP_FlowTGW}. Then, $\forall z_j$ in $\nu$, except $z_i \in \Omega_{\nu}$, we have $d_{\Tt_Z}(\bar{r}_z, z_j) = d_{\Tt_Z}(r_z, z_j) + d_{\Tt_Z}(\bar{r}_z, r_z)$. Thus, the path-length order (except those for $z_i \in \Omega_{\nu}$) is preserved. For supports in $\Omega_{\nu}$ (illustrated in the supplementary (\S B)), there are three following sub-cases:

$\bullet \,$ \textbf{Case 2a}: For supports $z_i \in \Omega_{\nu}$ which $\bar{r}_z \in \Pp(r_z, z_i)$, then $d_{\Tt_Z}(\bar{r}_z, z_i) = d_{\Tt_Z}(r_z, z_i) - d_{\Tt_Z}(r_z, \bar{r}_z)$. So, the path-length order of those supports are preserved.

$\bullet \,$ \textbf{Case 2b}: For supports $z_i \in \Omega_{\nu}$ which $z_i \in \Pp(r_z, \bar{r}_z)$, then $d_{\Tt_Z}(\bar{r}_z, z_i) = d_{\Tt_Z}(r_z, \bar{r}_z) - d_{\Tt_Z}(r_z, z_i)$. Therefore, the path-length order of those supports are reversed.

$\bullet \,$ \textbf{Case 2c}: For supports $z_i \in \Omega_{\nu}$ which $\bar{r}_z \notin \Pp(r_z, z_i)$ and $z_i \notin \Pp(r_z, \bar{r}_z)$, then one needs to find the corresponding closest common ancestor $\zeta_i$ of $\bar{r}_z$ and $z_i$, i.e., $\zeta_i$ is on both paths $\Pp(r_z, \bar{r}_z)$ and $\Pp(r_z, z_i)$, and $ d_{\Tt_Z}(\bar{r}_z, z_i) = d_{\Tt_Z}(r_z, \bar{r}_z) + d_{\Tt_Z}(r_z, z_i) - 2d_{\Tt_Z}(r_z, \zeta_i)$. Note that the path-length order of supports having the same $\zeta_i$ is preserved.

Therefore, one only needs to merge these ordered arrays with the complexity nearly $\O(N)$ (except the degenerated case where each array has only one node).

From the above observation, one may not need to sort for tree metrics between $\bar{r}_z$ and each support in $\nu$ by leveraging the sorted order of the tree metrics between $r_z$ and each support. Moreover, those computational steps can be done separately for each tree. Therefore, the complexity of $\FTGW$ reduces from $\O(N^3 log N)$ into nearly $\O(N^2)$. More details can be seen in the supplementary (\S C).

\subsection{Aligned-root \FlowAlign}\label{subsec:align_flow_tree_gromov}

We consider a special case of \FlowAlign~where roots have been already aligned. Therefore, we can leave out minimization step with roots in Def.~\ref{def:flow_based}, and name it as aligned-root \FlowAlign.

\begin{definition}
\label{def:align_flow_based}
Assume that root $r_x$ in $\Tt_{X}$ is aligned with root $r_z$ in $\Tt_{Z}$. Then, the aligned-root Flow-based Alignment discrepancy between $\mu$ and $\nu$ is defined as follow:
\begin{align}
\SRFFA(\mu, \nu; r_x, r_z) := \min_{T \in \Pi(\mu, \nu)} \sum_{i, j} \left|d_{\TMX}(r_x, x_i) - d_{\TMZ}(r_z, z_j) \right|^2 T_{ij}. \label{equ:RFTGW}
\end{align}
  \vspace{-15pt}
\end{definition}

The $\RFFA$ in Eq.~(\ref{equ:RFTGW}) is equivalent to the univariate Wasserstein distance between $\tilde{\mu} : = \sum_{i} a_i \delta_{d_{\TMX}(r_{\!x}, x_i)}$ and $\tilde{\nu} : = \sum_j b_j \delta_{d_{\TMZ}(r_{\!z}, z_j)}$, which is equal to the integral of the absolute difference between the generalized quantile functions of these two univariate probability distributions \cite{SantambrogioBook} (\S2). Therefore, one only needs to \textit{sort} flow lengths $d_{\TMX}(r_{\!x}, x_i) \mid_i$, and $d_{\TMZ}(r_{\!z}, z_j) \mid_j$ for the computation of $\RFTGW$, i.e., linearithmic complexity. Due to sharing the same structure as a univariate Wasserstein distance, $\RFFA$ inherits the same properties as those of the univariate Wasserstein distance. More precisely, $\RFFA$ is symmetric and satisfies the triangle inequality. Additionally, $\RFFA(\mu, \nu; r_x, r_z) = 0$ is equivalent to $\tilde{\mu} \equiv \tilde{\nu}$. See the supplementary (\S B) for its illustration.


\begin{wrapfigure}{r}{0.45\textwidth}
 \vspace{-30pt}
  \begin{center}
    \includegraphics[width=0.4\textwidth]{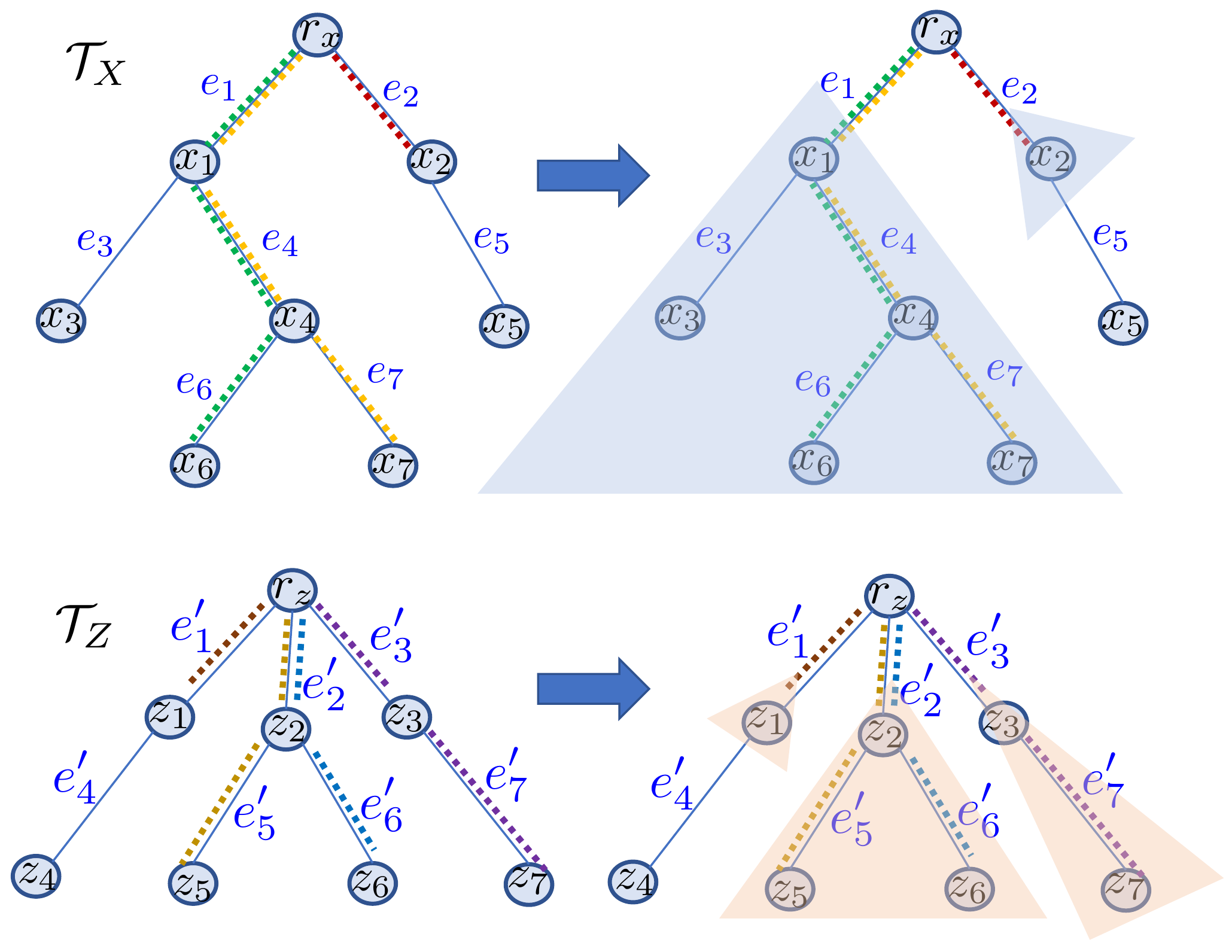}
  \end{center}
  \vspace{-16pt}
  \caption{An illustration for aligned-root \DepthAlign~$\RFDA$ between $\mu = a_1 \delta_{x_6} + a_2 \delta_{x_7} + a_3 \delta_{x_2}$ on $\Tt_X$ and $\nu = b_1 \delta_{z_1} + b_2 \delta_{z_5} + b_3 \delta_{z_6} + b_4 \delta_{z_7}$ on $\Tt_Z$. In $\RDTGW$, we consider the optimal alignment at each depth level $h$. At $h=1$, root $r_x$ is trivially aligned for root $r_z$. Since $r_x, r_z$ have their child nodes, the alignment $(r_x, r_z)$ is recursive into $h=2$. For $\mu$ in $\Tt_X$, $r_x$ has $2$ subtrees rooted at $x_1, x_2$, considered as ``leaves". Thus, $V(\Tt_{r_x}^2) = \left\{r_x, x_1, x_2\right\}$, and $\mu_{\Tt_{r_x}^2} = (a_1 + a_2)\delta_{x_1} + a_3 \delta_{x_2}$. Similarly, for $\nu$ in $\Tt_Z$, $\nu_{\Tt_{r_z}^2} = b_1 \delta_{z_1} + (b_2 + b_3)\delta_{z_2} + b_4 \delta_{z_3}$. The recursive procedure is repeated until the deepest level of the lower tree where only simple cases exist.}
  \label{fg:DepthBased_GTW}
  \vspace{-28pt}
\end{wrapfigure}

Note that in practical applications, we usually do not have priori knowledge about tree structures for probability measures. Therefore, we need to sample tree metrics for each support data space. Moreover, we can \textit{directly sample aligned-root tree metrics}, e.g., by choosing means of support data distributions as roots when using the clustering-based tree metric sampling~\cite{le2019tree}. Consequently, we can reduce the complexity of \FlowTGW~by using aligned-root \FlowTGW.

\textbf{Aligned-root \FlowAlign~barycenter.} The aligned-root \FlowAlign~can be handily used for a barycenter problem, especially in large-scale applications. Given $m$ probability measures $\mu_{i} \mid_{i \in [m]}$ in different tree metric spaces $(\Tt_{X_{i}}, d_{\TM_{\!X_{i}}}) \mid_{i \in [m]}$ with aligned-roots $r_{x_i} \mid_{i \in [m]}$ respectively, and corresponding weights $p_{i} \mid_{i \in [m]}$, the aligned-root \FlowAlign~barycenter aims to find a flow-based tree structure $\Delta_{\bar{\mu}} := \left\{ d_{\TMX}(r_x, x_i), a_i\right\}_{i \in [\texttt{k}]}$ of an optimal probability measure $\bar{\mu}$ with at most $\texttt{k}$ supports in $(\Tt_X, d_{\TMX})$ that takes the form:
\begin{align}
\Delta_{\bar{\mu}} \in \mathop{ \arg \min} \limits_{\Delta_{\hat{\mu}}} \biggr(\sum_{i = 1}^{m} p_{i}\SRFFA(\hat{\mu}, \mu_{i}; r_x, r_{x_{i}}) \biggr), \label{eq:bary_flow_based}
\end{align}
where the roots $r_{x_{i}}$ in $\Tt_{X_{i}}$ $\mid_{i \in [m]}$ are aligned with root $r_{x}$ in $\Tt_X$. The barycenter problem in Equ.~(\ref{eq:bary_flow_based}) is equivalent to the free-support univariate Wasserstein barycenter efficiently solved, e.g., by using Alg. $2$ in \cite{Cuturi-2014-Fast}.

\section{Depth-based Alignment Discrepancy}\label{sec:depth_tree_Gromov}

\FlowAlign~only focuses on flows from a root to each support and tree root alignment, but ignores the depth level of supports in trees. In this section, we take into account the depth level of supports, and propose Depth-based Alignment (\DepthAlign) discrepancy $\DA$ by considering the alignment for flows hierarchically for each depth level along the tree structures. We first introduce some necessary definitions to define $\DA$. Recall that, $\Ss(x)$ is a set of child nodes of $x$ in $\Tt$.
\begin{definition}\label{def:2-depth-level-tree}
Given node $x$ in $\Tt$, a 2-depth-level tree $\Tt(x, \Ss(x))$, or shortened as $\Tt_x^2$, is defined in $\Tt$ rooted at $x$, i.e., root $x$ is at depth level $1$, and $\left| \Ss(x) \right|$ subtrees rooted at $\bar{x} \mid_{\bar{x} \in \Ss(x)}$, considered as ``leaves" at depth level $2$ in $\Tt_{x}^2$.
\end{definition}

Let $V(\Tt_{x}^{2})$ be the set of vertices of $\Tt_x^2$. Following Def. \ref{def:2-depth-level-tree}, $V(\Tt_{x}^{2})$ contains $x$ and all $\bar{x} \in \Ss(x)$. Moreover, given $\mu$ in $\Tt$, we have a corresponding $\mu_{\Tt_{x}^{2}}$ in $\Tt_{x}^{2}$, defined as $\mu_{\Tt_{x}^{2}} = \sum_i \bar{a}_i \delta_{\bar{x}_i}$ where $\bar{x}_i \in V(\Tt_{x}^{2})$, and $\bar{a}_i = \mu(\Gamma(\bar{x}_i))/ \mu(\Gamma(x))$ if $\bar{x}_i \ne x$, otherwise  $\bar{a}_i = 1 - \left[\sum_{\bar{x}_j \in \Ss(x)} \mu(\Gamma(\bar{x}_j))\right]/ \mu(\Gamma(x))$. 

In order to define \DepthAlign, we start with its special case when roots are aligned.

\begin{definition}
\label{def:depth_based}
Assume that root $r_x$ in $\Tt_{X}$ is aligned with root $r_z$ in $\Tt_{Z}$. Then, the aligned-root Depth-based Alignment between $\mu$ and $\nu$ is defined as follows:
\begin{align}
\RFDA(\mu, \nu; r_x, r_z) := \sum_{h} \sum_{(x, z) \in \Mm_{h-1}} T_{h-1}^{*}(x, z) \RFFA(\mu_{\Tt_{x}^{2}}, \nu_{\Tt_{z}^{2}}; x, z), \label{equ:RDTGW}
\end{align}
where $h$ is the considered depth level, starting from $2$ to the deepest level of the lower tree between $\Tt_X$ and $\Tt_Z$; $\Mm_{h}$ is a set of optimal aligned pairs at the depth level $h$ where $\Mm_{1} = \left\{(r_x, r_z)\right\}$; $T_{h}^{*}(x, z)$ is the optimal matching mass for the pair $(x, z)$ at the depth level $h$ where $T_{1}^{*}(r_x, r_z)=1$.
\end{definition}
Intuitively, at each depth level $h$, we consider the alignment for the corresponding 2-depth-level trees. Note that the 2-depth-level tree structures are at the same depth level $h$ for both $\Tt_X$ and $\Tt_Z$, and one can consider $\RFFA$ for such alignment. Moreover, at $h=1$, $\RFDA$ trivially matches $r_x$ to $r_z$ with optimal matching mass $1$. Thus, the matching procedure is recursive along all depth levels in trees. The simple case of the recursive procedure is that either at least one node of considered pair does not have child nodes, or sum of weights of child nodes in the corresponding 2-depth-level tree is equal to $0$.


When roots of trees $\Tt_{X}$ and $\Tt_{Z}$ are not aligned yet, we optimize the root alignment as follow:
\begin{equation}
\label{equ:DTGW}
    \DA(\mu, \nu) : = \min_{r_x, r_z} \RFDA(\mu, \nu; r_x, r_z),
\end{equation}
which is referred to as \DepthAlign. Similar to \FlowAlign, we have a following theorem: 

\begin{theorem}\label{thm:DepthTGW}
\DepthTGW~is a pseudo-distance. It satisfies symmetry and the triangle inequality.
\vspace{-8pt}
\end{theorem}
See the supplementary (\S A) for the proof of Theorem~\ref{thm:DepthTGW}. When $\DA(\mu, \nu)=0$, we can find roots $r_x^{*}$ and $r_z^{*}$ such that all the hierarchical corresponding $\RFFA(\mu_{\Tt_{\cdot}^{2}}, \nu_{\Tt_{\cdot}^{2}}; \cdot, \cdot)$ for each depth level along the trees are equal to $0$. It demonstrates that $\mu$ and $\nu$ have similar weights on supports while their supports have similar depth levels, and for each depth level, the tree metrics of supports in the corresponding $\mu_{\Tt_{\cdot}^{2}}, \nu_{\Tt_{\cdot}^{2}}$ to the 2-depth-level-tree roots are identical, i.e., corresponding weight edges are  identical.


\section{Tree-sliced Variants by Sampling Tree Metrics}\label{sec:discussion}


Similar to TSW~\cite{le2019tree}, computing (aligned-root) \FlowAlign~/ \DepthAlign~requires to choose or sample tree metrics for each space of supports. We use fast adaptive methods, e.g., clustering-based tree metric sampling~\cite{le2019tree}, to sample tree metrics, and further propose their tree-sliced variants by averaging the corresponding (aligned-root) \FlowAlign~/ \DepthAlign~using those random sampled tree metrics.

\begin{definition}\label{def:tree_sliced_GW}
Given $\mu, \nu$ supported on a set in which tree metric spaces $\left\{(\Tt_{X_i}, d_{\TM_{\!X_i}}) \mid i \in [n] \right\}$ and $\left\{(\Tt_{Z_i}, d_{\TM_{\!Z_i}}) \mid i \in [n] \right\}$ can be defined respectively, the tree-sliced variants of (aligned-root) \FlowAlign~/ \DepthAlign~is defined as an average of corresponding (aligned-root) \FlowAlign~/ \DepthAlign~for $\mu, \nu$ on $(\Tt_{X_i}, d_{\TM_{X_i}})$, and $(\Tt_{Z_i}, d_{\TM_{Z_i}})$ respectively.
\end{definition}

As discussed in \cite{le2019tree}, the average over different random tree metrics can reduce quantization effects or clustering sensitivity problems in which data points may be partitioned or clustered to adjacent but different hypercubes or clusters respectively in tree metric sampling. Moreover, the complexity of tree metric sampling is negligible. Indeed, for clustering-based tree metric sampling, its complexity is $\O(H_{\Tt} m \log \kappa)$ when one fixes the same number of clusters $\kappa$ for the farthest-point clustering~\citep{gonzalez1985clustering} and sets $H_{\Tt}$ for the predefined deepest level of tree $\Tt$, and $m$ is the number of input data points.

\begin{remark}\label{remark:tree_GW}
For specific applications with priori knowledge about tree metrics for probability measures, one can apply \FlowAlign, or consider \DepthAlign~if the known tree structure is important for the applications. Moreover, if roots of those known tree metrics are already aligned, one can use the corresponding aligned-root formulations to reduce the complexity. For general applications without priori knowledge about tree metrics for probability measures, one can directly sample aligned-root tree metrics, e.g., by choosing a mean of support data as its root for the clustering-based tree metric sampling~\cite{le2019tree}, and use the aligned-root formulations for an efficient computation.
\end{remark}

\section{Experiments}\label{sec:experiment}

We evaluate our proposed discrepancies for quantum chemistry and document classification with \textit{randomly linear transform} word embeddings. We also carry out the large-scale \FlowTGW~barycenter problem within $k$-means clustering for point clouds of handwritten digits in \texttt{MNIST} dataset \textit{rotated arbitrarily} in the plane as in~\cite{peyre2016gromov}. 





\textbf{Setup.} We consider two baselines: (i) sliced GW (SGW) \cite{vayer2019sliced} and (ii) entropic GW (EGW) \cite{peyre2016gromov}. In all of our experiments, we do not have prior knowledge about tree metrics for probability measures. Therefore, we sample aligned-root tree metrics from support data points by applying the clustering-based tree metric sampling \cite{le2019tree} where means of support data points are chosen as tree roots. Consequently, we can leverage the aligned-root formulations for both \FlowAlign~(FA) and \DepthAlign~(DA) to reduce their complexity. For SGW, we follow Vayer et al. \cite{vayer2019sliced} to add artifact zero-padding for discrete measures having different numbers of supports, and use the binomial expansion to reduce its complexity. For EGW, we use the entropic regularization to optimize transport plan, but exclude it when computing GW, which gives comparative or better performances than those of standard EGW. We also apply the log-stabilized Sinkhorn \cite{schmitzer2019stabilized}. We observe that the quality of EGW is better when entropic regularization becomes smaller, but the computation is considerably slower. In our experiments, the computation for EGW is either usually blown up, or too slow for evaluation when entropic regularization is less than or equal $1$. We run experiments with Intel Xeon CPU E7-8891v3 (2.80GHz), and 256GB RAM. Reported time consumption for all methods has already included their corresponding preprocessing, e.g., tree metric sampling for \FlowAlign~and \DepthAlign, or one-dimensional projection for SGW.

\begin{wrapfigure}{r}{0.38\textwidth}
  \vspace{-50pt}
  \begin{center}
    \includegraphics[width=0.38\textwidth]{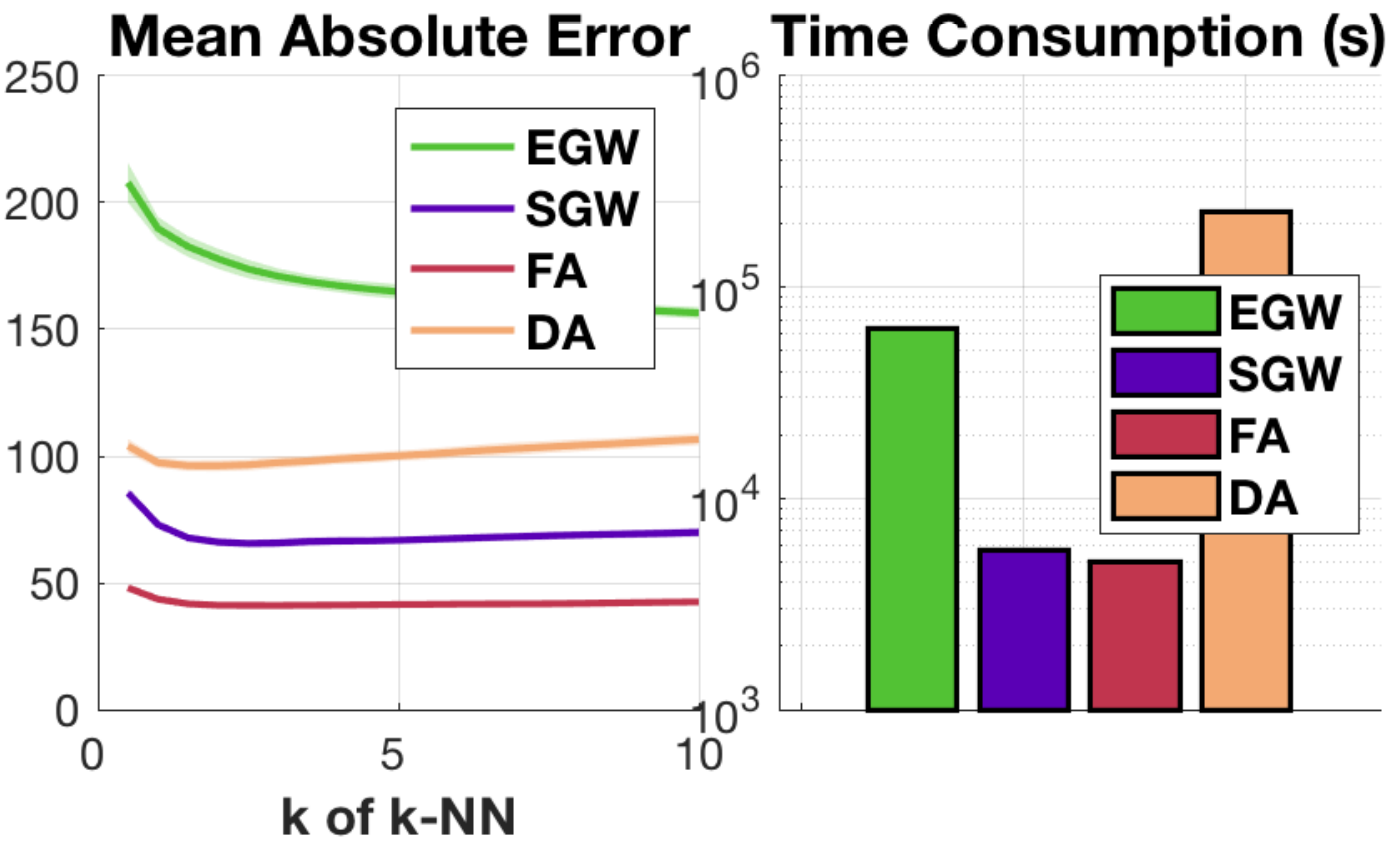}
  \end{center}
  \vspace{-10pt}
  \caption{MAE and time consumption of $k$-NN regression on \texttt{qm7} for EGW (eps=5), SGW (10 slices), FA (10 tree-slices), and DA (1 tree-slice).}
  \label{fg:exp_quantum}
 \vspace{-20pt}
\end{wrapfigure}

\subsection{Applications}

\textbf{Quantum chemistry.} We consider a regression problem on molecules for \texttt{qm7} dataset as in \cite{peyre2016gromov}. The task is to predict atomization energies for molecules based on similar labeled molecules instead of estimating them through expensive numerical simulations \cite{peyre2016gromov, rupp2012fast}. For simplicity, we only used the relative locations in $\RR^3$ of atoms in molecules, \textit{without information about atomic nuclear charges} as the experiments in \cite{peyre2016gromov, rupp2012fast}. We randomly split $80\%/20\%$ for training and test sets, and repeat $20$ times. Following Peyr{\'e} et al. \cite{peyre2016gromov}, we use $k$-nearest neighbor ($k$-NN) regression. 



\textbf{Document classification with non-registered word embeddings.} We also evaluate our proposed discrepancies for document classification with non-registered word embeddings in \texttt{TWITTER}, \texttt{RECIPE}, \texttt{CLASSIC}, and \texttt{AMAZON} datasets. For each document in these datasets, we use a \textit{randomly linear transform} for $word2vec$ word embedding \cite{mikolov2013distributed}, pre-trained on Google News\footnote{https://code.google.com/p/word2vec}, containing about $3$ million words/phrases. $word2vec$ maps those words/phrases into $\RR^{300}$. Following \cite{kusner2015word, le2019tree}, we remove SMART stop words \cite{salton1988term}, and drop words in documents if they are not in the pre-trained $word2vec$. We randomly split $80\%/20\%$ for training and test sets, and repeat $20$ times. 


\textbf{Performance results, time consumption and discussions.} The results of averaged mean absolute value (MAE) for different $k$ in $k$-NN regression, and time consumption of quantum chemistry in \texttt{qm7} dataset are illustrated in Fig.~\ref{fg:exp_quantum}, while the results of averaged accuracy for different $k$ in $k$-NN, and time consumption of document classification with non-registered word embeddings in \texttt{TWITTER}, \texttt{RECIPE}, \texttt{CLASSIC}, and \texttt{AMAZON} datasets are shown in Fig.~\ref{fg:exp_document}.

The computational time of \FlowAlign~is at least comparative to that of SGW, and several-order faster than that of EGW. Especially, in \texttt{CLASSIC}, it took $12$ minutes for \FlowAlign~($10$ slices), while $6.8$ hours for SGW ($10$ slices), and $3.3$ days for EGW (entropic regularization eps=10). Moreover, the performances of \FlowAlign~compare favorably with other baselines, except EGW in RECIPE dataset. \FlowAlign~performs better when the number of tree slices is increased, but its time consumption is also increased linearly. We show this trade-off on \texttt{qm7} in Fig.~\ref{fg:exp_quantum_FA_slice}. For \DepthAlign, its performances are comparative with other baselines. However, \DepthAlign~is slow in practice due to solving a large number of sub-problems, i.e., aligned-root \FlowAlign~between corresponding 2-depth-level trees. For EGW, its performances are improved when the entropic regularization small enough for the problems, but its computational time is considerably increased which makes EGW unsuitable for large-scale applications. The value of entropic regularization is important for performances of EGW, e.g., EGW performs well in \texttt{RECIPE}, and is comparative with other approaches on other datasets. For SGW, its computational time is slow down when document lengths are large, e.g., in \texttt{AMAZON} dataset, since it requires to use extra artificial zeros padding and uniform weights for probability measures with different number of supports (i.e., documents with different lengths), while other approaches work with an original number of supports (i.e., unique words in documents), and general weights (i.e., frequencies of unique words) for supports in probability measures.

 \begin{figure}
  \begin{center}
    \includegraphics[width=0.65\textwidth]{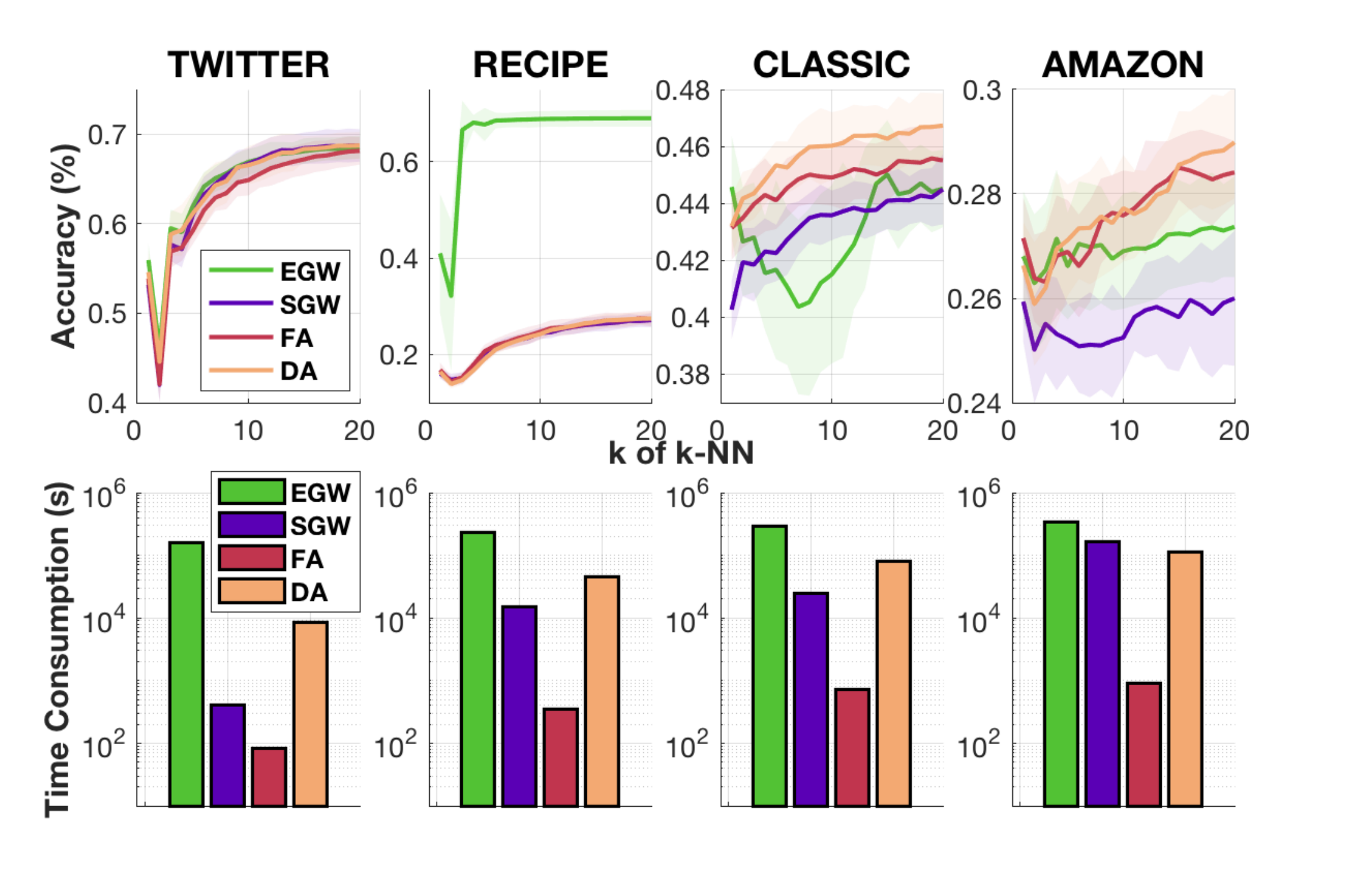}
  \end{center}
  \vspace{-10pt}
  \caption{Accuracy and time consumption of $k$-NN classification on document datasets for EGW (eps=5 in \texttt{TWITTER}, and eps=10 in others), SGW (10 slices), FA (10 tree-slices), and DA (1 tree-slice).}
  \label{fg:exp_document}
 \vspace{-15pt}
\end{figure}


\begin{wrapfigure}{r}{0.4\textwidth}
  \vspace{-18pt}
  \begin{center}
    \includegraphics[width=0.4\textwidth]{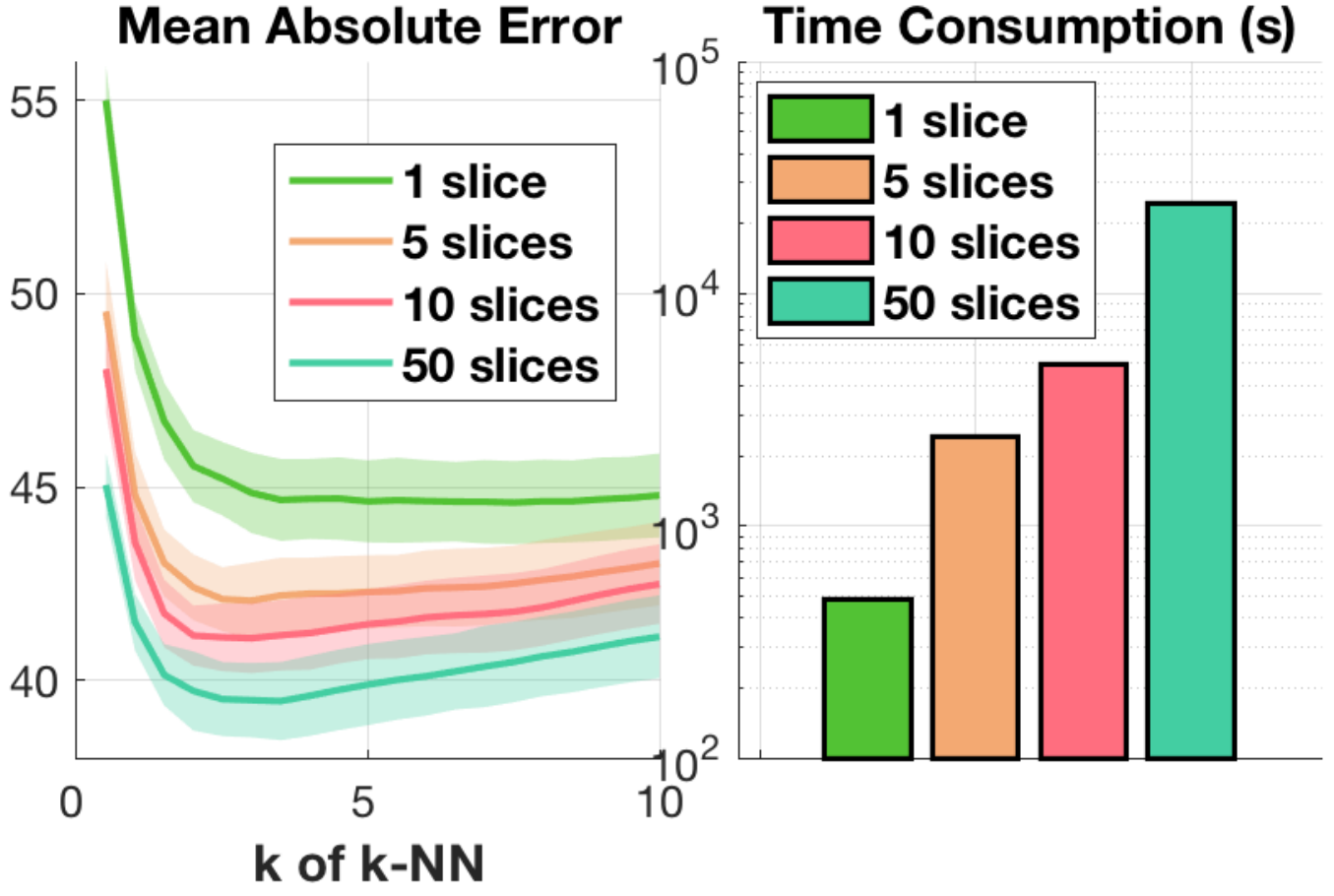}
  \end{center}
  \vspace{-10pt}
  \caption{Trade-off between performances and time consumption for \FlowAlign~w.r.t. tree slices in \texttt{qm7}.}
  \label{fg:exp_quantum_FA_slice}
 \vspace{-15pt}
\end{wrapfigure}

Similar to tree metric sampling for TSW \cite{le2019tree}, we also observe that the clustering-based tree metric sampling for \FlowAlign~and \DepthAlign~is fast and its time consumption is negligible compared to that of either \FlowAlign~or \DepthAlign. For examples, for each tree metric sampling with the suggested parameters (e.g., the predefined deepest level $H_{\Tt}=6$, and the number of clusters $\kappa=4$ for the farthest-point clustering), it only took about $0.4, 1.5, 11.0, 17.5, 20.5$ seconds for \texttt{qm7}, \texttt{TWITTER}, \texttt{RECIPE}, \texttt{CLASSIC}, and \texttt{AMAZON} datasets respectively.

Many further experimental results about performances and time consumptions of the discrepancies with different parameters (e.g., entropic regularization in EGW, and number of (tree) slices in SGW, \FlowAlign~and \DepthAlign), and clustering-based tree metric sampling with different parameters (i.e., $H_{\Tt}, \kappa$); empirical relations among the discrepancies for all mentioned datasets can be seen in the supplementary (\S D, \S G).





\subsection{Large-scale \FlowTGW~barycenter within $k$-means clustering}

We applied \FlowAlign~barycenter (\S\ref{subsec:align_flow_tree_gromov}), using Algorithm $2$ in \cite{Cuturi-2014-Fast} where we set $\texttt{k}=100$ for the maximum number of supports in barycenters, into a larger machine learning pipeline such as $k$-means clustering on \texttt{MNIST} dataset where point clouds of handwritten digits are \textit{rotated arbitrarily} in the plane as in \cite{peyre2016gromov}. For each handwritten digit, we randomly extracted $6000$ point clouds. We evaluated $k$-means with \FlowTGW~for $60000, 120000, 240000, 480000$, and $960000$ handwritten-digit point clouds where each handwritten digit is \textit{randomly rotated} $1, 2, 4, 8$, and $16$ times respectively. Furthermore, we grouped the handwritten digit $6$ and digit $9$ together due to applying random rotation. We used $k$-means++ initialization technique \cite{arthur2007k}, set $20$ for the maximum iterations of $k$-means, and repeated $10$ times with different random seeds for $k$-means++ initialization. In Fig.~\ref{fg:kmeans_flowtreeGW}, we show the averaged time consumption and $F_{\beta}$ measure~\cite{manning2008introduction} where $\beta$ is chosen as in~\cite{le2015unsupervised} for the results of $k$-means clustering with \FlowAlign. Note that, in these settings, the barycenter problem from EGW has extremely slow running time. A small experimental setup for performance comparison can be found in the supplementary (\S D).

\begin{figure}
 \vspace{-20pt}
  \begin{center}
    \includegraphics[width=0.45\textwidth]{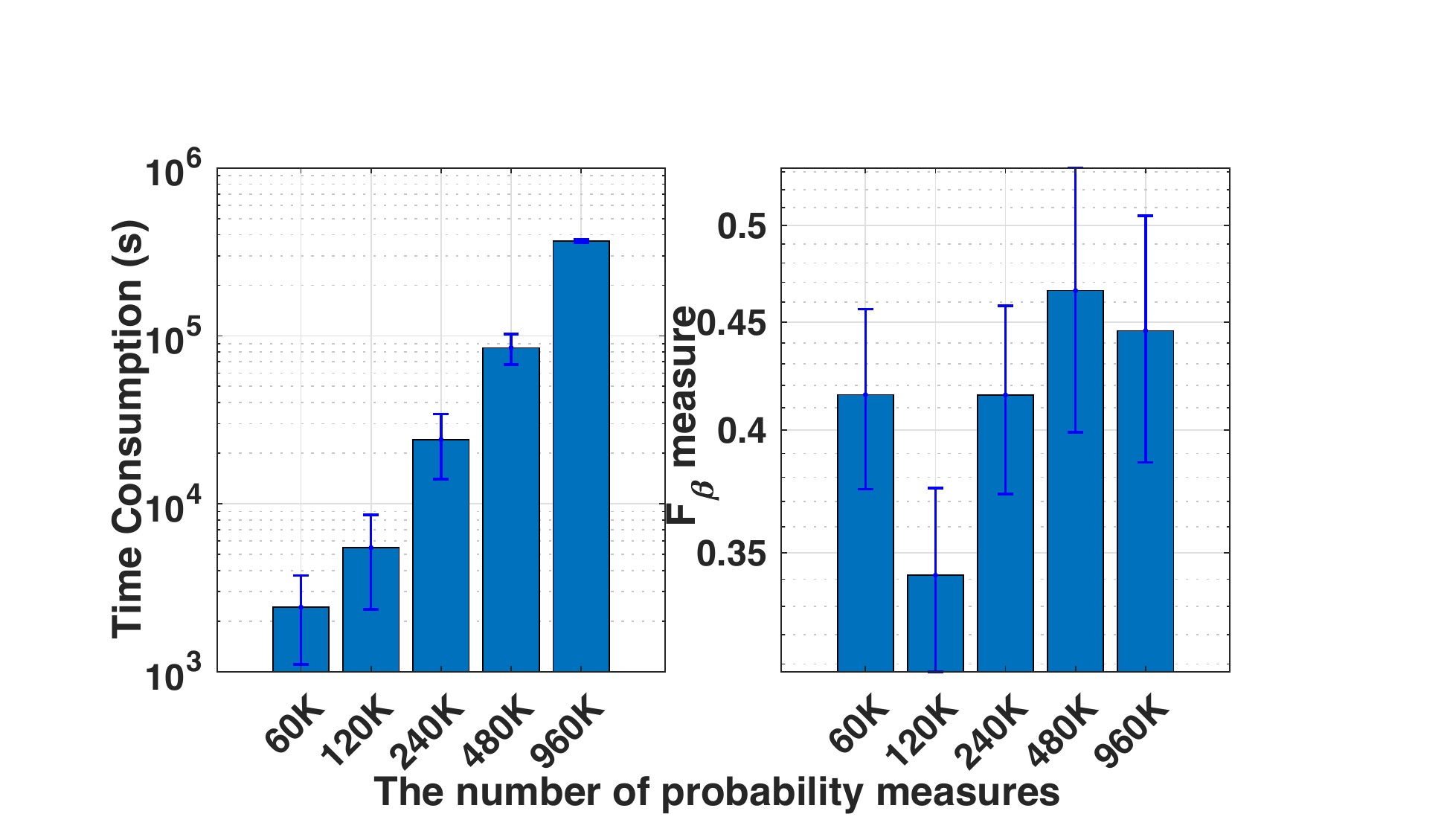}
  \end{center}
  \vspace{-12pt}
  \caption{Time consumption and $F_{\beta}$ measure for $k$-means clustering with \FlowAlign~for \textit{randomly rotated} \texttt{MNIST}.}
  \label{fg:kmeans_flowtreeGW}
  \vspace{-14pt}
\end{figure}

\section{Conclusion}
\label{sec:discuss}
We proposed in this paper two novel discrepancies \FlowAlign~and \DepthAlign~for probability measures whose supports are in different metric spaces by considering a particular family of cost metrics, namely tree metrics. By leverage a tree structure, we proposed to align flows from a root to each support instead of pair-wise tree metrics of supports in GW for probability measures. The proposed \FlowAlign~is not only fast, but its performances also compare favorably with other baseline approaches. Moreover, the \FlowAlign~can be applied for large-scale applications (e.g., a million probability measures) which are usually prohibited for (entropic) GW. The questions about sampling efficiently tree metrics from support data points for the proposed discrepancies, or using them for more involved parametric inference are left for the future work.

\balance
\bibliographystyle{plain}
\bibliography{neurips2020}

\begin{thebibliography}{10}

\bibitem{ahuja1988network}
Ravindra~K Ahuja, Thomas~L Magnanti, and James~B Orlin.
\newblock {\em Network flows}.
\newblock Cambridge, Mass.: Alfred P. Sloan School of Management,
  Massachusetts, 1988.

\bibitem{Altschuler-2017-Near}
J.~Altschuler, J.~Weed, and P.~Rigollet.
\newblock Near-linear time approximation algorithms for optimal transport via
  {S}inkhorn iteration.
\newblock In {\em Advances in neural information processing systems}, pages
  1964--1974, 2017.

\bibitem{altschuler2019massively}
Jason Altschuler, Francis Bach, Alessandro Rudi, and Jonathan Niles-Weed.
\newblock Massively scalable {S}inkhorn distances via the {N}ystr{\"o}m method.
\newblock In {\em Advances in Neural Information Processing Systems}, pages
  4429--4439, 2019.

\bibitem{Alvarez-2018-Gromov}
David Alvarez-Melis and Tommi Jaakkola.
\newblock Gromov-{W}asserstein alignment of word embedding spaces.
\newblock In {\em Proceedings of the Conference on Empirical Methods in Natural
  Language Processing (EMNLP)}, pages 1881--1890, 2018.

\bibitem{Arjovsky-2017-Wasserstein}
M.~Arjovsky, S.~Chintala, and L.~Bottou.
\newblock Wasserstein generative adversarial networks.
\newblock In {\em International conference on machine learning}, pages
  214--223, 2017.

\bibitem{arthur2007k}
David Arthur and Sergei Vassilvitskii.
\newblock k-means++: The advantages of careful seeding.
\newblock In {\em Proceedings of the eighteenth annual ACM-SIAM symposium on
  Discrete algorithms}, pages 1027--1035, 2007.

\bibitem{bartal1996probabilistic}
Yair Bartal.
\newblock Probabilistic approximation of metric spaces and its algorithmic
  applications.
\newblock In {\em Proceedings of 37th Conference on Foundations of Computer
  Science}, pages 184--193, 1996.

\bibitem{bartal1998approximating}
Yair Bartal.
\newblock On approximating arbitrary metrices by tree metrics.
\newblock In {\em ACM Symposium on Theory of Computing (STOC)}, volume~98,
  pages 161--168, 1998.

\bibitem{bhushan2018deepjdot}
Bharath Bhushan~Damodaran, Benjamin Kellenberger, R{\'e}mi Flamary, Devis Tuia,
  and Nicolas Courty.
\newblock Deepjdot: Deep joint distribution optimal transport for unsupervised
  domain adaptation.
\newblock In {\em Proceedings of the European Conference on Computer Vision
  (ECCV)}, pages 447--463, 2018.

\bibitem{bonneel2016wasserstein}
Nicolas Bonneel, Gabriel Peyr{\'e}, and Marco Cuturi.
\newblock Wasserstein barycentric coordinates: histogram regression using
  optimal transport.
\newblock {\em ACM Trans. Graph.}, 35(4):71--1, 2016.

\bibitem{bunne2019learning}
Charlotte Bunne, David Alvarez-Melis, Andreas Krause, and Stefanie Jegelka.
\newblock Learning generative models across incomparable spaces.
\newblock In {\em International Conference on Machine Learning}, 2019.

\bibitem{charikar1998approximating}
Moses Charikar, Chandra Chekuri, Ashish Goel, Sudipto Guha, and Serge Plotkin.
\newblock Approximating a finite metric by a small number of tree metrics.
\newblock In {\em Proceedings 39th Annual Symposium on Foundations of Computer
  Science (FOCS)}, pages 379--388, 1998.

\bibitem{courty2017joint}
Nicolas Courty, R{\'e}mi Flamary, Amaury Habrard, and Alain Rakotomamonjy.
\newblock Joint distribution optimal transportation for domain adaptation.
\newblock In {\em Advances in Neural Information Processing Systems}, pages
  3730--3739, 2017.

\bibitem{courty2016optimal}
Nicolas Courty, R{\'e}mi Flamary, Devis Tuia, and Alain Rakotomamonjy.
\newblock Optimal transport for domain adaptation.
\newblock {\em Pattern analysis and machine intelligence (PAMI)},
  39(9):1853--1865, 2016.

\bibitem{Cuturi-2013-Sinkhorn}
M.~Cuturi.
\newblock Sinkhorn distances: {L}ightspeed computation of optimal transport.
\newblock In {\em Advances in Neural Information Processing Systems}, pages
  2292--2300, 2013.

\bibitem{Cuturi-2014-Fast}
M.~Cuturi and A.~Doucet.
\newblock Fast computation of {W}asserstein barycenters.
\newblock In {\em International conference on machine learning}, pages
  685--693, 2014.

\bibitem{Dvurechensky-2018-Computational}
P.~Dvurechensky, A.~Gasnikov, and A.~Kroshnin.
\newblock Computational optimal transport: {C}omplexity by accelerated gradient
  descent is better than by {S}inkhorn’s algorithm.
\newblock In {\em International conference on machine learning}, pages
  1367--1376, 2018.

\bibitem{fakcharoenphol2004tight}
Jittat Fakcharoenphol, Satish Rao, and Kunal Talwar.
\newblock A tight bound on approximating arbitrary metrics by tree metrics.
\newblock {\em Journal of Computer and System Sciences}, 69(3):485--497, 2004.

\bibitem{feder1988optimal}
Tomas Feder and Daniel Greene.
\newblock Optimal algorithms for approximate clustering.
\newblock In {\em Proceedings of the twentieth annual ACM symposium on Theory
  of computing}, pages 434--444. ACM, 1988.

\bibitem{genevay2016stochastic}
Aude Genevay, Marco Cuturi, Gabriel Peyr{\'e}, and Francis Bach.
\newblock Stochastic optimization for large-scale optimal transport.
\newblock In {\em Advances in neural information processing systems}, pages
  3440--3448, 2016.

\bibitem{pmlr-v84-genevay18a}
Aude Genevay, Gabriel Peyre, and Marco Cuturi.
\newblock Learning generative models with sinkhorn divergences.
\newblock In {\em Proceedings of the Twenty-First International Conference on
  Artificial Intelligence and Statistics}, pages 1608--1617, 2018.

\bibitem{gonzalez1985clustering}
Teofilo~F Gonzalez.
\newblock Clustering to minimize the maximum intercluster distance.
\newblock {\em Theoretical Computer Science}, 38:293--306, 1985.

\bibitem{grave2019unsupervised}
Edouard Grave, Armand Joulin, and Quentin Berthet.
\newblock Unsupervised alignment of embeddings with {W}asserstein procrustes.
\newblock In {\em International Conference on Artifi- cial Intelligence and
  Statistics (AISTATS)}, pages 1880--1890, 2019.

\bibitem{Gulrajani-2017-Improved}
I.~Gulrajani, F.~Ahmed, M.~Arjovsky, V.~Dumoulin, and A.~C. Courville.
\newblock Improved training of {W}asserstein {GAN}s.
\newblock In {\em Advances in Neural Information Processing Systems}, pages
  5767--5777, 2017.

\bibitem{indyk2001algorithmic}
Piotr Indyk.
\newblock Algorithmic applications of low-distortion geometric embeddings.
\newblock In {\em Proceedings 42nd IEEE Symposium on Foundations of Computer
  Science (FOCS)}, pages 10--33, 2001.

\bibitem{kolouri2018sliced}
Soheil Kolouri, Phillip~E. Pope, Charles~E. Martin, and Gustavo~K. Rohde.
\newblock Sliced wasserstein auto-encoders.
\newblock In {\em International Conference on Learning Representations}, 2019.

\bibitem{kusner2015word}
Matt Kusner, Yu~Sun, Nicholas Kolkin, and Kilian Weinberger.
\newblock From word embeddings to document distances.
\newblock In {\em International conference on machine learning}, pages
  957--966, 2015.

\bibitem{lavenant2018dynamical}
Hugo Lavenant, Sebastian Claici, Edward Chien, and Justin Solomon.
\newblock Dynamical optimal transport on discrete surfaces.
\newblock In {\em SIGGRAPH Asia 2018 Technical Papers}, page 250. ACM, 2018.

\bibitem{le2015unsupervised}
Tam Le and Marco Cuturi.
\newblock Unsupervised {R}iemannian metric learning for histograms using
  {A}itchison transformations.
\newblock In {\em International Conference on Machine Learning}, pages
  2002--2011, 2015.

\bibitem{le2019tree}
Tam Le, Makoto Yamada, Kenji Fukumizu, and Marco Cuturi.
\newblock Tree-sliced variants of {W}asserstein distances.
\newblock In {\em Advances in neural information processing systems}, pages
  12283--12294, 2019.

\bibitem{Lin-2019-Efficient}
Tianyi Lin, Nhat Ho, and Michael Jordan.
\newblock On efficient optimal transport: An analysis of greedy and accelerated
  mirror descent algorithms.
\newblock In {\em Proceedings of the 36th International Conference on Machine
  Learning}, pages 3982--3991, 2019.

\bibitem{pmlr-v97-liutkus19a}
Antoine Liutkus, Umut Simsekli, Szymon Majewski, Alain Durmus, and
  Fabian-Robert St{\"o}ter.
\newblock Sliced-{W}asserstein flows: Nonparametric generative modeling via
  optimal transport and diffusions.
\newblock In {\em Proceedings of the 36th International Conference on Machine
  Learning}, pages 4104--4113, 2019.

\bibitem{NIPS2019_9130}
Giulia Luise, Saverio Salzo, Massimiliano Pontil, and Carlo Ciliberto.
\newblock Sinkhorn barycenters with free support via {F}rank-{W}olfe algorithm.
\newblock In {\em Advances in Neural Information Processing Systems}, pages
  9318--9329, 2019.

\bibitem{manning2008introduction}
Christopher~D Manning, Prabhakar Raghavan, and Hinrich Sch{\"u}tze.
\newblock {\em Introduction to information retrieval}.
\newblock Cambridge university press, 2008.

\bibitem{Memoli_Gromov}
F.~M\'{e}moli.
\newblock Gromov–{W}asserstein distances and the metric approach to object
  matching.
\newblock {\em Foundations of Computational Mathematics}, 11(4):417–487,
  2011.

\bibitem{NIPS2019_8703}
Gonzalo Mena and Jonathan Niles-Weed.
\newblock Statistical bounds for entropic optimal transport: sample complexity
  and the central limit theorem.
\newblock In {\em Advances in Neural Information Processing Systems}, pages
  4543--4553, 2019.

\bibitem{mikolov2013distributed}
Tomas Mikolov, Ilya Sutskever, Kai Chen, Greg~S Corrado, and Jeff Dean.
\newblock Distributed representations of words and phrases and their
  compositionality.
\newblock In {\em Advances in neural information processing systems}, pages
  3111--3119, 2013.

\bibitem{muzellec2018generalizing}
Boris Muzellec and Marco Cuturi.
\newblock Generalizing point embeddings using the wasserstein space of
  elliptical distributions.
\newblock In {\em Advances in Neural Information Processing Systems}, pages
  10237--10248, 2018.

\bibitem{NIPS2019_8318}
Kimia Nadjahi, Alain Durmus, Umut Simsekli, and Roland Badeau.
\newblock Asymptotic guarantees for learning generative models with the
  sliced-{W}asserstein distance.
\newblock In {\em Advances in Neural Information Processing Systems}, pages
  250--260, 2019.

\bibitem{pmlr-v97-paty19a}
Fran{\c{c}}ois-Pierre Paty and Marco Cuturi.
\newblock Subspace robust {W}asserstein distances.
\newblock In {\em Proceedings of the 36th International Conference on Machine
  Learning}, pages 5072--5081, 2019.

\bibitem{perrot2016mapping}
Micha{\"e}l Perrot, Nicolas Courty, R{\'e}mi Flamary, and Amaury Habrard.
\newblock Mapping estimation for discrete optimal transport.
\newblock In {\em Advances in Neural Information Processing Systems}, pages
  4197--4205, 2016.

\bibitem{Peyre_Gromov}
G.~Peyr\'{e}, M.~Cuturi, and J.~Solomon.
\newblock Gromov-{W}asserstein averaging of kernel and distance matrices.
\newblock In {\em Proceedings of the International Conference on Machine
  Learning}, 2016.

\bibitem{peyre2019computational}
Gabriel Peyr{\'e} and Marco Cuturi.
\newblock Computational optimal transport.
\newblock {\em Foundations and Trends{\textregistered} in Machine Learning},
  11(5-6):355--607, 2019.

\bibitem{peyre2016gromov}
Gabriel Peyr{\'e}, Marco Cuturi, and Justin Solomon.
\newblock Gromov-wasserstein averaging of kernel and distance matrices.
\newblock In {\em International Conference on Machine Learning}, pages
  2664--2672, 2016.

\bibitem{pmlr-v89-redko19a}
Ievgen Redko, Nicolas Courty, R\'emi Flamary, and Devis Tuia.
\newblock Optimal transport for multi-source domain adaptation under target
  shift.
\newblock In {\em International Conference on Artificial Intelligence and
  Statistics}, pages 849--858, 2019.

\bibitem{rupp2012fast}
Matthias Rupp, Alexandre Tkatchenko, Klaus-Robert M{\"u}ller, and O~Anatole
  Von~Lilienfeld.
\newblock Fast and accurate modeling of molecular atomization energies with
  machine learning.
\newblock {\em Physical review letters}, 108(5):058301, 2012.

\bibitem{salton1988term}
Gerard Salton and Christopher Buckley.
\newblock Term-weighting approaches in automatic text retrieval.
\newblock {\em Information processing \& management}, 24(5):513--523, 1988.

\bibitem{SantambrogioBook}
Filippo Santambrogio.
\newblock {\em Optimal transport for applied mathematicians}.
\newblock Birk{\"a}user, 2015.

\bibitem{schmitzer2019stabilized}
Bernhard Schmitzer.
\newblock Stabilized sparse scaling algorithms for entropy regularized
  transport problems.
\newblock {\em SIAM Journal on Scientific Computing}, 41(3):A1443--A1481, 2019.

\bibitem{semple2003phylogenetics}
Charles Semple and Mike Steel.
\newblock Phylogenetics.
\newblock {\em Oxford Lecture Series in Mathematics and its Applications},
  2003.

\bibitem{solomon2015convolutional}
Justin Solomon, Fernando De~Goes, Gabriel Peyr{\'e}, Marco Cuturi, Adrian
  Butscher, Andy Nguyen, Tao Du, and Leonidas Guibas.
\newblock Convolutional {W}asserstein distances: Efficient optimal
  transportation on geometric domains.
\newblock {\em ACM Transactions on Graphics (TOG)}, 34(4):66, 2015.

\bibitem{solomon2016entropic}
Justin Solomon, Gabriel Peyr{\'e}, Vladimir~G Kim, and Suvrit Sra.
\newblock Entropic metric alignment for correspondence problems.
\newblock {\em ACM Transactions on Graphics (TOG)}, 35(4):72, 2016.

\bibitem{solomon2019optimal}
Justin Solomon and Amir Vaxman.
\newblock Optimal transport-based polar interpolation of directional fields.
\newblock {\em ACM Transactions on Graphics (TOG)}, 38(4):1--13, 2019.

\bibitem{NIPS2019_8872}
Matteo Togninalli, Elisabetta Ghisu, Felipe Llinares-L\'{o}pez, Bastian Rieck,
  and Karsten Borgwardt.
\newblock Wasserstein {W}eisfeiler-{L}ehman graph kernels.
\newblock In {\em Advances in Neural Information Processing Systems}, pages
  6436--6446, 2019.

\bibitem{vayer2019sliced}
Titouan Vayer, R{\'e}mi Flamary, Romain Tavenard, Laetitia Chapel, and Nicolas
  Courty.
\newblock Sliced {G}romov-{W}asserstein.
\newblock {\em Advances in Neural Information Processing Systems}, 2019.

\bibitem{Villani-03}
C\'edric Villani.
\newblock {\em Topics in Optimal Transportation}.
\newblock American Mathematical Society, 2003.

\bibitem{wu2019sliced}
Jiqing Wu, Zhiwu Huang, Dinesh Acharya, Wen Li, Janine Thoma, Danda~Pani
  Paudel, and Luc~Van Gool.
\newblock Sliced {W}asserstein generative models.
\newblock In {\em Proceedings of the IEEE conference on computer vision and
  pattern recognition}, pages 3713--3722, 2019.

\bibitem{Xu-2019-Scalable}
H.~Xu, D.~Luo, and L.~Carin.
\newblock Scalable {G}romov-{W}asserstein learning for graph partitioning and
  matching.
\newblock In {\em Advances in neural information processing systems}, 2019.

\bibitem{Xu-2019-Gromov}
H.~Xu, D.~Luo, H.~Zha, and L.~Carin.
\newblock Gromov-{W}asserstein learning for graph matching and node embedding.
\newblock In {\em International conference on machine learning}, 2019.

\end{thebibliography}



\newpage
\appendix

\begin{center}
\textbf{\Large{\textit{Supplementary Material for:} Flow-based Alignment Approaches \\ for Probability Measures in Different Spaces}}
\end{center}

We organize this supplementary material as follow:
\begin{itemize}
\item In Appendix~\ref{sec:proofs}, we provide proofs for theoretical results: Theorem 1, Theorem 2, and Proposition 1 in the main text.

\item In Appendix~\ref{sec:further_illustration}, we show more illustrations and geometric properties, and discuss about the rotational and translational invariance for Flow-based Alignment (\FlowTGW) mentioned in the main text.
\item In Appendix~\ref{sec:detail_algorithm}, we describe further details for \FlowTGW~and Depth-based Alignment (\DepthTGW), e.g., algorithms and complexity for \FlowAlign~and \DepthAlign~on applications with or without priori knowledge about tree structures for probability measures.

\item In Appendix~\ref{sec:further_experimental_results}, we illustrate 
	\begin{itemize}
	\item further experimental results in quantum chemistry (\texttt{qm7} dataset), document classification (\texttt{TWITTER, RECIPE, CLASSIC, AMAZON} datasets) considered in the main text;
	\item time consumption for the clustering-based tree metric sampling; 
	\item results with different parameters for tree metric sampling;
	\item and a small experimental setup for performance comparison for $k$-means clustering on \textit{randomly rotated} \texttt{MNIST} dataset.
	\end{itemize}

\item In Appendix~\ref{sec:brief_review}, we give some brief reviews for
	\begin{itemize}
	\item the farthest-point clustering; 
	\item clustering-based tree metric sampling; 
	\item tree metric;
	\item $F_{\beta}$ measure for clustering evaluation; 
	\item and more information for datasets.
	\end{itemize}
\item In Appendix~\ref{sec:further_discussions}, we provide some further discussions.

\item In Appendix~\ref{sec:empirical_relation}, we investigate empirical relations among the considered discrepancies (e.g., \FlowTGW, \DepthTGW, sliced GW (SGW), entropic GW (EGW) and the standard entropic GW where entropic regularization is used in both transportation plan optimization and objective function computation (EGW$_0$)) for probability measures in different spaces.
\end{itemize}

\textbf{Notations.} We use same notations as in the main text.  

\section{Proofs}
\label{sec:proofs}
In this section, we provide the proofs for the pseudo-distances and properties of \FlowTGW~and \DepthTGW~discrepancies, i.e., Theorem 1, Theorem 2, and Proposition 1 in the main text.

\subsection{Proof of Theorem 1 in the main text}
\label{subsec:proof:prop:triangle_flow_based}

From the definition of \FlowTGW~$\FA$, it is symmetric, namely, $\FA(\mu, \nu) = \FA(\nu, \mu)$. In addition, it is clear that $\FA(\mu, \mu) = 0$. Finally, we show that $\FA$ also satisfies triangle inequality as in Proposition~\ref{prop:triangle_flow_based}.


\begin{proposition}
\label{prop:triangle_flow_based}
Given three probability measures $\mu, \nu, \gamma$ in three different metric spaces $(\Tt_X, d_{\TMX})$, $(\Tt_Y, d_{\TM_Y})$, and $(\Tt_Z, d_{\TMZ})$. Then, we have:
\[
    \FA(\mu, \gamma) \leq \FA(\mu, \nu) + \FA(\nu, \gamma). 
\]
\end{proposition}

\begin{proof}
It is sufficient to demonstrate that

\begin{align}
    \label{eq:flow_ine}
    \RFFA(\mu, \gamma; r_x, r_y) \leq \RFFA(\mu, \nu; r_x, r_z) + \RFFA(\nu, \gamma; r_y, r_z),
\end{align}

for any roots $r_x, r_y, r_z$ of $\Tt_X, \Tt_Y, \Tt_Z$ respectively. Our proof for the above inequality is a direct application of the gluing lemma in~\citep{Villani-03}. In particular, for any roots $r_{x}, r_{y}, r_{z}$ of of $\Tt_X, \Tt_Y, \Tt_Z$, we denote $\widehat{T}^{1} \in \Pi(\mu, \nu)$ and $\widehat{T}^{2} \in \Pi(\nu, \gamma)$ as optimal transport plans for $\RFFA(\mu, \nu; r_x, r_z)$ and $\RFFA(\nu, \gamma; r_y, r_z)$ respectively. Based on the gluing lemma, there exists $T$ with marginal of the first and the third factors as $\widehat{T}^{1}$ and marginal of the second and the third factors as $\widehat{T}^{2}$. We denote the marginal of its first and second factors as $\bar{T}$, which is a transport plan between $\mu$ and $\gamma$. Therefore, from the definition of aligned-root \FlowTGW~discrepancy, we have

\begin{align}
    \RFFA^2(\mu, \gamma; r_x, r_y) & \leq \sum_{i, j} \left|d_{\TM_X}(r_x, x_i) - d_{\TM_Y}(r_y, y_j) \right|^2 \bar{T}_{ij} \nonumber \\
    & = \sum_{i, j, k} \left|d_{\TM_X}(r_x, x_i) - d_{\TM_Y}(r_y, y_j) \right|^2 T_{ijk} \nonumber\\
    & = \sum_{i, j, k} \left|d_{\TM_X}(r_x, x_i) - d_{\TM_Z}(r_z, z_k) \right|^2 T_{ijk} + \sum_{i, j, k} \left|d_{\TM_X}(r_y, y_j) - d_{\TM_Z}(r_z, z_k) \right|^2 T_{ijk} \nonumber\\ 
    & \hspace{ 5 em} - 2 \sum_{i, j, k} (d_{\TM_X}(r_x, x_i) - d_{\TM_Z}(r_z, z_k))(d_{\TM_Y}(r_y, y_j) - d_{\TM_Z}(r_z, z_k)) T_{ijk}  \nonumber\\
    & \leq \RFFA^2(\mu, \nu; r_x, r_z) + \RFFA^2(\nu, \gamma; r_y, r_z) \nonumber \\
    & \hspace{ 5 em} + 2 \RFFA(\mu, \nu; r_x, r_z) \RFFA(\nu, \gamma; r_y, r_z) \nonumber \\
    & = (\RFFA(\mu, \nu; r_x, r_z) + \RFFA(\nu, \gamma; r_y, r_z))^2,
\end{align}

where we used H\"older's inequality for the third term for the second inequality. As a consequence, we obtain the conclusion of the inequality in Equation~\eqref{eq:flow_ine}.
\end{proof}

\paragraph{Discussion about $\SFA(\mu, \nu) = 0$.} As discussed in the main text, when $\SFA(\mu, \nu) = 0$, we can find roots $r_{x}^{*}$ and $r_{z}^{*}$ such that $\tilde{\mu}^{*} \equiv \tilde{\nu}^{*}$ where $\tilde{\mu}^{*} = \sum_{i} a_{i} \delta_{d_{\TMX}(r_{\!x}^{*}, x_i)}$ and $\tilde{\nu}^{*} = \sum_j b_j \delta_{d_{\TMZ}(r_{\!z}^{*}, z_j)}$. It demonstrates that $\mu$ and $\nu$ have the same weights on supports (i.e., flow masses) while the tree metrics of their supports to the corresponding root $r_{x}^{*}$ or $r_{z}^{*}$ (i.e., flow lengths) are identical. 


\subsection{Proof of Theorem 2 in the main text}
\label{subsec:proof:prop:triangle_depth_based}

In fact, from the definition of \DepthTGW~$\DA$, it is clear that $\DA(\mu, \nu) = \DA(\nu, \mu)$ and $\DA(\mu, \mu) = 0$. Furthermore, $\DA$ satisfies triangle inequality as in Proposition~\ref{prop:triangle_depth_based}.

\begin{proposition}
\label{prop:triangle_depth_based}
Given three probability measures $\mu, \nu, \gamma$ in three different metric spaces $(\Tt_X, d_{\TMX})$, $(\Tt_Y, d_{\TM_Y})$, and $(\Tt_Z, d_{\TMZ})$. Then, we have:
\begin{align*}
    \DA(\mu, \gamma) \leq \DA(\mu, \nu) + \DA(\nu, \gamma). 
\end{align*}
\end{proposition}

\begin{proof}
Similar to the proof of Proposition~\ref{prop:triangle_flow_based}, it is sufficient to demonstrate that

\begin{align}
    \label{eq:depth_ine}
    \RFDA(\mu, \gamma; r_x, r_y) \leq \RFDA(\mu, \nu; r_x, r_z) + \RFDA(\nu, \gamma; r_y, r_z),
\end{align}

for any roots $r_x, r_y, r_z$ of $\Tt_X, \Tt_Y, \Tt_Z$ respectively. According to the definition of aligned-root \DepthTGW, the above inequality is equivalent to

\begin{align}
    \hspace{-2 em}\sum_{h} \sum_{(x, y) \in \Mm_{h-1}^{1,2}} T^{*}_{h}(x, y) \RFFA(\mu_{\Tt_{x}^{2}}, \gamma_{\Tt_{y}^{2}}; x, y) \leq \sum_{h} \sum_{(x, z) \in \Mm_{h-1}^{1,3}} \bar{T}^{*}_{h}(x, z) \RFFA(\mu_{\Tt_{x}^{2}}, \nu_{\Tt_{z}^{2}}; x, z) \nonumber \\
    & \hspace{- 22 em} + \sum_{h} \sum_{(y, z) \in \Mm_{h-1}^{2,3}} \tilde{T}^{*}_{h}(y, z) \RFFA(\mu_{\Tt_{y}^{2}}, \nu_{\Tt_{z}^{2}}; y, z),
\end{align}

where $\Mm_{h}^{1,2}, \Mm_{h}^{1,3}, \Mm_{h}^{2,3}$ are respectively sets of optimal aligned pairs at the depth level $h$ from trees $\Tt_X$ and $\Tt_Y$, from trees $\Tt_X$ and $\Tt_Z$, and from trees $\Tt_Y$ and $\Tt_Z$; $T^{*}_{h}(x, y), \bar{T}^{*}_{h}(x, z), \tilde{T}^{*}_{h}(y, z)$ are respectively optimal matching masses for the pairs $(x, y) \in \Mm_{h - 1}^{1,2}, (x, z) \in \Mm_{h-1}^{1,3}, (y, z) \in \Mm_{h-1}^{2,3}$. In order to demonstrate the above inequality, we only need to verify that

\begin{align}
    \sum_{(x, y) \in \Mm_{h}^{1,2}} T^{*}_{h}(x, y) \RFFA(\mu_{\Tt_{x}^{2}}, \gamma_{\Tt_{y}^{2}}; x, y) \leq \sum_{(x, z) \in \Mm_{h}^{1,3}} \bar{T}^{*}_{h}(x, z) \RFFA(\mu_{\Tt_{x}^{2}}, \nu_{\Tt_{z}^{2}}; x, z) \nonumber \\
    & \hspace{- 17 em} + \sum_{(y, z) \in \Mm_{h}^{2,3}} \tilde{T}^{*}_{h}(y, z) \RFFA(\mu_{\Tt_{y}^{2}}, \nu_{\Tt_{z}^{2}}; y, z), \label{eq:key_depth}
\end{align}

for any depth level $h \geq 1$. We respectively denote $\Tt_X^{h} = \{\bar{x}_{1}^{(h)}, \ldots, \bar{x}_{k_{1,h}}^{(h)}\}, \Tt_Y^{h} = \{\bar{y}_{1}^{(h)}, \ldots, \bar{y}_{k_{2,h}}^{(h)}\}, \Tt_Z^{h} = \{\bar{z}_{1}^{(h)}, \ldots, \bar{z}_{k_{3,h}}^{(h)}\}$ the set of nodes in depth level $h$ of the trees $\Tt_X, \Tt_Y, \Tt_Z$. The inequality in Equation~\eqref{eq:key_depth} can be rewritten as

\begin{align}
    & \hspace{-5 em} \sum_{i = 1}^{k_{1, h - 1}} \sum_{j = 1}^{k_{2, h - 1}} \sum_{x \in \Ss(\bar{x}_{i}^{(h - 1)}), y \in \Ss(\bar{y}_{j}^{(h - 1)})} T^{*}_{h}(x, y) \RFFA(\mu_{\Tt_{x}^{2}}, \gamma_{\Tt_{y}^{2}}; x, y) \nonumber \\
    & \leq  \sum_{i = 1}^{k_{1, h - 1}} \sum_{j = 1}^{k_{3, h - 1}} \sum_{x \in \Ss(\bar{x}_{i}^{(h - 1)}), z \in \Ss(\bar{z}_{j}^{(h - 1)})} \bar{T}^{*}_{h}(x, z) \RFFA(\mu_{\Tt_{x}^{2}}, \nu_{\Tt_{y}^{2}}; x, z) \nonumber \\
    & \qquad \qquad +  \sum_{i = 1}^{k_{2, h - 1}} \sum_{j = 1}^{k_{3, h - 1}} \sum_{y \in \Ss(\bar{y}_{i}^{(h - 1)}), z \in \Ss(\bar{z}_{j}^{(h - 1)})} \tilde{T}^{*}_{h}(y, z) \RFFA(\gamma_{\Tt_{y}^{2}}, \nu_{\Tt_{y}^{2}}; y, z). \label{eq:equiv_depth}
\end{align}

In order to obtain the conclusion of inequality in Equation~\eqref{eq:equiv_depth}, we only need to prove that 

\begin{align*}
   & \hspace{-8 em} \sum_{x \in \Ss(\bar{x}_{i}^{(h - 1)}), y \in \Ss(\bar{y}_{j}^{(h - 1)})} T^{*}_{h}(x, y) \RFFA(\mu_{\Tt_{x}^{2}}, \gamma_{\Tt_{y}^{2}}; x, y) \\
   & \leq \sum_{x \in \Ss(\bar{x}_{i}^{(h - 1)}), z \in \Ss(\bar{z}_{l}^{(h - 1)})} \bar{T}^{*}_{h}(x, z) \RFFA(\mu_{\Tt_{x}^{2}}, \nu_{\Tt_{y}^{2}}; x, z) \\
   & \qquad + \sum_{y \in \Ss(\bar{y}_{i}^{(h - 1)}), z \in \Ss(\bar{z}_{l}^{(h - 1)})} \tilde{T}^{*}_{h}(y, z) \RFFA(\gamma_{\Tt_{y}^{2}}, \nu_{\Tt_{y}^{2}}; y, z),
\end{align*}

for any $i \in \{1, \ldots, k_{1, h - 1}\}$, $j \in \{1, \ldots, k_{2, h - 1} \}$, and $l \in \{1, \ldots, k_{3, h - 1}\}$. We can make use of the gluing lemma~\citep{Villani-03} to prove that inequality. In fact, there exists $T$ with marginal of its first and third factors as $\bar{T}^{*}_{h}(x, z)$ and marginal of its second and third factors as $\bar{T}^{*}_{h}(y, z)$. We denote the marginal of its first and second factors as $\hat{T}$, which is the transport plan for probability measures with supports at $x \in \Ss(\bar{x}_{i}^{(h - 1)})$ and $y \in \Ss(\bar{y}_{j}^{(h - 1)})$. Now, we have the following inequalities

\begin{align*}
   &\hspace{-6 em} \sum_{x \in \Ss(\bar{x}_{i}^{(h - 1)}), y \in \Ss(\bar{y}_{j}^{(h - 1)})} T^{*}_{h}(x, y) \RFFA(\mu_{\Tt_{x}^{2}}, \gamma_{\Tt_{y}^{2}}; x, y) \\
   & = \sum_{x \in \Ss(\bar{x}_{i}^{(h - 1)}), y \in \Ss(\bar{y}_{j}^{(h - 1)}), z \in \Ss(\bar{z}_{l}^{(h - 1)})} T_{xyz} \RFFA(\mu_{\Tt_{x}^{2}}, \gamma_{\Tt_{y}^{2}}; x, y) \\
    & \leq \sum_{x \in \Ss(\bar{x}_{i}^{(h - 1)}), y \in \Ss(\bar{y}_{j}^{(h - 1)}), z \in \Ss(\bar{z}_{l}^{(h - 1)})} T_{xyz} \RFFA(\mu_{\Tt_{x}^{2}}, \nu_{\Tt_{z}^{2}}; x, z) \\
    & \qquad \qquad + \sum_{x \in \Ss(\bar{x}_{i}^{(h - 1)}), y \in \Ss(\bar{y}_{j}^{(h - 1)}), z \in \Ss(\bar{z}_{l}^{(h - 1)})} T_{xyz} \RFFA(\gamma_{\Tt_{y}^{2}}, \nu_{\Tt_{z}^{2}}; y, z) \\
    & = \sum_{x \in \Ss(\bar{x}_{i}^{(h - 1)}), z \in \Ss(\bar{z}_{l}^{(h - 1)})} \bar{T}^{*}_{h}(x, z) \RFFA(\mu_{\Tt_{x}^{2}}, \nu_{\Tt_{y}^{2}}; x, z) \\
    & \qquad  \qquad + \sum_{y \in \Ss(\bar{y}_{i}^{(h - 1)}), z \in \Ss(\bar{z}_{l}^{(h - 1)})} \tilde{T}^{*}_{h}(y, z) \RFFA(\gamma_{\Tt_{y}^{2}}, \nu_{\Tt_{y}^{2}}; y, z).
\end{align*}

As a consequence, we obtain the conclusion of the proposition.
\end{proof}

\paragraph{Discussion about $\DA(\mu, \nu)=0$.} As discussed in the main text, when $\DA(\mu, \nu)=0$, we can find roots $r_x^{*}$ and $r_z^{*}$ such that all the hierarchical corresponding $\RFFA(\mu_{\Tt_{\cdot}^{2}}, \nu_{\Tt_{\cdot}^{2}}; \cdot, \cdot)$ for each depth level along the trees are equal to $0$. It demonstrates that $\mu$ and $\nu$ have the same weights on supports while their supports have the same depth levels. For each depth level, the tree metrics of supports in the corresponding $\mu_{\Tt_{\cdot}^{2}}, \nu_{\Tt_{\cdot}^{2}}$ to the 2-depth-level-tree roots are identical, i.e., corresponding weight edges are identical.

\subsection{Proof of Proposition 1 in the main text}
\label{subsec:proof:prop:flow_to_Gromov}


The proof of Proposition 1 in the main text is a direct application of Cauchy-Schwarz. In particular, since the deepest levels of trees $\Tt_{X}$ and $\Tt_{Z}$ are equal to two for some roots $r_{x}$ and $r_{z}$ respectively, we obtain that
\begin{align}\label{eq:GW_FlowAlign}
    \GW^2(\mu, \nu) & =  \min_{T \in \Pi(\mu, \nu)} \sum_{i,j,i', j'} \left| d_{\TMX}(x_i, x_{i'}) - d_{\TMZ}(z_j, z_{j'} )\right|^2 T_{ij} T_{i'j'} \nonumber \\
    & = \min_{T \in \Pi(\mu, \nu)} \sum_{i,j,i', j'} \left| d_{\TMX}(x_i, r_{x}) + d_{\TMX}(r_{x}, x_{i'}) - d_{\TMZ}(z_j,r_{z}) - d_{\TMZ}(r_{z}, z_{j'} )\right|^2 T_{ij} T_{i'j'} \nonumber \\
    & \leq \min_{T \in \Pi(\mu, \nu)} \sum_{i,j,i', j'} 2 \bigr[(d_{\TMX}(x_i, r_{x}) - d_{\TMZ}(z_j,r_{z}))^2 + (d_{\TMX}(r_{x}, x_{i'}) - d_{\TMZ}(r_{z}, z_{j'} ))^2 \bigr] T_{ij} T_{i'j'} \nonumber \\
    & = 4\SRFFA(\mu, \nu; r_x, r_z)
\end{align}
where the inequality is due to the standard Cauchy-Schwarz inequality $(a + b)^2 \leq 2(a^2 + b^2)$ for any $a, b \in \mathbb{R}$. By taking the infimum over $r_{x}, r_{z}$, then square root on both sides of the above inequality \eqref{eq:GW_FlowAlign}, we obtain the conclusion of the proposition that 
\[
\GW(\mu, \nu) \le 2\FA(\mu, \nu).
\]

\section{Further illustrations, geometric properties, and discussion about rotational and translational invariance for \FlowTGW}
\label{sec:further_illustration}

In this section, we provide further illustrations for \FlowAlign~mentioned in the main text. 
\begin{itemize}
\item In Figure~\ref{fg:FlowBased_GTW}, we illustrate the aligned-root \FlowAlign~$\RFFA$.
\item In Figure~\ref{fg:DP_FlowTGW_Case2}, we illustrate for supports in $\Omega_{\nu}$ in the \textbf{Case 2} in the efficient computation approach for \FlowAlign~$\FA$.

\end{itemize}

\paragraph{Rotational and translational invariance for \FlowAlign.} In practice, we usually do not have priori knowledge about tree structures for probability measures\footnote{For a priori tree metric space, recall that tree metric space is finite. So, rotation/translation may not be directly well-defined in tree metric space.}. Therefore, we need to choose or sample trees $\Tt_X$ and $\Tt_Z$ from support data points, e.g. by clustering-based tree metric sampling \cite{le2019tree}. 

For the clustering-based tree metric sampling\footnote{see Appendix~\ref{sec:clustering_based_TM} for a review}, the farthest-point clustering\footnote{see Appendix~\ref{sec:farthest_point_clustering} for a review} within the clustering-based tree metric sampling gives the same results for rotational and/or translational support data points and for the original ones, when one uses the same corresponding point in the given finite set of support data points as its initialization). Therefore, tree metric sampled from the clustering-based tree metric sampling is rotational and translational invariance (i.e., tree structure and lengths of edges are the same, only nodes are represented for the corresponding rotational and/or translational support data points instead of the original ones.). Consequently, the \FlowAlign~has rotational and translational invariance. As showed in the main text (\S6.1), the \FlowAlign~works well with the quantum chemistry (in a real dataset: \texttt{qm7}) where one needs translational and rotational invariance for the relative positions of atoms in $\RR^3$ for each molecule.

As a trivial extension, due to rotational and translation invariance for tree metrics sampled from the clustering-based tree metric sampling, \DepthAlign~also has rotational and translational invariance.
 
 \begin{figure}
  \begin{center}
    \includegraphics[width=0.7\textwidth]{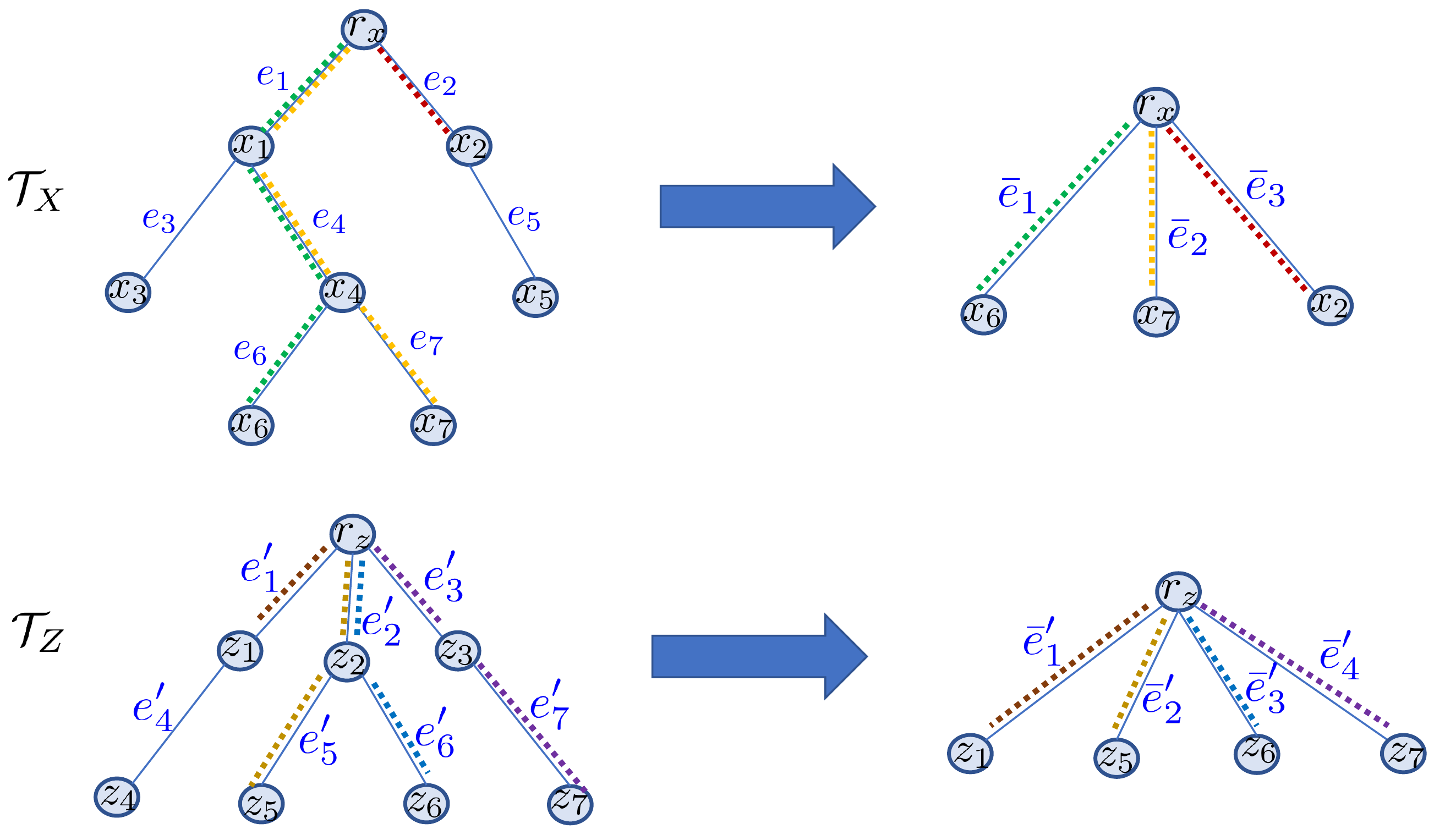}
  \end{center}
  \caption{An illustration for the aligned-root \FlowTGW~$\RFFA$ between $\mu = a_1 \delta_{x_6} + a_2 \delta_{x_7} + a_3 \delta_{x_2}$ on $\Tt_X$ and $\nu = b_1 \delta_{z_1} + b_2 \delta_{z_5} + b_3 \delta_{z_6} + b_4 \delta_{z_7}$ on $\Tt_Z$. In $\RFFA$, we consider the flows from the root to each support of a measure as illustrated in the corresponding right figures where the length of each edge in the right figures is equal to the length of the path from the root to that node on the corresponding tree $\Tt_{\cdot}$ in the left figures respectively.}
  \label{fg:FlowBased_GTW}
\end{figure}

 \begin{figure}
  \begin{center}
    \includegraphics[width=0.35\textwidth]{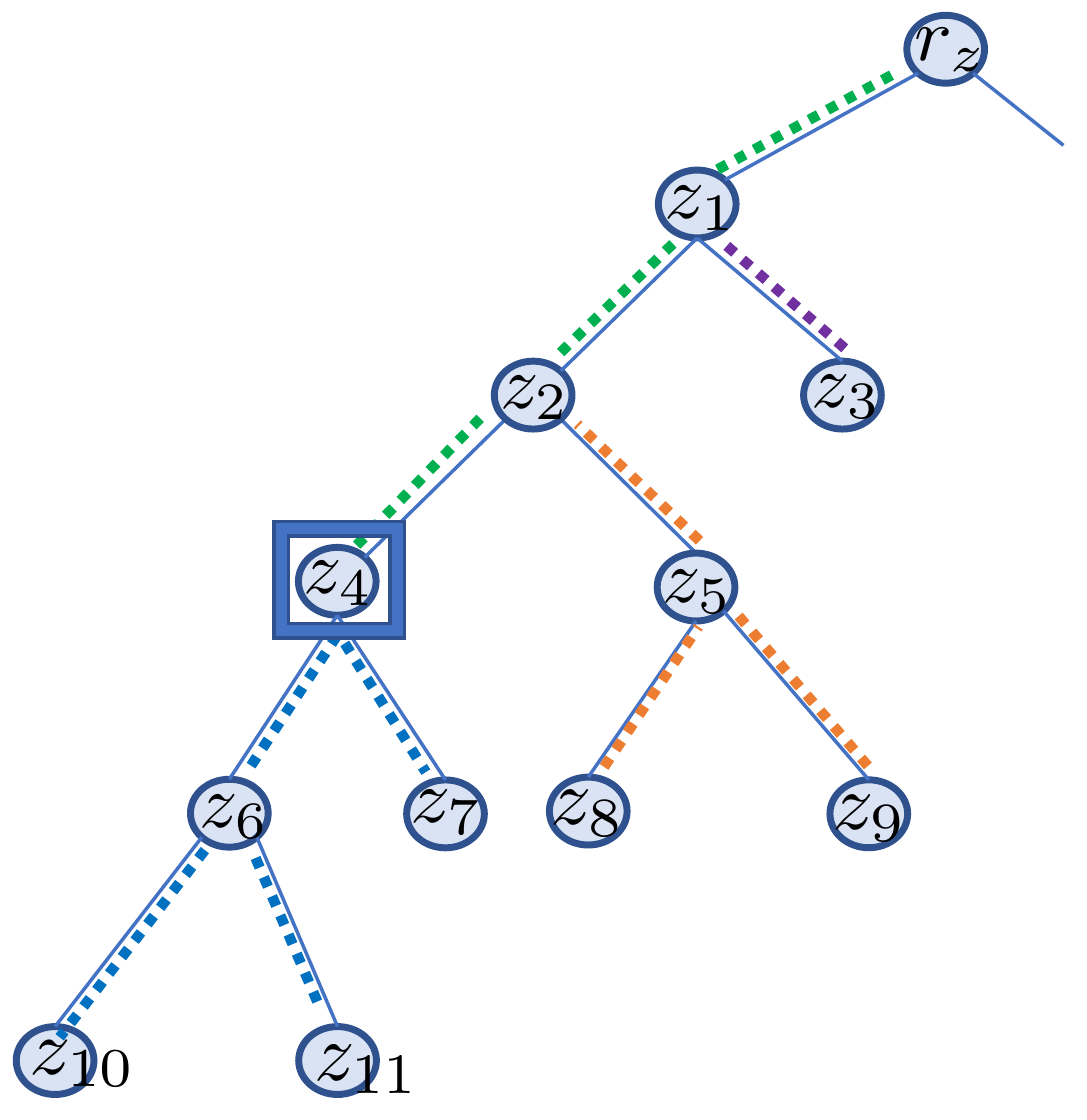}
  \end{center}
  \caption{An illustration for supports in $\Omega_{\nu}$ in the \textbf{Case 2} in the efficient computation approach for $\FA$. Assume that $\Omega_{\nu} = \left\{z_1, z_2, z_3, z_7, z_8, z_9,  z_{10}, z_{11} \right\}$ and the new root $\bar{r}_z = z_4$ (emphasized by the square border). We have supports $z_7, z_{10}, z_{11}$ for \textbf{Case 2a} (blue dots), supports $z_1, z_2$ for \textbf{Case 2b} (green dots), and supports $z_3, z_8, z_9$ for \textbf{Case 2c} (purple and orange dots) where the corresponding closest common ancestors $\zeta_3 = z_1$, $\zeta_8 = z_2$, and $\zeta_9 = z_2$ respectively. Note that, $\zeta_8 = \zeta_9$ (orange dots), therefore the order of supports $z_8, z_9$ is preserved when one changes into the new root.}
  \label{fg:DP_FlowTGW_Case2}
\end{figure}

\section{Further details for \FlowTGW~and \DepthTGW}\label{sec:detail_algorithm}

In this section, we first derive a computation for a univariate optimal transport for empirical measures. Then, we give some further details about \FlowTGW~and \DepthTGW~proposed in the main text.

\subsection{Univariate optimal transport (OT) for empirical measures}\label{sec:detail_OT1D}

Recall that the univariate OT, i.e., univariate Wasserstein, is equal to the integral of the absolute difference between the generalized quantile functions of two univariate probability distributions \citep{SantambrogioBook} (\S2). Therefore, one only needs to sort their supports for the computation with linearithmic complexity.

\begin{algorithm}[H] 
\caption{Univariate optimal transport for empirical measures} 
\label{alg:OT_1D_EM} 
\begin{algorithmic}[1] 
    \REQUIRE Input empirical measures with sorted supports $\mu = \sum_{i \in [n]} a_i \delta_{x_i}$, and $\nu = \sum_{j \in [m]} b_j \delta_{z_j}$ (i.e., $x_1 \le x_2 \le \dots \le x_n$, and $z_1 \le z_2 \le \dots \le z_m$), and a ground distance $\ell$ (e.g., $\ell(x, z) = \ell_1(x, z) = \left|x - z\right|$).
    \ENSURE OT distance $d$ and optimal transport plan $T$.
	\STATE Initialize $d \leftarrow 0$, $T \leftarrow 0_{n \times m}$,	$i \leftarrow 1$, $j \leftarrow 1$.
     	\WHILE{$i \le n$ and $j \le m$}
	\IF{$a_i \le b_j$}	
		\STATE $T_{ij} \leftarrow a_i$.
		\STATE $d \leftarrow d + a_i \ell(x_i, z_j)$.
		\STATE Update $b_j \leftarrow b_j - a_i$, $i \leftarrow i + 1$.
		\IF{$b_j == 0$}
			\STATE $j \leftarrow j+1$.
		\ENDIF
	\ELSE
		\STATE $T_{ij} \leftarrow b_j$.
		\STATE $d \leftarrow d + b_j \ell(x_i, z_j)$.
		\STATE Update $a_i \leftarrow a_i - b_j$, $j \leftarrow j + 1$.
		\IF{$a_i == 0$}
			\STATE $i \leftarrow i+1$.
		\ENDIF
	\ENDIF
	\ENDWHILE
\end{algorithmic}
\end{algorithm}

In particular, given two empirical measures $\mu = \sum_{i \in [n]} \bar{a}_i \delta_{\bar{x}_i}$ and $\nu = \sum_{j \in [m]} \bar{b}_j \delta_{\bar{z}_j}$ whose supports are in one-dimensional space, i.e., $\bar{x}_i, \bar{z}_j \in \RR, \forall i \in [n], j \in [m]$. Firstly, we sort supports of $\mu, \nu$ in an increasing order, denoted as $\mu = \sum_{i \in [n]} a_i \delta_{x_i}$ and $\nu = \sum_{j \in [m]} b_j \delta_{z_j}$ (i.e., $x_1 \le x_2 \le \dots \le x_n$, and $z_1 \le z_2 \le \dots \le z_m$). Without loss of generality, assume that $n \ge m$, the complexity of this sorting is $\O(n \log n)$. We summarize the algorithm for the univariate OT between $\mu$ and $\nu$ (whose supports are already sorted) in Algorithm~\ref{alg:OT_1D_EM}.

The complexity of Algorithm~\ref{alg:OT_1D_EM} is $\O(n)$. Therefore, the complexity of the univariate OT for empirical measures is $\O(n \log n)$, or its main complexity is to sort supports of empirical measures.

\subsection{\FlowTGW}\label{sec:detail_FlowTGW}
There are two types of applications: \textit{without} or \textit{with} priori knowledge about tree metrics for supports in probability measures.

\subsubsection{Applications \textit{without} priori knowledge about tree metrics for supports in probability measures}\label{sec:detail_FTGW_without} 

In general applications, one usually does not have priori knowledge about tree metrics for supports in probability measures. However, one can sample tree metrics for the space of supports of probability measures, e.g., using clustering-based tree metric sampling~\citep{le2019tree} (\S4), for \FlowTGW.

One can compute \FlowTGW~between $\mu=\sum_{i \in [n]} a_i \delta_{x_i}$ and $\nu=\sum_{j \in [m]} b_j \delta_{z_j}$ as follow:
\begin{itemize}
\item Step 1: Sample aligned-root tree metrics $\Tt_X$ and $\Tt_Z$ for supports $x_i \mid_{i \in [n]}$, and $z_j \mid_{j \in [m]}$ of probability measures $\mu$ and $\nu$ respectively, e.g., by choosing means of support data distributions as roots when using the clustering-based tree metric sampling~\citep{le2019tree} (\S4) (See Section~\ref{sec:clustering_based_TM} for a review about clustering-based tree metric sampling).

\item Step 2: Based on the sampled aligned-root tree metrics, \FlowTGW~(Equation (2) in the main text) is equivalent to aligned-root \FlowTGW~(Equation (4) in the main text). Consequently, \FlowTGW~between $\mu$ and $\nu$ is equivalent to the univariate OT distance between $\tilde{\mu} := \sum_{i \in [n]} a_i \delta_{d_{\TMX}(r_x, x_i)}$ and $\tilde{\nu} := \sum_{j \in [m]} b_j \delta_{d_{\TMZ}(r_z, z_j)}$ where $r_x, r_z$ are roots of $\Tt_X, \Tt_Z$ respectively.

\item Step 3: Sort supports of $\tilde{\mu}$ and $\tilde{\nu}$, and then apply Algorithm~\ref{alg:OT_1D_EM} to compute the univariate OT between $\tilde{\mu}$ and $\tilde{\nu}$.
\end{itemize}

We next show a complexity analysis for \FlowTGW: 
\begin{itemize}
\item The complexity of Step 1 is $\O(\bar{N} H_{\Tt} \log \kappa)$ where $H_{\Tt}$ is a predefined deepest level of tree $\Tt$ and $\kappa$ is the number of clusters in the farthest-point clustering used in the clustering-based tree metric sampling~\citep{le2019tree}; $\bar{N}$ is the input number of supports\footnote{One can use supports of the input probability measures, or a (sub)set of supports from several input probability measures, e.g., in case, supports are in non-registered, but same-dimensional spaces, to sample tree metrics having the same tree structure. Therefore, we have $\bar{N} \approx tn$, where $t$ is the number of probability measures whose supports are used to sample tree metrics.}. Let $N$ be the number of nodes in the sampled tree $\Tt$, we have $N \le \left(\kappa^{H_{\Tt}} -1\right)/\left(\kappa -1\right)$. 

\item The complexity of Step 2 is $\O(n H_{\Tt})$ for computing supports of $\tilde{\mu}, \tilde{\nu}$. 

\item The complexity of Step 3 is $\O(n \log n)$ as in Section~\ref{sec:detail_OT1D}. 
\end{itemize}
In general, one usually chooses small values for $H_{\Tt}$ and $\kappa$ (e.g., $(H_{\Tt}=6, \kappa=4)$ are suggested parameters for the clustering-based tree metric sampling~\citep{le2019tree}); and has $n \le N$ (each support is corresponding to a node in a tree). Therefore, the overall complexity of \FlowTGW~is $\O(\bar{N}H_{\Tt}\log \kappa + n H_{\Tt} + n \log n)$, or approximately $\O(N \log N)$.

\subsubsection{Applications \textit{with} priori knowledge about tree metrics for supports in probability measures}\label{sec:detail_FTGW_with} 

In some specific applications where one has a priori knowledge about tree metric for supports in each probability measure. One can compute \FlowTGW~as in Equation (2) in the main text, where one can exhaustedly search the optimal aligned roots, or apply the efficient computation in Section 3.2 in the main text to reduce this complexity.

Assume that one have priori knowledge about tree metrics\footnote{Assume that for each tree metric, each node has at most $\kappa$ child nodes, and the deepest level is $H_{\Tt}$.} $\Tt_X, \Tt_Z$ for supports of probability measure $\mu, \nu$ respectively. In general, one needs to search the optimal aligned roots $r_x, r_z$ for tree $\Tt_X, \Tt_Z$. The complexity of exhausted search is $\O(N^2)$ where $N$ is the number of nodes in trees. Additionally, the complexity of aligned-root \FlowTGW~is $\O(\bar{N}H_{\Tt}\log \kappa + n H_{\Tt} + n \log n)$, or approximately $\O(N \log N)$ (See Section~\ref{sec:detail_FTGW_without}). Therefore, the overall complexity of \FlowTGW~with exhausted search for optimal aligned-roots is $\O(\bar{N}H_{\Tt}\log \kappa + N^2n H_{\Tt} + N^2n \log n + N^2)$\footnote{Naively computing $N^2$ aligned-roots \FlowTGW, and comparing those $N^2$ values to obtain the optimal.}, or approximately $\O(N^3 \log N)$.

As described in Section 3.2 in the main text, those computational steps, e.g., tree metrics between a root to each support, and sorting for those tree metrics between a root to each support or its efficient computation, can be done separately for each tree before one applies Algorithm~\ref{alg:OT_1D_EM} for those sorted tree metrics between a root and each support, then compares those $N^2$ values to find the optimal pair of roots, one can reduce the complexity of \FlowTGW
\begin{itemize}

\item into $\O(\bar{N}H_{\Tt}\log \kappa + Nn H_{\Tt} + n \log n + N^2)$, or $\O(N^2)$ for Case 1 in the main text, since one needs to compute tree metrics from a root to each support with complexity $\O(n H_{\Tt})$ for $N$ times due to changing a root in a tree; sort tree metrics between a root to each support for only $1$ time; and compare aligned-root \FlowTGW~results for $N^2$ cases of pairs of roots. 

\item or nearly into $\O(\bar{N}H_{\Tt}\log \kappa + Nn H_{\Tt} + n \log n + Nn + N^2)$, or nearly $\O(N^2)$ for Case 2 in the main text, since one needs to compute tree metrics from a root to each support with complexity $\O(n H_{\Tt})$ for $N$ times due to changing a root in a tree; sort tree metrics between a root to each support for only $1$ time; merge some ordered arrays with complexity nearly $\O(n)$ for  $N$ times due to changing a root in a tree; and compare aligned-root \FlowTGW~results for $N^2$ cases of pairs of roots. 

When the degenerated case happens, one needs to merge $n$ ordered arrays, where each array only has $1$ node. Therefore, the complexity is $\O(n \log n)$, or one simply needs to resort for those $n$ tree metrics from a root to each support when changing a root of a tree. Hence, the overall complexity for the degenerated case is $\O(\bar{N}H_{\Tt}\log \kappa + Nn H_{\Tt} + Nn\log n + N^2)$, or approximately $\O(N^2\log N)$. 

\end{itemize}

Thus, for \FlowTGW, one can reduce its complexity $\O(N^3 \log N)$ for a naive implementation into nearly $\O(N^2)$ (or into $\O(N^2\log N)$ for the degenerate case) with the proposed efficient computation in Section 3.2 in the main text. 

Note that when one can not screen out any aligned-root \FlowTGW s, the computation of \FlowTGW~requires at least $N^2$ comparisons among aligned-root \FlowTGW~(choosing the optimal pair of roots from the values of $N^2$ aligned-root \FlowTGW). Therefore, for this case, $\O(N^2)$ is also the optimal complexity for \FlowTGW.

\subsection{\DepthTGW}\label{sec:detail_DepthTGW}

Similar to \FlowTGW~in Section~\ref{sec:detail_FlowTGW}, there are two types of applications: \textit{without} or \textit{with} priori knowledge about tree metrics for supports in probability measures. For general applications where one usually does not have priori knowledge about tree metrics for probability measures, one can apply clustering-based tree metric sampling~\citep{le2019tree} to sample tree metrics for supports of probability measures. For some specific applications where one knows tree metric for each probability measure, one needs to search the optimal aligned roots, e.g., by exhausted search.

\subsubsection{Applications \textit{without} priori knowledge about tree metrics for supports in probability measures}\label{sec:detail_DTGW_without} 
For those general applications without priori knowledge about tree metrics for supports in probability measures, one can use clustering-based tree metric approach to sample tree metrics for supports of the probability measures.

One can compute \DepthTGW~between $\mu=\sum_{i \in [n]} a_i \delta_{x_i}$ and $\nu=\sum_{j \in [m]} b_j \delta_{z_j}$ as follow:
\begin{itemize}
\item Step 1: Sample aligned-root tree metrics $\Tt_X$ and $\Tt_Z$ for supports $x_i \mid_{i \in [n]}$ and $z_j \mid_{j \in [m]}$ in probability measures $\mu$ and $\nu$ respectively (similar to Step 1 for \FlowTGW).
\item Step 2: Based on the sampled aligned-root tree metrics, \DepthTGW~(Equation (7) in the main text) is equivalent to aligned-root \DepthTGW~(Equation (6) in the main text). For each probability measure, we construct 2-depth-level tree for all nodes from the tree root to each support of the probability measure as in Algorithm~\ref{alg:Construct_2DepthLevelTree}\footnote{Constructing 2-depth-level tree for all nodes from the tree root to each support of probability measures as in Algorithm~\ref{alg:Construct_2DepthLevelTree} can be considered as a preprocessing step since those 2-depth-level trees are needed during the hierarchical alignment along each deep level in trees for a computation of \DepthTGW.}.

\begin{algorithm}[H] 
\caption{Construct 2-depth-level tree} 
\label{alg:Construct_2DepthLevelTree} 
\begin{algorithmic}[1] 
    \REQUIRE Input empirical measure $\mu = \sum_{i \in [n]} a_i \delta_{x_i}$, tree metric $\Tt_X$.
    \ENSURE Set of 2-depth-level trees for $\mu$.
	\STATE Construct a set of paths $S_{\tilde{p}}$ for supports $x_i \mid_{i \in [n]}$ where each element is a path from a root to each support.
	\STATE From the set of paths $S_{\tilde{n}}$, construct a set of nodes where each node belongs to at least one path of the set of paths $S_{\tilde{p}}$.
	\STATE For each node in $S_{\tilde{n}}$, construct a 2-depth-level tree for that node in tree $\Tt$ for $\mu$, as in Section 4 in the main text. 
	\STATE Gather all those 2-depth-level trees to form the set of 2-depth-level trees for $\mu$.
\end{algorithmic}
\end{algorithm}

\item Step 3: Compute the aligned-root \DepthTGW~(Equation (6) in the main text). It starts from a comparison between 2-depth-level tree constructed from each root of $\Tt_X$, and $\Tt_Z$ for $\mu$ and $\nu$ respectively, with optimal matching mass $1$. 
\begin{itemize}
\item If it is \textit{not} a simple case\footnote{A simple case for a pair of considered nodes is defined as: at least one node of the considered pair does not have child nodes, or sum of its child-node weights is equal to $0$.}, then we compute the disrepancy between 2-depth-level tree as aligned-root \FlowTGW~by simply sorting supports and using Algorithm~\ref{alg:OT_1D_EM}. Then, we push all the matching pairs between child nodes and their optimal matching mass into the queue.

\item If it is a simple case where both two nodes of the considered pair do not have child nodes, or sum of their child-node weights is equal to $0$, then their discrepancy is equal to $0$.

\item If it is a simple case where one of two considered nodes, but not both of them, does not have child nodes, or sum of its child-node weights is equal to $0$, then their discrepancy is equal to sum of normalized-weighted lengths of paths from that node to supports of a corresponding measure which are in the subtree rooted at that node. 
\end{itemize}
We stop the computation when the queue is empty. The aligned-root \DepthTGW~is equal to sum of all weighted discrepancies between 2-depth-level trees.
\end{itemize}

We summarize the computation for \DepthTGW~by sampling aligned-root tree metrics in Algorithm~\ref{alg:DepthTGW_SampledTM}.

We next give a complexity analysis for \DepthTGW:
\begin{itemize}
\item Recall that the complexity of sampling tree metric $\O(\bar{N} H_{\Tt} \log \kappa)$ where $H_{\Tt}$ is a predefined deepest level of tree $\Tt$ and $\kappa$ is the number of clusters in the farthest-point clustering for the clustering-based tree metric sampling~\citep{le2019tree}; $\bar{N}$ is the input number of supports. 
\item The complexity of constructing 2-depth-level trees is $\O(n H_{\Tt} \kappa)$ (we have $n$ supports, each path from a root to a support has less than $H_{\Tt}$ nodes, and each 2-depth-level tree has less than or equal $(\kappa + 1)$ nodes).
\item The complexity to compute the univariate OT between 2-depth-level trees is $\O(\kappa \log \kappa)$.
\item At deep level $(h + 1)$, the number of nodes is not more than $\kappa^h$. So, the number of pairs of nodes at deep level $(h + 1)$ is not more than $\kappa^{2h}$. Let $\Delta$ be the number of comparisons $\Delta$ for 2-depth-level trees, we have $\Delta \le \left(\kappa^{2H_{\Tt}} - 1 \right)/\left(\kappa^2 - 1\right)$.
\end{itemize}
Therefore, one can implement the computation of \DepthTGW~with a complexity $\O(\bar{N} H_{\Tt} \log \kappa + n H_{\Tt} \kappa + \Delta \kappa \log \kappa)$.

\begin{algorithm}[H] 
\caption{\DepthTGW~for probability measures by sampling aligned-root tree metrics} 
\label{alg:DepthTGW_SampledTM} 
\begin{algorithmic}[1] 
    \REQUIRE Input probability measures $\mu = \sum_{i \in [n]} a_i \delta_{x_i}$, and $\nu = \sum_{j \in [m]} b_j \delta_{z_j}$.
    \ENSURE \DepthTGW~discrepancy $d$.
	\STATE Sample aligned-root tree metrics $\Tt_X, \Tt_Z$ for $\mu, \nu$ respectively, e.g., by choosing a mean of support data as its root for the clustering-based tree metric sampling~\citep{le2019tree}.
	
	\STATE Construct $S_{\Tt_X}, S_{\Tt_Z}$: sets of 2-depth-level trees for $\mu, \nu$ in tree $\Tt_X, \Tt_Z$ respectively by using Algorithm~\ref{alg:Construct_2DepthLevelTree}.
	
	\STATE Initialization with a pair of roots $(r_x, r_z)$ of tree $\Tt_X, \Tt_Z$ respectively, and the optimal matching mass $T^{*}(r_x, r_z) = 1$, and $d \leftarrow 0$.
	\STATE Push $\left\{ (r_x, r_z), T^{*}(r_x, r_z)  \right\}$ into a queue $\texttt{Q}$.
	
     	\WHILE{$\texttt{Q}$ is not empty}
	\STATE Pull $(x, z), T^{*}(x, z)$ from queue $\texttt{Q}$.
	\STATE Get corresponding 2-depth-level trees: $\Tt^2_{x}, \Tt^2_{z}$ for $\mu, \nu$ from $S_{\Tt_X}, S_{\Tt_Z}$ respectively.
	\STATE \% For two simple 2-depth-level trees, the discrepancy between $\mu_{\Tt^2_x}, \nu_{\Tt^2_z}$ is equal to $0$.
	\IF{One of two 2-depth-level trees $\Tt^2_{x}, \Tt^2_{z}$ is simple, but not both of them}	
		\IF{$\Tt^2_{x}$ is simple}
			\STATE Get all paths from $x$ to each support of $\mu$ in the subtree of $\Tt_X$ rooted at $x$.
			\STATE Normalize for weights of those supports of corresponding paths.
			\STATE Compute $\tilde{d}$ as a sum of weighted lengths for those paths.
			\STATE $d \leftarrow d + T^{*}(x, z)\tilde{d}$.
		\ELSE
			\STATE Get all paths from $z$ to each support of $\nu$ in the subtree of $\Tt_Z$ rooted at $z$.
			\STATE Normalize for weights of those supports of corresponding paths.
			\STATE Compute $\tilde{d}$ as a sum of weighted lengths for those paths.
			\STATE $d \leftarrow d + T^{*}(x, z)\tilde{d}$.
		\ENDIF
	\ELSIF{Two 2-depth-level trees $\Tt^2_{x}, \Tt^2_{z}$ are not simple}
		\STATE Sort the distances from a root to each node in tree $\Tt^2_{x}, \Tt^2_{z}$.
		\STATE Compute univariate OT $\tilde{d}$ and the optimal transportation plan $\tilde{T}^{*}$ for empirical measures with sorted supports $\mu_{\Tt^2_x}, \nu_{\Tt^2_z}$ by using Algorithm~\ref{alg:OT_1D_EM}.
		\STATE $d \leftarrow d + T^{*}(x, z)\tilde{d}$.
		\STATE Compute weighted optimal transport plan $\tilde{T}^{*} \leftarrow T^{*}(x, z)\tilde{T}^{*}$.
		\STATE For all child nodes $u, v$ of $\Tt^2_{x}, \Tt^2_{z}$ respectively, if their optimal matching mass $\tilde{T}^{*}(u, v) > 0$, push $\{(u, v), \tilde{T}^{*}(u, v)\}$ into queue $\texttt{Q}$.
	\ENDIF
	\ENDWHILE
\end{algorithmic}
\end{algorithm}

\subsubsection{Applications \textit{with} priori knowledge about tree metrics for supports in probability measures}\label{sec:detail_DTGW_with} 

For some specific applications where one has a priori knowledge about tree metric for supports in each probability measures, one can easily tailor Algorithm~\ref{alg:DepthTGW_SampledTM} with existing tree metrics to compute aligned-root \DepthTGW. Thus, for the \DepthTGW, one needs to search the optimal pair of roots for the given tree metrics\footnote{Assume that for each tree metric, each node has at most $\kappa$ child nodes, the tree deep is $H_{\Tt}$, and the number of nodes in tree is about $N$.}, where one uses the aligned-root \DepthTGW~for each pair of roots. Overall, the complexity of \DepthTGW~with exhausted search for the optimal pair of roots is approximately $\O(N^2 n H_{\Tt} \kappa + N^2 \Delta \kappa \log \kappa)$.

\section{Further experimental results}\label{sec:further_experimental_results}

We denote EGW$_0$ for the standard entropic Gromov-Wasserstein where we use entropic regularization for both optimizing the transport plan and computing entropic GW. 

\subsection{Further experimental results on quantum chemistry and document classification} 


We illustrate the trade-off between performances and time consumption for the discrepancies for probability measures in different spaces when their parameters are changed, e.g., entropic regularization in EGW and EGW$_0$, and the number of (tree) slices in SGW, \FlowTGW~(FA), and \DepthTGW~(DA)
\begin{itemize}
\item for quantum chemistry (\texttt{qm7} dataset) in Figure~\ref{fg:qm7_para_sub},
\item for document classification
	\begin{itemize}
		\item in \texttt{TWITTER} dataset in Figure~\ref{fg:twitter_para_sub},
		\item in \texttt{RECIPE} dataset in Figure~\ref{fg:recipe_para_sub},
		\item in \texttt{CLASSIC} dataset in Figure~\ref{fg:classic_para_sub},
		\item in \texttt{AMAZON} dataset in Figure~\ref{fg:amazon_para_sub}.
	\end{itemize}
\end{itemize}

The entropic term in standard entropic GW (EGW$_0$) may harm its performances (comparing with EGW) (e.g., in \texttt{qm7, RECIPE, CLASSIC} datasets illustrated in Figure~\ref{fg:qm7_para_sub}, Figure~\ref{fg:recipe_para_sub}, Figure~\ref{fg:classic_para_sub} respectively). Performances of EGW and the standard EGW$_{0}$ are improved when entropic regularization (eps) is smaller, but their computational time is considerably increased.


\begin{figure*}
  \begin{center}
       \includegraphics[width=\textwidth]{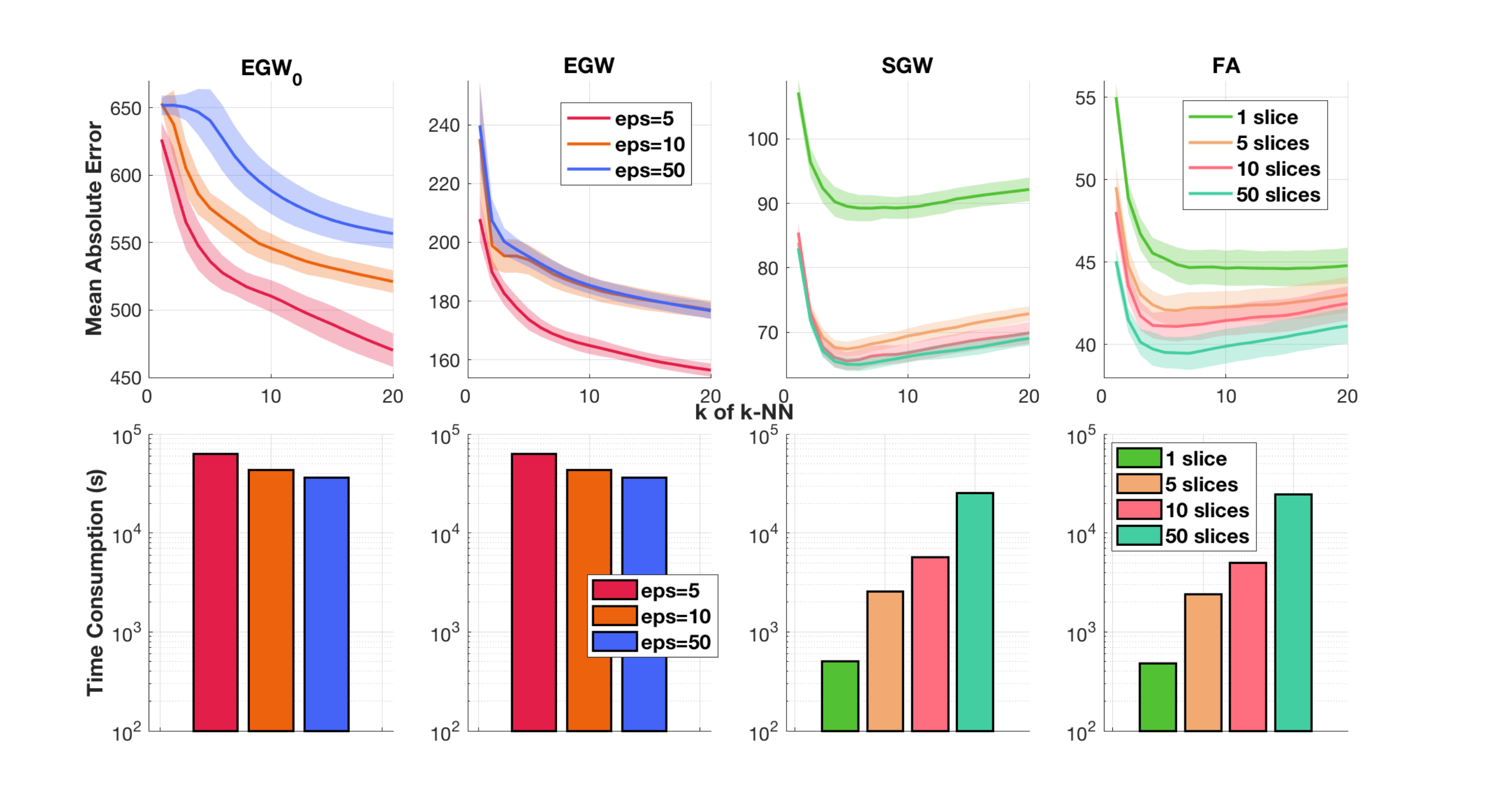}
  \end{center}
  \caption{Results of MAE and time consumption for the discrepancies with different parameters (e.g. entropic regularization in EGW$_0$/EGW, and number of (tree) slices in SGW/FA) in $k$-NN regression in \texttt{qm7} dataset. For clustering-based tree metric approach, we used its suggested parameters ($\kappa=4$, $H_{\Tt}=6$).}
  \label{fg:qm7_para_sub}
\end{figure*}

\begin{figure*}
  \begin{center}
    \includegraphics[width=\textwidth]{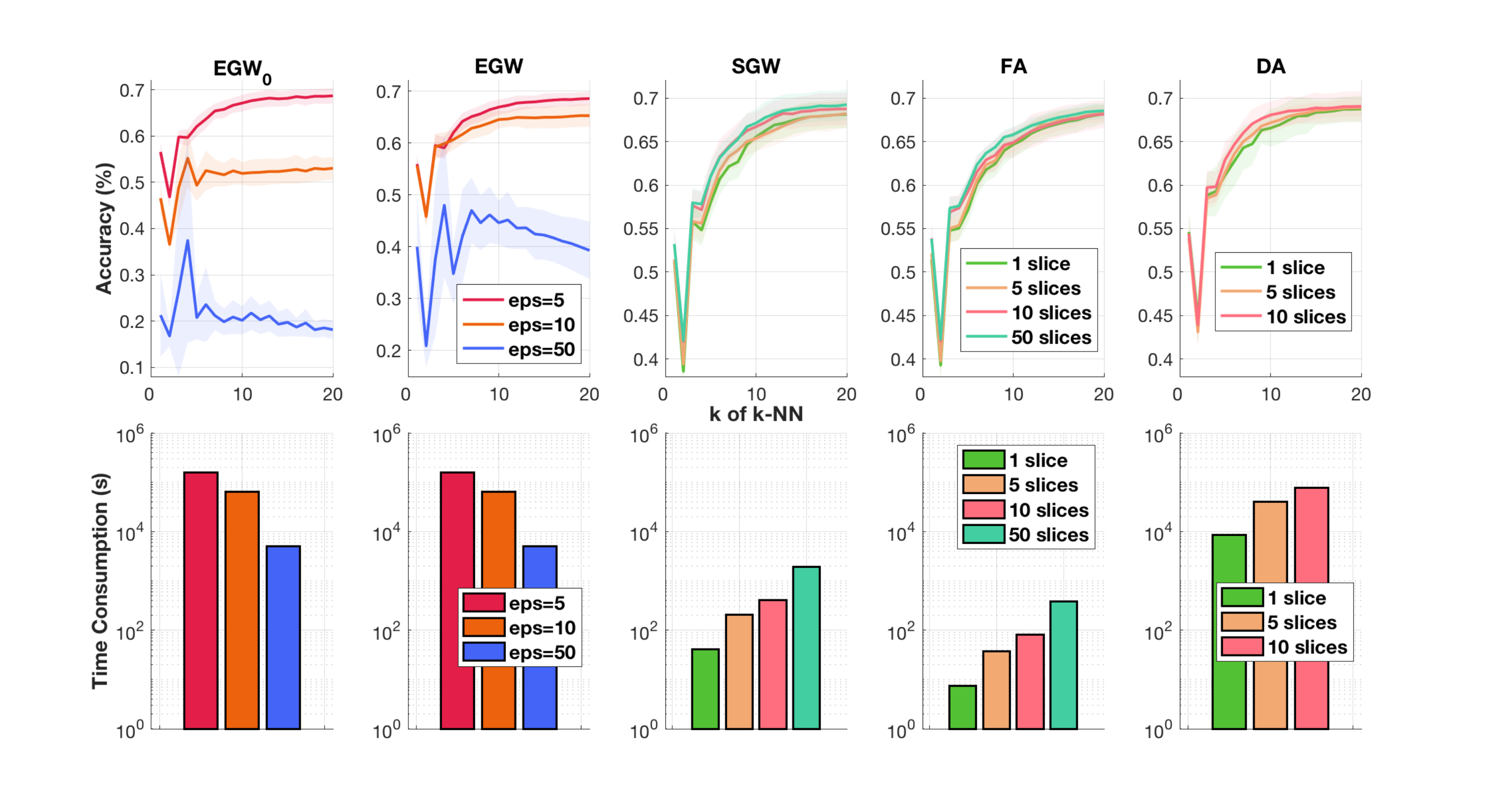}
  \end{center}
  \caption{Results of averaged accuracy and time consumption for the discrepancies with different parameters (e.g., entropic regularization in EGW$_0$/EGW, and the number of (tree) slices in SGW/FA/DA in $k$-NN in \texttt{TWITTER} dataset. For clustering-based tree metric approach, we used its suggested parameters ($\kappa=4$, $H_{\Tt}=6$).}
  \label{fg:twitter_para_sub}
\end{figure*}

\begin{figure*}
  \begin{center}
    \includegraphics[width=\textwidth]{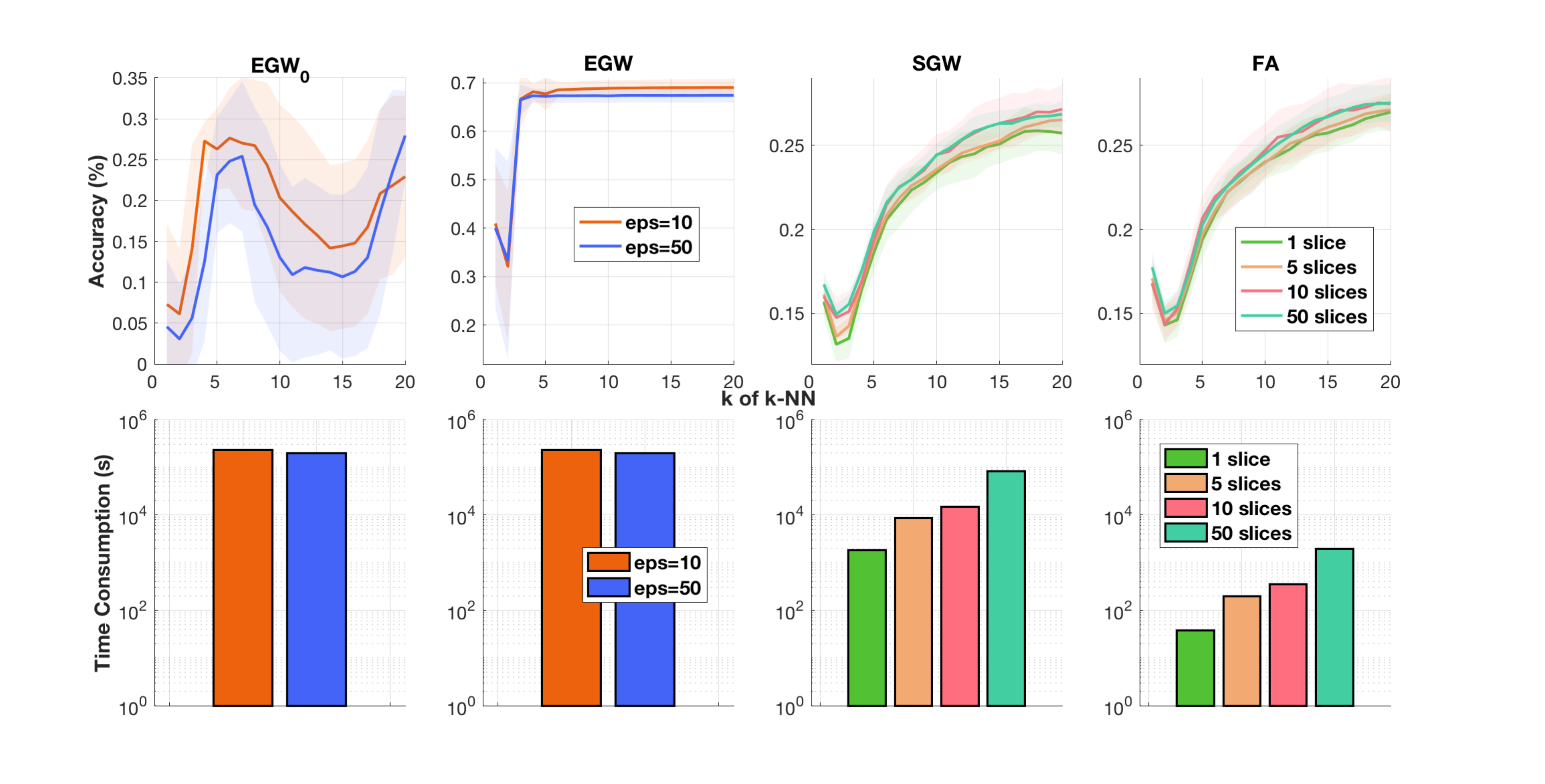}
  \end{center}
  \caption{Results of averaged accuracy and time consumption for the discrepancies with different parameters, e.g., entropic regularization in EGW$_0$/EGW, and the number of (tree) slices in SGW/FA in $k$-NN in \texttt{RECIPE} dataset. For clustering-based tree metric approach, we used its suggested parameters ($\kappa=4$, $H_{\Tt}=6$).}
  \label{fg:recipe_para_sub}
\end{figure*}

\begin{figure*}
  \begin{center}
    \includegraphics[width=\textwidth]{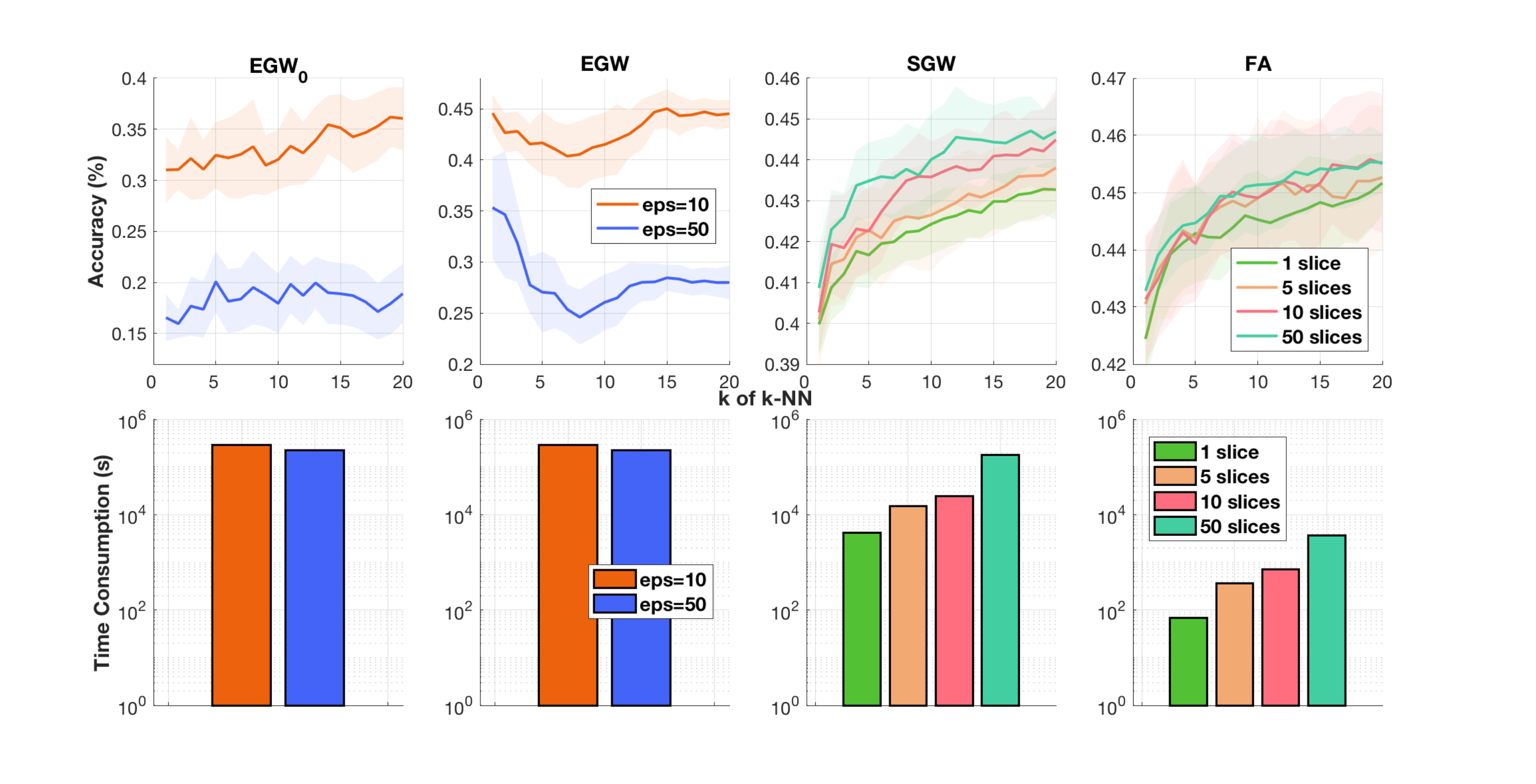}
  \end{center}
  \caption{Results of averaged accuracy and time consumption for the discrepancies with different parameters, e.g., entropic regularization in EGW$_0$/EGW, and the number of (tree) slices in SGW/FA in $k$-NN in \texttt{CLASSIC} dataset. For clustering-based tree metric approach, we used its suggested parameters ($\kappa=4$, $H_{\Tt}=6$).}
  \label{fg:classic_para_sub}
\end{figure*}

\begin{figure*}
  \begin{center}
    \includegraphics[width=\textwidth]{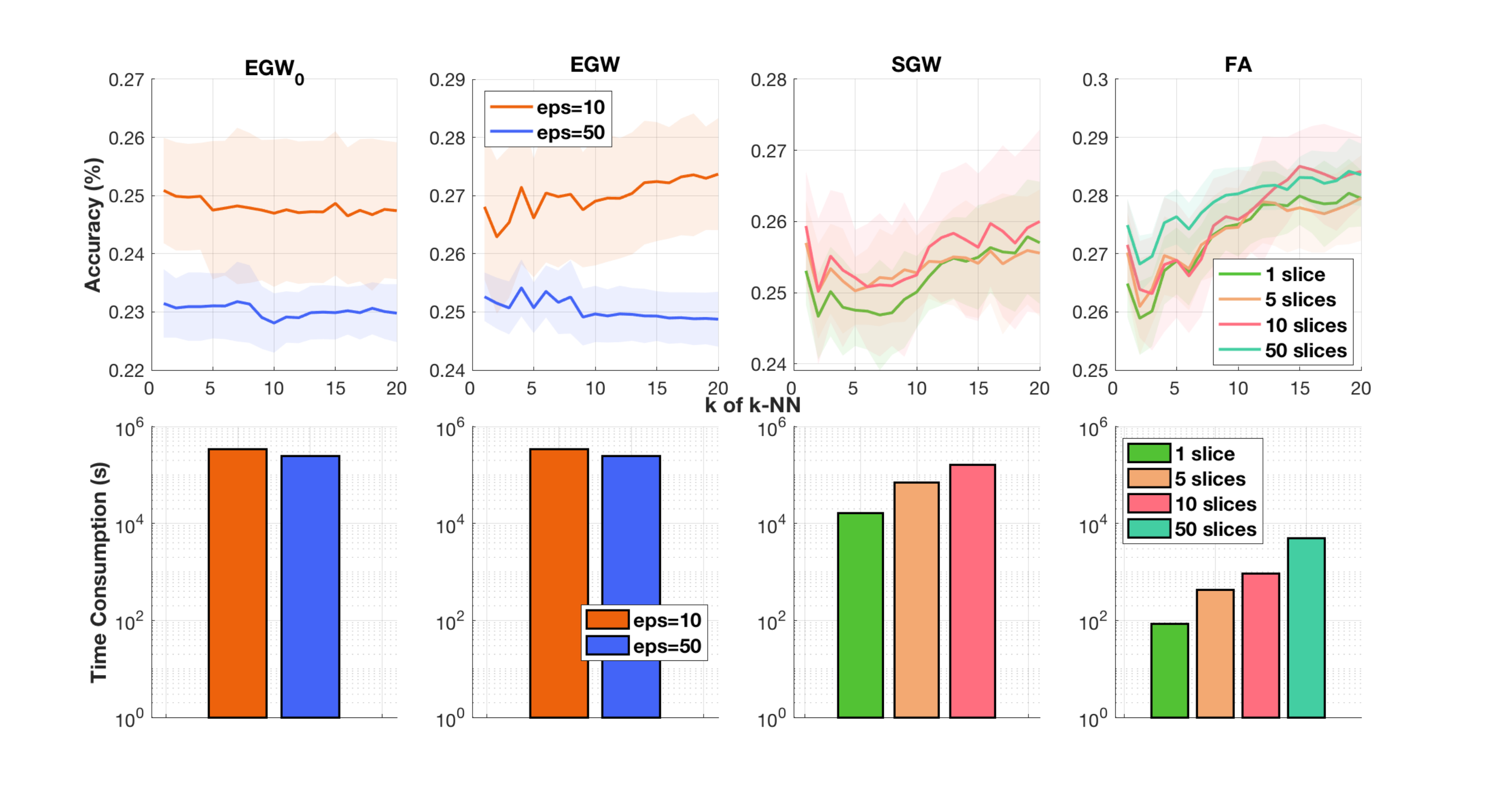}
  \end{center}
  \caption{Results of averaged accuracy and time consumption for the discrepancies with different parameters, e.g., entropic regularization in EGW$_0$/EGW, and the number of (tree) slices in SGW/FA in $k$-NN in \texttt{AMAZON} dataset. For clustering-based tree metric approach, we used its suggested parameters ($\kappa=4$, $H_{\Tt}=6$).}
  \label{fg:amazon_para_sub}
\end{figure*}

\subsection{Time consumption for the clustering-based tree metric sampling}

Time consumption for tree metric sampling by the clustering-based tree metric method \citep{le2019tree} is negligible in computation for both \FlowAlign~and \DepthAlign. Indeed, we illustrate time consumption for tree metric sampling with different parameters, e.g., the predefined deepest level $H_{\Tt}$, the number of clusters $\kappa$, for the clustering-based tree metric sampling~\citep{le2019tree} in \texttt{qm7, TWITTER, RECIPE, CLASSIC, AMAZON} datasets in Figure~\ref{fg:treemetric_time}. For examples, for each tree metric sampling with the suggested parameters $(H_{\Tt}=6, \kappa=4)$, it only took about $0.4$ seconds for \texttt{qm7} dataset, $1.5$ seconds for \texttt{TWITTER} dataset, $11.0$ seconds for \texttt{RECIPE} dataset, $17.5$ seconds for \texttt{CLASSIC} dataset, and $20.5$ seconds for \texttt{AMAZON} dataset. Furthermore, we give a brief review for the clustering-based tree metric sampling in Section~\ref{sec:clustering_based_TM}.

\begin{figure*}
  \begin{center}
    \includegraphics[width=\textwidth]{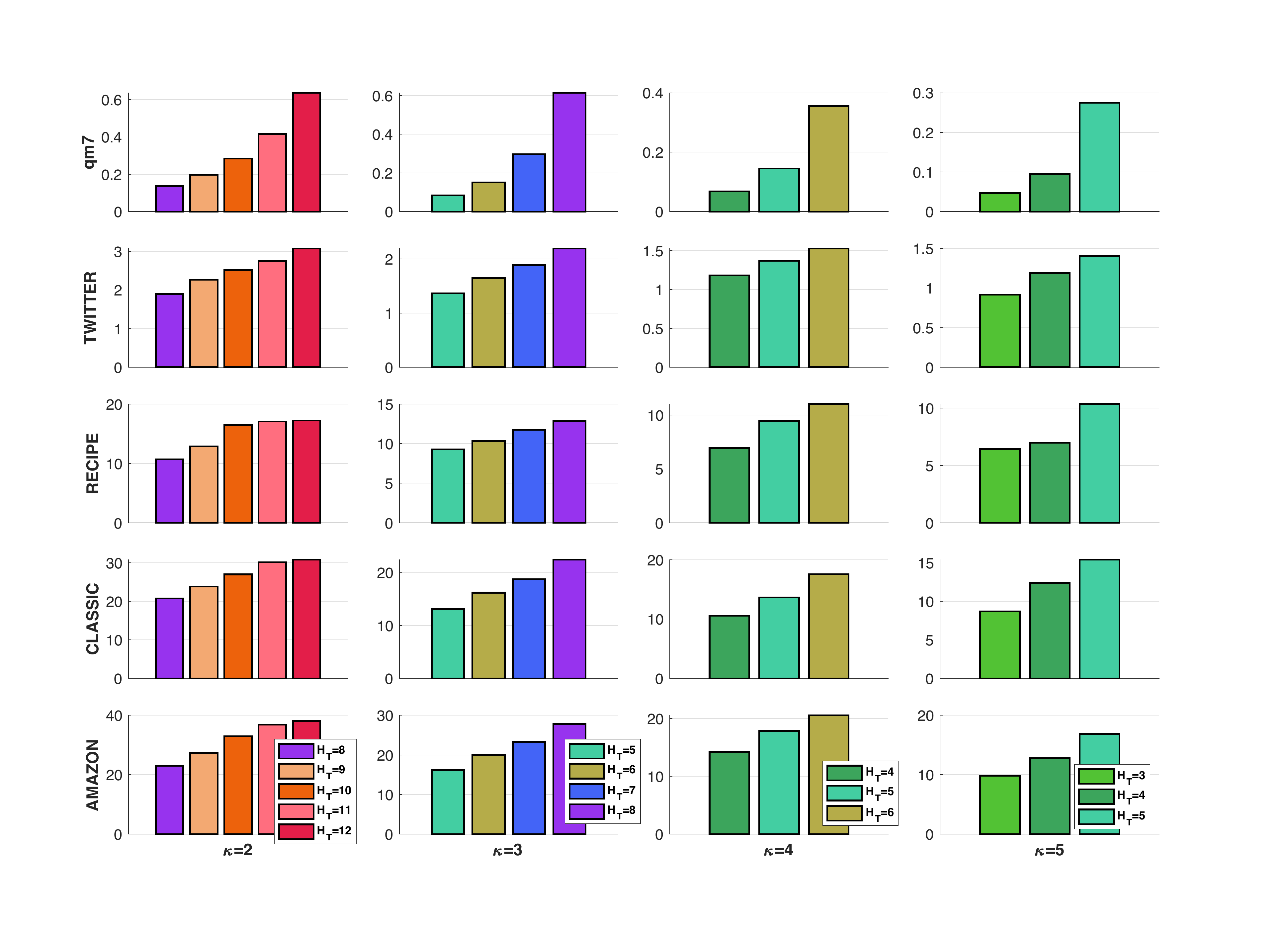}
  \end{center}
  \caption{Time consumption (seconds) for a tree metric sampling in \FlowAlign~and \DepthAlign~by the clustering-based tree metric method~\citep{le2019tree} with different parameters (e.g. the predefined deepest level $H_{\Tt}$, the number of clusters $\kappa$) in quantum chemistry (\texttt{qm7} dataset), and document classification (\texttt{TWITTER, RECIPE, CLASSIC, AMAZON} datasets).}
  \label{fg:treemetric_time}
\end{figure*}

\subsection{Experiment results with different parameters for tree metric sampling}

We illustrate results of mean absolute error (MAE) and time consumption for \FlowAlign~(10 tree slices) with different parameters, e.g., the predefined deepest level $H_{\Tt}$, the number of clusters $\kappa$, in the clustering-based tree metric sampling:

\begin{itemize}
\item in \texttt{qm7} dataset in Figure~\ref{fg:qm7_para_TM10_appendix},

\item in \texttt{TWITTER} dataset in Figure~\ref{fg:twitter_para_TM10_appendix},

\item in \texttt{RECIPE} dataset in Figure~\ref{fg:recipe_para_TM10_appendix},

\item in \texttt{CLASSIC} dataset in Figure~\ref{fg:classic_para_TM10_appendix},

\item in \texttt{AMAZON} dataset in Figure~\ref{fg:amazon_para_TM10_appendix}.
\end{itemize}

\begin{figure}
  \begin{center}
    \includegraphics[width=\textwidth]{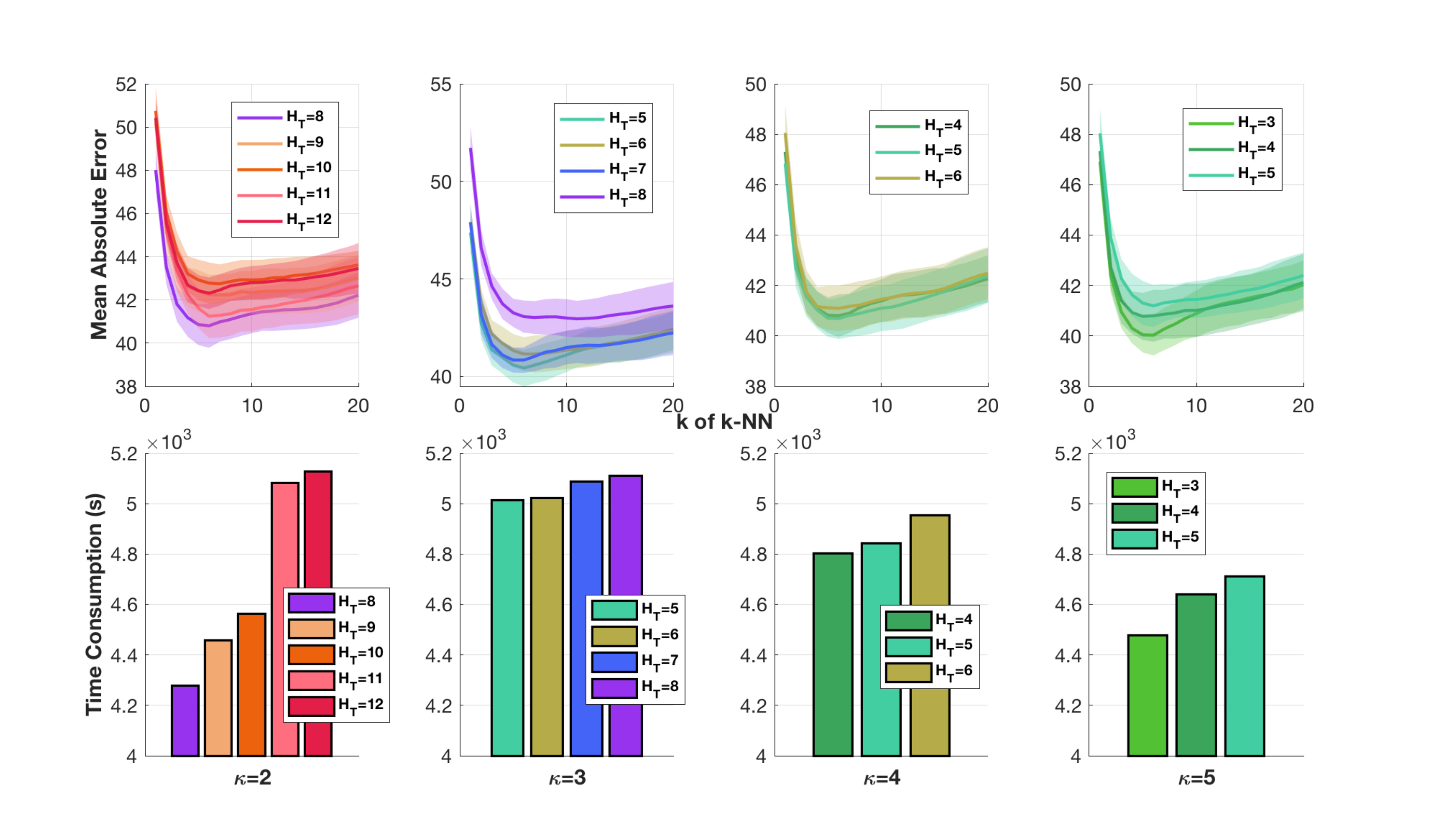}
  \end{center}
  \caption{Results of mean absolute error and time consumption for \FlowAlign~(10 tree slices) with different parameters (e.g. the predefined deepest level $H_{\Tt}$, the number of clusters $\kappa$) in the clustering-based tree metric sampling in \texttt{qm7} dataset.}
  \label{fg:qm7_para_TM10_appendix}
\end{figure}


\begin{figure}
  \begin{center}
    \includegraphics[width=\textwidth]{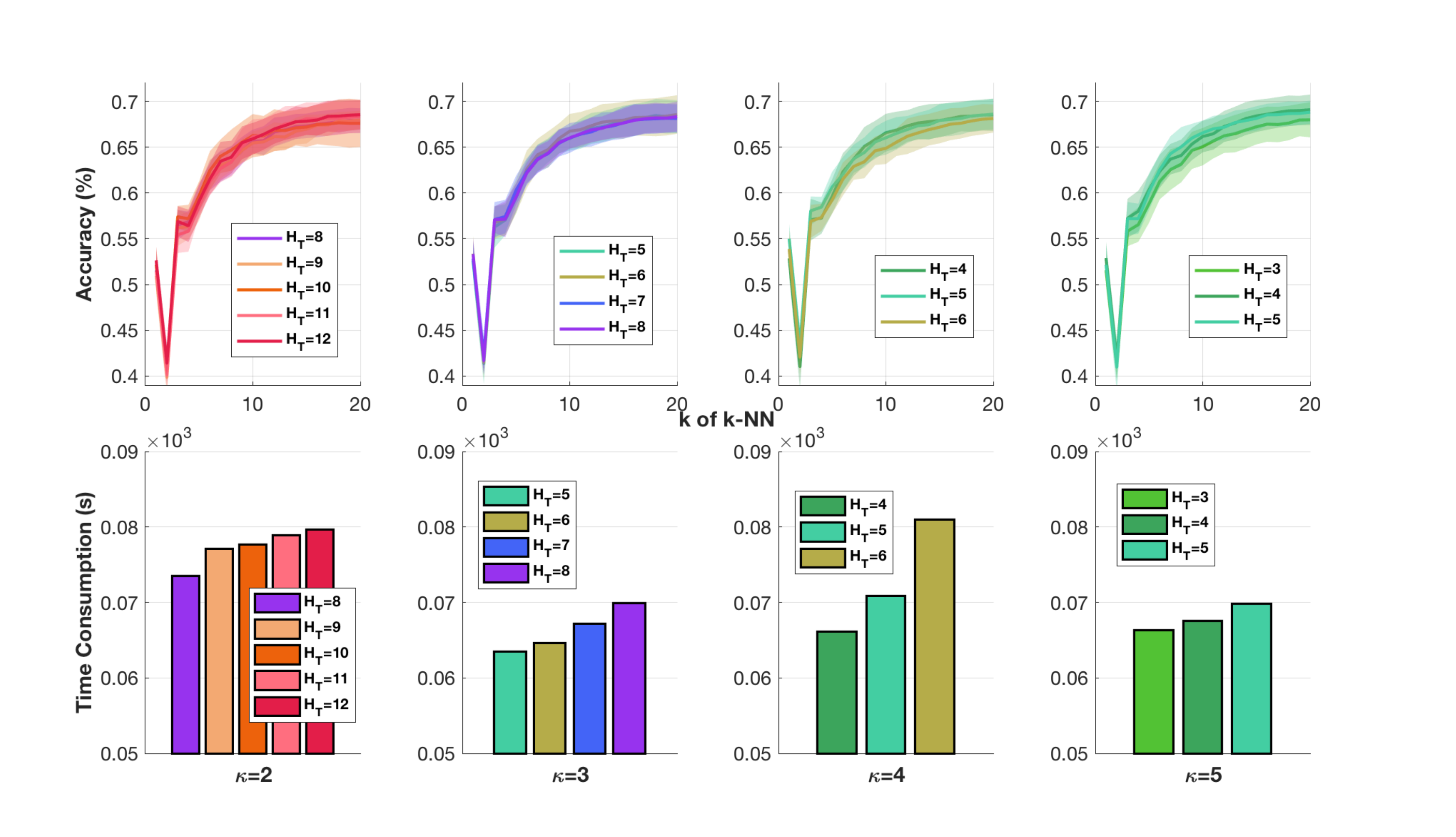}
  \end{center}
  \caption{Results of averaged accuracy and time consumption for \FlowAlign~(10 tree slices) with different parameters (e.g. the predefined deepest level $H_{\Tt}$, the number of clusters $\kappa$) in the clustering-based tree metric sampling in \texttt{TWITTER} dataset.}
  \label{fg:twitter_para_TM10_appendix}
\end{figure}


\begin{figure}
  \begin{center}
    \includegraphics[width=\textwidth]{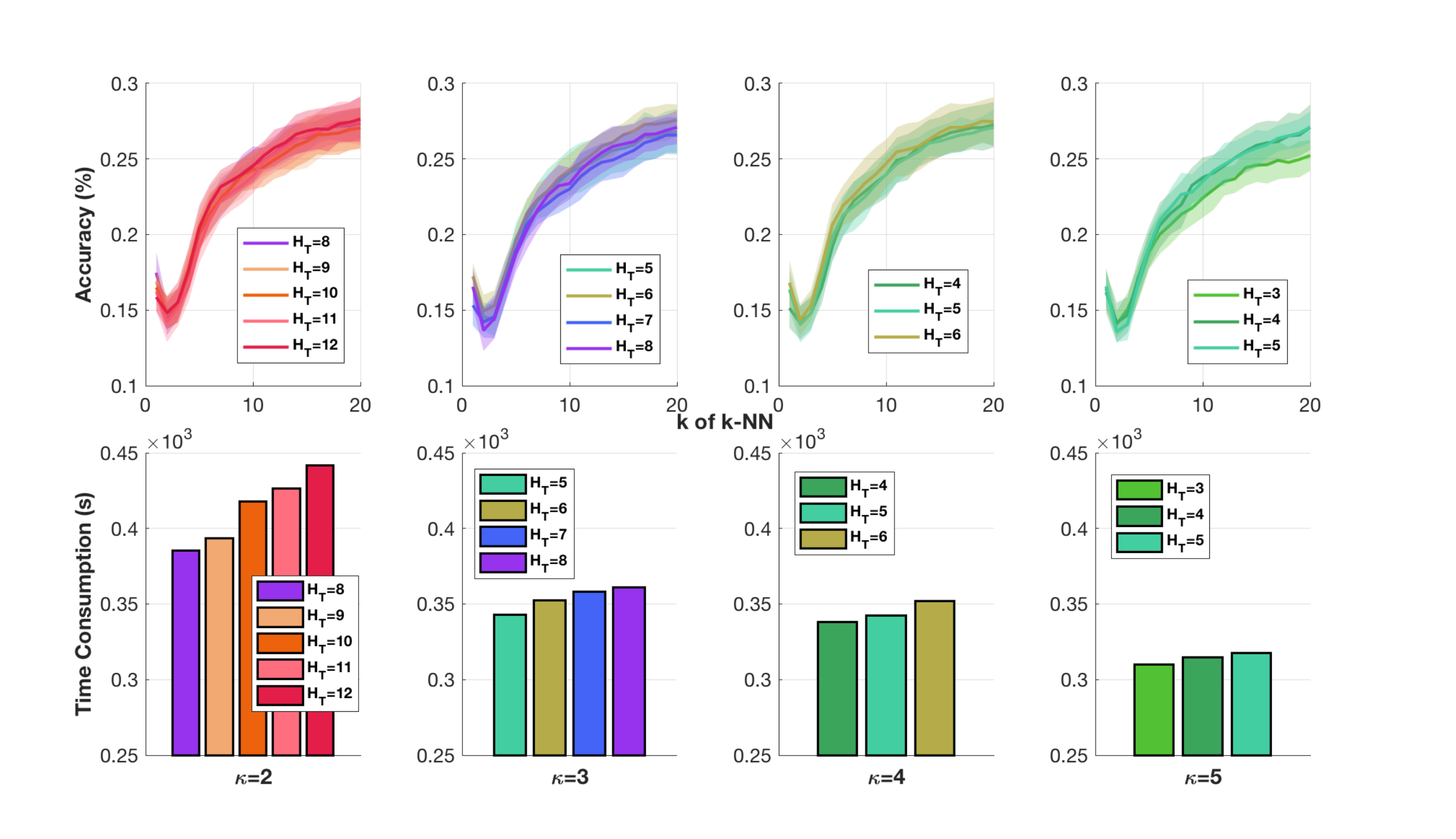}
  \end{center}
  \caption{Results of averaged accuracy and time consumption for \FlowAlign~(10 tree slices) with different parameters (e.g. the predefined deepest level $H_{\Tt}$, the number of clusters $\kappa$) in the clustering-based tree metric sampling in \texttt{RECIPE} dataset.}
  \label{fg:recipe_para_TM10_appendix}
\end{figure}


\begin{figure}
  \begin{center}
    \includegraphics[width=\textwidth]{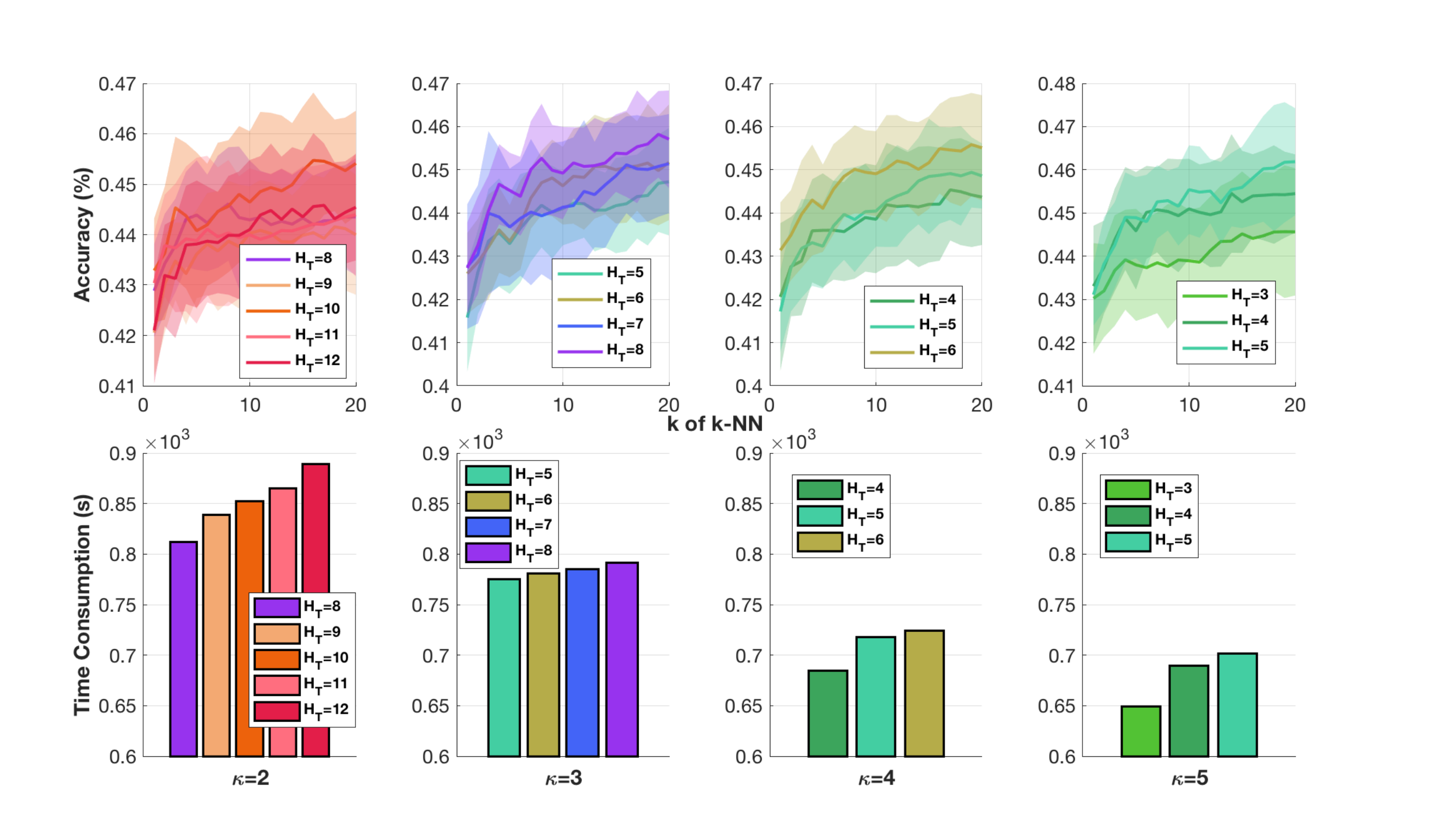}
  \end{center}
  \caption{Results of averaged accuracy and time consumption for \FlowAlign~(10 tree slices) with different parameters (e.g. the predefined deepest level $H_{\Tt}$, the number of clusters $\kappa$) in the clustering-based tree metric sampling in \texttt{CLASSIC} dataset.}
  \label{fg:classic_para_TM10_appendix}
\end{figure}


\begin{figure}
  \begin{center}
    \includegraphics[width=\textwidth]{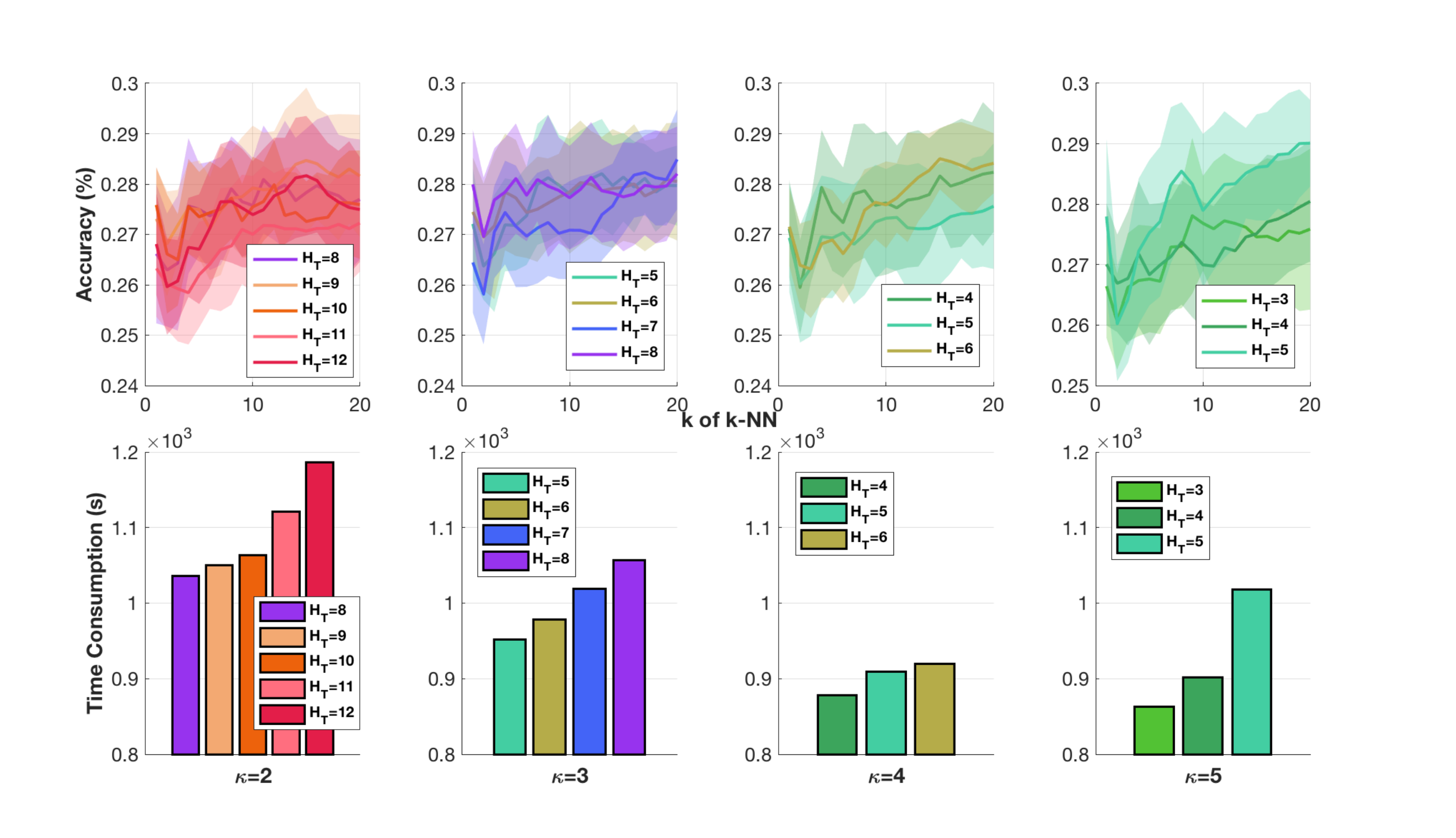}
  \end{center}
  \caption{Results of averaged accuracy and time consumption for \FlowAlign~(10 tree slices) with different parameters (e.g. the predefined deepest level $H_{\Tt}$, the number of clusters $\kappa$) in the clustering-based tree metric sampling in \texttt{AMAZON} dataset.}
  \label{fg:amazon_para_TM10_appendix}
\end{figure}

\subsection{Further experimental results: $k$-means clustering on a small experimental setup for performance comparison on random rotated \texttt{MNIST} dataset} 

We follow the small experimental setup for performance comparison on \textit{random rotated} \texttt{MNIST} dataset as in \citep{peyre2016gromov}. We randomly select $50$ point clouds from each digits $0$ to $4$, apply $k$-means clustering with $k=5$, and $k$-means++ initialization. We show the performance comparison for $k$-means clustering in Figure~\ref{fg:kmeans_small_MNIST}\footnote{The barycenter for SGW is not published yet.}. The performances of \FlowAlign~are comparative with EGW. Moreover, \FlowAlign~is several order faster than EGW. The performances of EGW are better when entropic regularization (eps) is smaller, but the time consumption is also higher. 

\begin{figure}
  \begin{center}
    \includegraphics[width=0.45\textwidth]{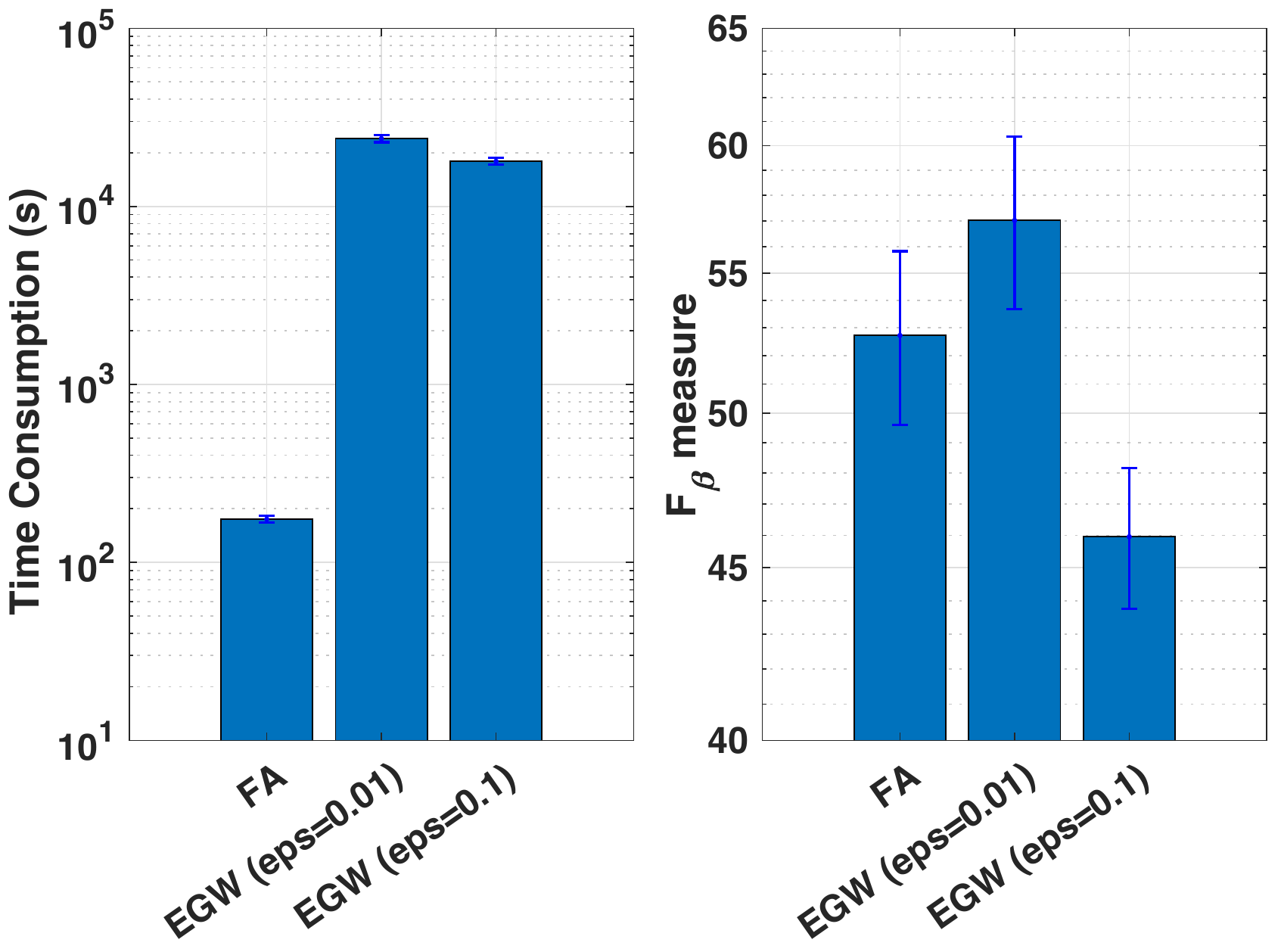}
  \end{center}
  \caption{Results of $k$-means clustering on a small experiment setup for performance comparison on \textit{random rotated} \texttt{MNIST} dataset.}
  \label{fg:kmeans_small_MNIST}
\end{figure}

\section{Some brief reviews}\label{sec:brief_review}

We give brief reviews for the farthest-point clustering \citep{gonzalez1985clustering} (a more detail summarization and discussion can be seen in~\citep{le2019tree}), the clustering-based tree metric sampling~\citep{le2019tree}, tree metric in \citep{semple2003phylogenetics}, and $F_{\beta}$ measure for clustering evaluation \citep{manning2008introduction} where $\beta$ is chosen as in \citep{le2015unsupervised}.

\subsection{The farthest-point clustering}\label{sec:farthest_point_clustering}

The farthest-point clustering \citep{gonzalez1985clustering} is a simple fast greedy approach for a $\kappa$-center problem. The $\kappa$-center problem is defined as finding a partition of $n$ points into $\kappa$ clusters to minimize the maximum radius of clusters. The complexity of a direct implementation, e.g., Algorithm~\ref{alg:FarthestPointClustering}, is $\O(n\kappa)$. Moreover, by using the algorithm in \citep{feder1988optimal}\footnote{Code is available at https://github.com/vmorariu/figtree/blob/master/matlab/figtreeKCenterClustering.m}, the complexity for the farthest-point clustering can be reduced into $\O(n\log\kappa)$.

\begin{algorithm} 
\caption{The farthest point clustering} 
\label{alg:FarthestPointClustering} 
\begin{algorithmic}[1] 
    \REQUIRE $X = \left\{ x_i \mid_{i \in [n]} \right\}$ is a set of $n$ input data points, and $\kappa$ is the predefined number of clusters for the farthest-point clustering.
    \ENSURE A set of clustering centers $C = \left\{ c_i \mid_{i \in [\kappa]} \right\}$, and cluster indices for $x_i \mid_{i \in [n]}$.
    \STATE Initialize $C \leftarrow \varnothing$.
     \STATE $c_1 \leftarrow$ a random data point $x \in X$.
     \STATE $C \leftarrow c_1$.
     \STATE $i \leftarrow 1$.
     \WHILE{$i < \kappa$ and $n-i > 0$}
     	\STATE $i \leftarrow i + 1$.
     	\STATE $c_i \leftarrow \max_{x \in X} {\min_{c \in C} \norm{x - c}}$. $\hspace{6.3 em}$ \% (find the farthest point $x \in X$ to $C$).
	\STATE $C \leftarrow C \cup c_i$. $\hspace{14 em}$ \% (add the new cluster center into $C$).
     \ENDWHILE   	
     \STATE Each data point $x \in X$ is assigned to its nearest cluster center $c \in C$.
     
\end{algorithmic}
\end{algorithm}

\subsection{Clustering-based tree metric sampling}\label{sec:clustering_based_TM}

The clustering-based tree metric sampling \citep{le2019tree} is a practical fast approach to sample tree metric from input data points. Its main idea is to use a (fast) clustering method, e.g., the farthest-point clustering, to cluster input data points hierarchically to build a tree structure, as summarized in Algorithm~\ref{alg:clustering_based_TM}\footnote{Code is available at https://github.com/lttam/TreeWasserstein/blob/master/BuildTreeMetric\_HighDim\_V2.m}. As discussed in \citep{le2019tree}, one can use any clustering method for the clustering-based tree metric sampling. The farthest-point clustering is suggested due to its fast computation (see Section~\ref{sec:farthest_point_clustering}). The complexity of the clustering-based tree metric sampling for $n$ input data points where one uses the same number of clusters $\kappa$ for the farthest-point clustering and $H_{\Tt}$ for the predefined deepest level of tree $\Tt$, is $\O(n H_{\Tt} \log \kappa)$. Therefore, the clustering-based tree metric sampling is very fast for applications.

\paragraph{A cluster sensitivity problem.} As discussed in \citep{le2019tree}, for data points near a border of adjacent, but different clusters, they are close to each other but in different clusters. The fact that whether those data points are clustered in the same cluster or not, depends on an initialization of the farthest-point clustering. Therefore, by leveraging various clustering results, obtained with different initializations for the farthest-point clustering, e.g. as in our proposed flow-based alignment approaches: \FlowAlign~and \DepthAlign, one can reduce an affect of the cluster sensitivity problem.

\begin{algorithm}[H] 
\caption{Clustering-based tree metric (with the farthest-point clustering)} 
\label{alg:clustering_based_TM} 
\begin{algorithmic}[1] 
    \REQUIRE $X$ is a set of $m$ input data points, $\tilde{x}_p$ is a parent node for those input data points in $X$, $h$ is a current depth level, $H_{\Tt}$ is the predefined deepest level of tree $\Tt$, $\kappa$ is the predefined number of clusters for the farthest-point clustering.
    \ENSURE tree metric $\Tt$
     \IF{$m > 0$}
     	\IF{$h > 0$}	
		\STATE Node $\tilde{x}_{c}$ $\leftarrow$ a center of $X$, e.g., the mean data point of $X$.
		\STATE Length of edge $(\tilde{x}_p, \tilde{x}_{c})$ $\leftarrow$ distance $(\tilde{x}_p, \tilde{x}_{c})$.
	\ELSE
		\STATE Node $\tilde{x}_{c} \leftarrow \tilde{x}_p$.
	\ENDIF
	\IF{$m > 1$ and $h < H_{\Tt}$}
     		\STATE Run the farthest-point clustering for $X$ into $\kappa$ clusters $X_i \mid_{i \in [\kappa]}$.
		\FOR{\textbf{each} cluster $X_i \mid_{i \in [\kappa]}$} 
			\STATE Recursive the clustering-based tree metric for input data points in set $X_i$, a parent node $\tilde{x}_{c}$, a current depth level $(h + 1)$ with the predefined deepest level $H_{\Tt}$ for tree $\Tt$, and the predefined number of clusters $\kappa$ for the farthest-point clustering).
		\ENDFOR
	\ENDIF
     \ENDIF	
\end{algorithmic}
\end{algorithm}

\subsection{Tree metric}
We recall the definition of tree metric in~\citep{semple2003phylogenetics} (\S7, p.145--182).
\begin{definition}
A metric $d:\Omega\times\Omega\rightarrow \RR_{+}$ is a tree metric on a finite set $\Omega$ if there exists a tree $\Tt$ with non-negative edge lengths such that all elements of $\Omega$ are nodes in $\Tt$, and for $x, z \in \Omega$, $d(x, z)$ equals to the length of the (unique) path in $\Tt$ between $x$ and $z$.
\end{definition}

\subsection{$F_{\beta}$ measure for clustering evaluation}
We summarize the $F_{\beta}$ measure for clustering evaluation as in \citep{manning2008introduction} where $\beta$ is chosen as in ~\citep{le2015unsupervised}. The main idea is that a pair of data points is assigned to the same cluster if and only if they are in the same class and otherwise. We have some following quantities:
\begin{itemize} 
\item TP: the number of a true positive decisions which assign a pair of data points in the same class to the same cluster.
\item TN: the number of a true negative decisions which assign a pair of data points in the different classes to the different clusters.
\item FP: the number of a false positive decisions which assign a pair of data points of different classes to the same cluster.
\item FN: the number of a false negative decisions which assign a pair of data points of the same class to different clusters. 
\end{itemize}

Consequently, we have the precision 
\begin{equation}\label{eq:precision}
\mathbf{P} = \frac{\text{TP}}{\text{TP} + \text{FP}},
\end{equation}
and recall 
\begin{equation}\label{eq:recall}
\mathbf{R} = \frac{\text{TP}}{\text{TP} + \text{FN}}.
\end{equation}

Note that we usually have many more pairs of data points in different classes than in the same class in clustering. Therefore, we need to penalize false negative error more strongly than false positive error. $\mathbf{F}_\beta$ measure can take into account of this idea by using a scalar $\beta > 1$, defined as follow:
\begin{equation}\label{eq:Fbeta}
\mathbf{F}_{\beta} = \frac{\left(\beta^2 + 1 \right)\mathbf{PR}}{\beta^2\mathbf{P} + \mathbf{R}}.
\end{equation}

Following~\citep{le2015unsupervised}, we plug Equation~\eqref{eq:precision}, and Equation~\eqref{eq:recall} into Equation~\eqref{eq:Fbeta}, and observe that $\mathbf{F}_{\beta}$ penalizes false negative error ${\beta}^2$ times more than false positive error. Then, we can set 
\begin{equation}
\beta = \sqrt{\frac{| \texttt{D} |}{| \texttt{S} |}},
\end{equation}
where $| \cdot |$ denotes a cardinality of a set let; $\texttt{D}$, $\texttt{S}$ are sets of pairs of data points in different and same classes respectively. 

\subsection{More information about datasets}\label{sec:datasets}

\paragraph{Quantum chemistry.} One can download \texttt{qm7} dataset from: \url{http://quantum-machine.org/datasets/}. We emphasize that for simplicity, we only use the Cartesian coordinate of each atom ($\RR^3$) in the molecules. We do \textit{not} use the atomic nuclear charge for each molecule for the atomization energy prediction task as used in experiments of~\cite{peyre2016gromov, rupp2012fast}.

There are $7165$ molecules and each molecule has no more than $23$ atoms in \texttt{qm7} dataset.

\paragraph{Document classification with non-registered word embeddings.} One can download document datasets, e.g., \texttt{TWITTER, RECIPE, CLASSIC, AMAZON} datasets from: \url{https://github.com/mkusner/wmd}.

After preprocessing, there are 
\begin{itemize}
\item $3108$ documents in $3$ classes where each document length is not more than $29$ in \texttt{TWITTER} dataset, 
\item $4370$ documents in $15$ classes where each document length is not more than $628$ in \texttt{RECIPE} dataset, 
\item $7093$ documents in $4$ classes where each document length is not more than $348$ in \texttt{CLASSIC} dataset,
\item $8000$ documents in $4$ classes where each document length is not more than $4592$ in \texttt{AMAZON} dataset. 
\end{itemize}

\section{Some further discussions}\label{sec:further_discussions}

\paragraph{Network flow.} In combinatorial optimization, network flow is a class of computational problems in which the input is a graph with capacities on its edges~\citep{ahuja1988network}. The minimum-cost flow problem is one of popular classes of network flow problems. Especially, optimal transport (OT) for probability measures whose supports are in the same space, can be regarded as one of instances of the minimum-cost flow problems, and one can use the network simplex algorithm to solve it. However, for GW, the supports of input probability measures are in different spaces. Therefore, one may not use algorithms for minimum-cost flow problems, e.g., network simplex, to optimize the alignment in GW problem with tree metrics (Equation (1) in the main text) where supports of input probability measures are in different tree metric spaces.

Recall that our proposed flow-based alignment approaches (i.e., \FlowAlign~and \DepthAlign) for probability measures in different tree metric spaces is based on matching both \textit{flows} from a root to each support in the probability measure, and root alignment for the corresponding tree structures. Thus, one should distinguish between our proposed flow-based alignment approaches in \FlowAlign~and \DepthAlign~for probability measures in different tree metric spaces, and algorithms for minimum-cost flow problems. Note that the flows of our flow-based representation shares the same spirit with the flows modeled in the proof for the closed-form computation of tree-Wasserstein distance \citep{le2019tree} (\S3). 

\paragraph{Tree metric sampling.} Our goal is \textbf{\textit{not}} to approximate the GW distance between probability measures whose supports are in the Euclidean space (i.e., the ground metric is Euclidean metric), but rather to sample tree metrics for each space of supports, and then use those random sampled tree metrics as ground metrics for supports of input probability measures in GW, similar to tree-sliced-Wasserstein~\citep{le2019tree}.

Similar to the case of one-dimensional projections for sliced Wasserstein, or sliced GW, which do not give good properties from a distortion point of view, but remain useful for sliced Wasserstein or sliced GW in applications, we believe that tree metrics with a large distortion can be \textit{useful}, similar to the case of tree-sliced-Wasserstein in practical applications.

\paragraph{A correction of the binomial expansion trick in \citep{vayer2019sliced}.} There is a typo in the binomial expansion trick in \citep{vayer2019sliced}. We correct it as follows:
\begin{align}
\sum_{i,j}\left( (x_i - x_j)^2 - (y_{\sigma_i} - y_{\sigma_j})^2\right)^2 &= 2n\left(\sum_i x_i^4\right) - 8\left(\sum_i x_i^3\right)\left( \sum_i x_i\right) + 6 \left(\sum_i x_i^2 \right)^2  \nonumber \\
& \qquad + 2n\left(\sum_i y_i^4\right) - 8\left(\sum_i y_i^3\right)\left( \sum_i y_i\right) + 6 \left(\sum_i y_i^2 \right)^2 \nonumber \\
& \qquad -4\left(\sum_i x_i^2\right)\left(\sum_i y_i^2\right) - 4n\left(\sum_i x_i^2 y_{\sigma_i}^2 \right) \nonumber \\
& \qquad +8\left(\sum_i x_i \right)\left( \sum_i x_i y_{\sigma_i}^2\right) +8\left(\sum_i y_i \right)\left( \sum_i x_i^2 y_{\sigma_i}\right) \nonumber \\
& \qquad -8\left( \sum_i x_i y_{\sigma_i}\right)^2. \label{eq:bin_expansion}
\end{align}

The $\sigma$ in Equation~\eqref{eq:bin_expansion} is a permutation. Note that the binomial expansion trick can be applied for GW when one uses the squared $\ell_2$ loss and input probability measures have the same number of supports with uniform weights as considered in sliced GW \citep{vayer2019sliced}.

\section{Empirical relation for discrepancies for probability measures in different spaces}
\label{sec:empirical_relation}

We emphasize that the proposed \FlowTGW~and \DepthTGW~are two \textbf{\textit{novel}} discrepancies for probability measures in different tree metric spaces, and we do not try to mimic or approximate either the entropic GW or sliced GW.

In this section, we investigate an empirical relation between a pair of discrepancies, e.g., let denote those considered discrepancies as $d_{\alpha}$ and $d_{\beta}$. We carried out following experiments\footnote{The experimental setup is similar to that of \cite{le2019tree} for investigating an empirical relation between tree-sliced-Wasserstein and optimal transport with Euclidean ground metric.}: 

For a query point $q$, we denote $q_{\text{NN}}$ as the nearest neighbor of $q$ with respect to $d_{\alpha}$. Then, we investigate the frequency of rank order of $q_{\text{NN}}$ among nearest neighbor of $q$ with respect to $d_{\beta}$. For those experiments, we randomly split $90\%/10\%$ for training and test. Reported results are averaged over $1000$ runs.

We recall some following notations: FA for \FlowAlign, DA for \DepthAlign, SGW for sliced GW, EGW for entropic GW (only use entropic regularization for transport plan optimization, but exclude for computing entropic GW objective) and EGW$_0$ for standard entropic GW (use entropic regularization for both transport plan optimization and objective computation). We also recall that supports for probability measures are in low-dimensional spaces (dim=3) in \texttt{qm7} dataset; and in high-dimensional spaces (dim=300) in \texttt{TWITTER, RECIPE, CLASSIC, AMAZON} datasets. The number of supports for probability measures are small in \texttt{qm7} (\#supports $\le 23$), and \texttt{TWITTER} (\#supports $\le 29$) datasets; and are large in \texttt{RECIPE} (\#supports $\le 628$), \texttt{CLASSIC} (\#supports $\le 348$), and \texttt{AMAZON} (\#supports $\le 4592$) datasets.

\subsection{Empirical relation between FA and DA}\label{sec:relation_FTGW_DTGW}

We set $d_{\alpha}$ := FA and $d_{\beta}$ := DA. Figure~\ref{fg:rank_fixDTGW_FTGW} illustrates an empirical relation between FA and DA in \texttt{qm7, TWITTER, RECIPE, CLASSIC, AMAZON} datasets. We used DA with $1$ tree slice for \texttt{qm7}, $10$ tree slices for \texttt{TWITTER}, $5$ slices for \texttt{RECIPE}, $1$ tree slice for \texttt{CLASSIC}, and $1$ tree slice for \texttt{AMAZON}.

The empirical results show that FA agrees with some aspects of DA, especially when the number of support for probability measures is small (information about relative deep levels of supports is small), and the degree of agreement may increase when supports for probability measures are low-dimensional space (tree structure becomes simpler) as in \texttt{qm7}. From Figure~\ref{fg:rank_fixDTGW_FTGW}, we also observe that the degree of agreement decreases when the number of supports in datasets increases.

Recall that DA is a generalized version of FA which takes into account deep levels of supports in tree structures of tree metric spaces. When the number of supports for probability measures is large, the information about relative deep levels of supports is increased. Therefore, DA operates differently to FA, e.g., in \texttt{RECIPE} and \texttt{AMAZON} datasets. In addition, tree structure for high-dimensional spaces of supports (e.g., in \texttt{TWITTER} and \texttt{CLASSIC} datasets) is usually more complex than that of low-dimensional space of supports (e.g., in \texttt{qm7} dataset). Thus, DA behaves more similar to FA in \texttt{qm7} dataset than in \texttt{TWITTER} and \texttt{CLASSIC} datasets.

\begin{figure}
  \begin{center}
    \includegraphics[width=0.9\textwidth]{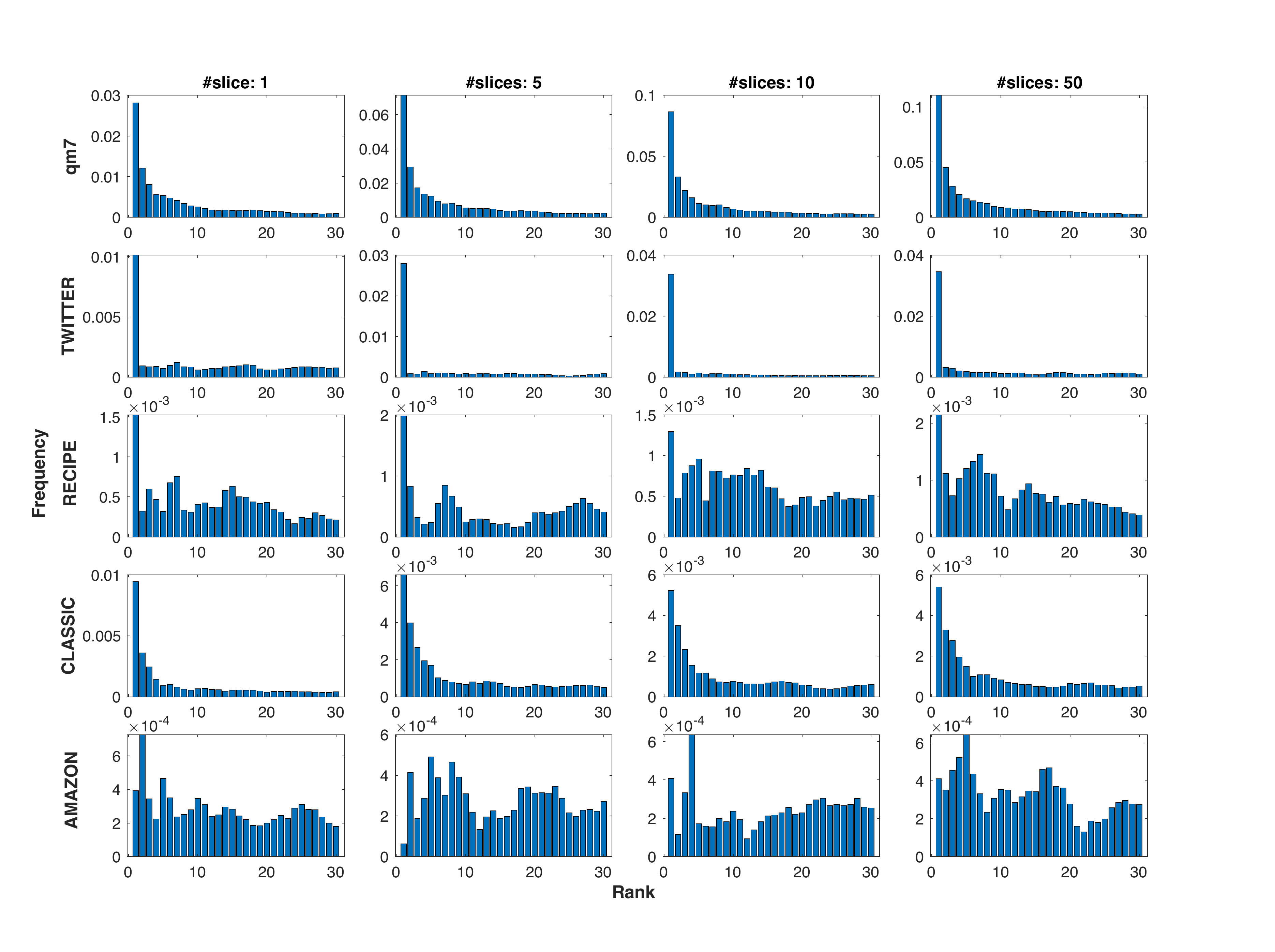}
  \end{center}
  \caption{Empirical relation: $d_{\alpha}$ := FA and $d_{\beta}$ := DA. We used DA with $1$ tree slice for \texttt{qm7}, $10$ tree slices for \texttt{TWITTER}, $5$ tree slices for \texttt{RECIPE}, $1$ tree slice for \texttt{CLASSIC}, and $1$ tree slice for \texttt{AMAZON}.}
  \label{fg:rank_fixDTGW_FTGW}
\end{figure}

\subsection{Empirical relation between FA and SGW}\label{sec:relation_FTGW_SGW}

We first set $d_{\alpha}$ := SGW and $d_{\beta}$ := FA (10 tree slices). For SGW in \texttt{CLASSIC, AMAZON} datasets, we only evaluate it until $10$ slices due to its slowness with a larger number of slices. We illustrate an empirical relation between FA and SGW in \texttt{qm7, TWITTER, RECIPE, CLASSIC, AMAZON} datasets in Figure~\ref{fg:rank_fixFTGW10_SGW}.

Secondly, we set $d_{\alpha}$ := FA and $d_{\beta}$ := SGW (10 slices). We illustrate another empirical relation between FA and SGW in \texttt{qm7, TWITTER, RECIPE, CLASSIC, AMAZON} datasets in Figure~\ref{fg:rank_fixSGW10_FTGW}.

The empirical results show that SGW and FA may agree with each other some aspects when supports are in low-dimensional spaces, e.g., in \texttt{qm7} dataset, but they become more different when supports are in high-dimensional spaces, e.g., in document datasets: \texttt{TWITTER, RECIPE, CLASSIC, AMAZON} datasets. Note that a one-dimensional space is a special case of tree metric (a tree metric is a chain). For supports in low-dimensional spaces, both projecting those supports in one-dimensional spaces and using tree metric sampling seem to be able to capture the structure of a distribution of supports at a certain level. However, for supports in high-dimensional spaces, projecting the supports in one-dimensional space limits its capacity to capture the structure of a distribution of supports~\citep{pmlr-v97-liutkus19a} while sampling tree metric can remedy this problem~\citep{le2019tree}. 

\begin{figure}
  \begin{center}
    \includegraphics[width=0.9\textwidth]{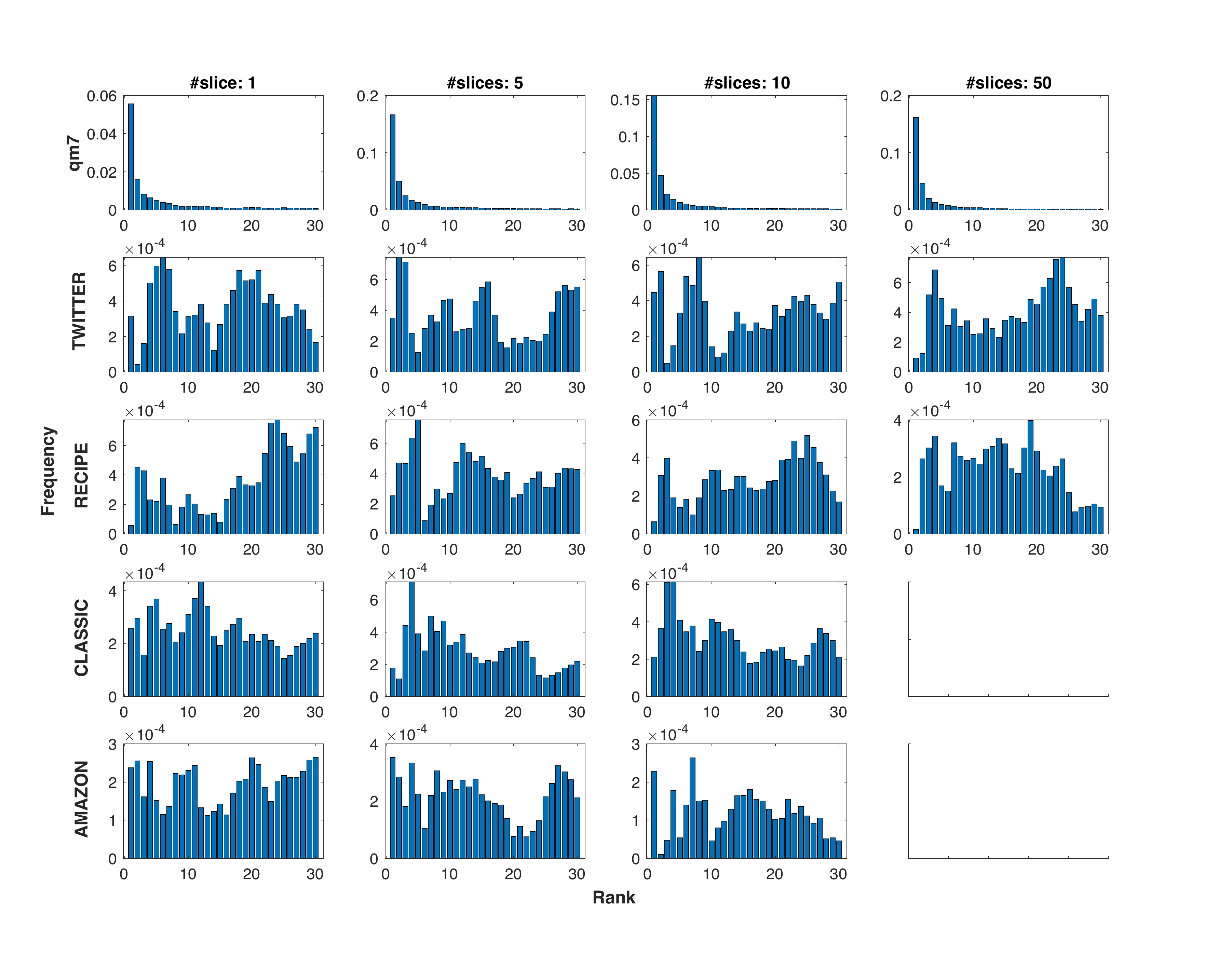}
  \end{center}
  \caption{Empirical relation: $d_{\alpha}$ := SGW and $d_{\beta}$ := FA (10 tree slices). For SGW in \texttt{CLASSIC, AMAZON} datasets, we only evaluate it until $10$ slices due to its slowness with a larger number of slices.}
  \label{fg:rank_fixFTGW10_SGW}
\end{figure}


\begin{figure}
  \begin{center}
    \includegraphics[width=0.9\textwidth]{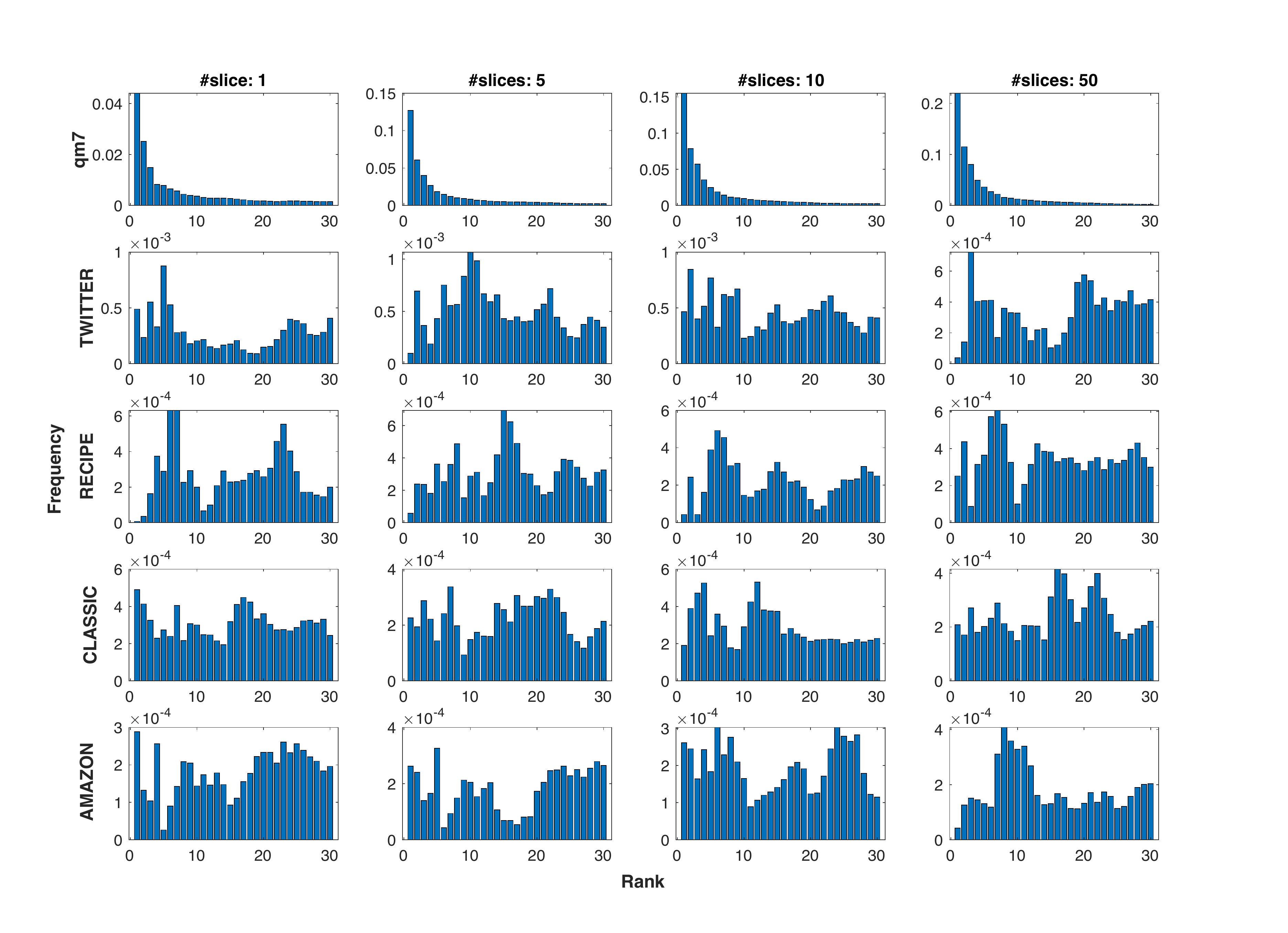}
  \end{center}
  \caption{Empirical relation: $d_{\alpha}$ := FA and $d_{\beta}$ := SGW (10 slices).}
  \label{fg:rank_fixSGW10_FTGW}
\end{figure}


\subsection{Some other empirical relations}

\subsubsection{Empirical relation between SGW and DA}

We set $d_{\alpha}$ := SGW and $d_{\beta}$ := DA. Figure~\ref{fg:rank_fixDTGW_SGW} illustrates an empirical relation between SGW and DA in \texttt{qm7, TWITTER, RECIPE, CLASSIC, AMAZON} datasets. We used DA with $1$ tree slice for \texttt{qm7}, $10$ tree slices for \texttt{TWITTER}, $5$ tree slices for \texttt{RECIPE}, $1$ tree slice for \texttt{CLASSIC}, and $1$ tree slice for \texttt{AMAZON}. For SGW in \texttt{CLASSIC, AMAZON} datasets, we only evaluate it until $10$ slices due to its slowness with a larger number of slices.

The empirical results show that SGW agrees with some aspects of DA when supports for probability measures are in a low-dimensional space (tree structure becomes simpler), as in \texttt{qm7} dataset. When supports for probability measure are in a low-dimensional space (e.g., in \texttt{qm7} dataset), both SGW and DA agree with some aspect of FA (see more discussions in Section~\ref{sec:relation_FTGW_DTGW}, and Section~\ref{sec:relation_FTGW_SGW}), or SGW and DA agrees with each other some aspects. However, when supports for probability measure are in high-dimensional spaces (e.g., in document datasets: \texttt{TWITTER, RECIPE, CLASSIC, AMAZON} datasets), they become different (similar to the empirical relation between SGW and FA as discussed in Section~\ref{sec:relation_FTGW_SGW}).

\begin{figure}
  \begin{center}
    \includegraphics[width=0.9\textwidth]{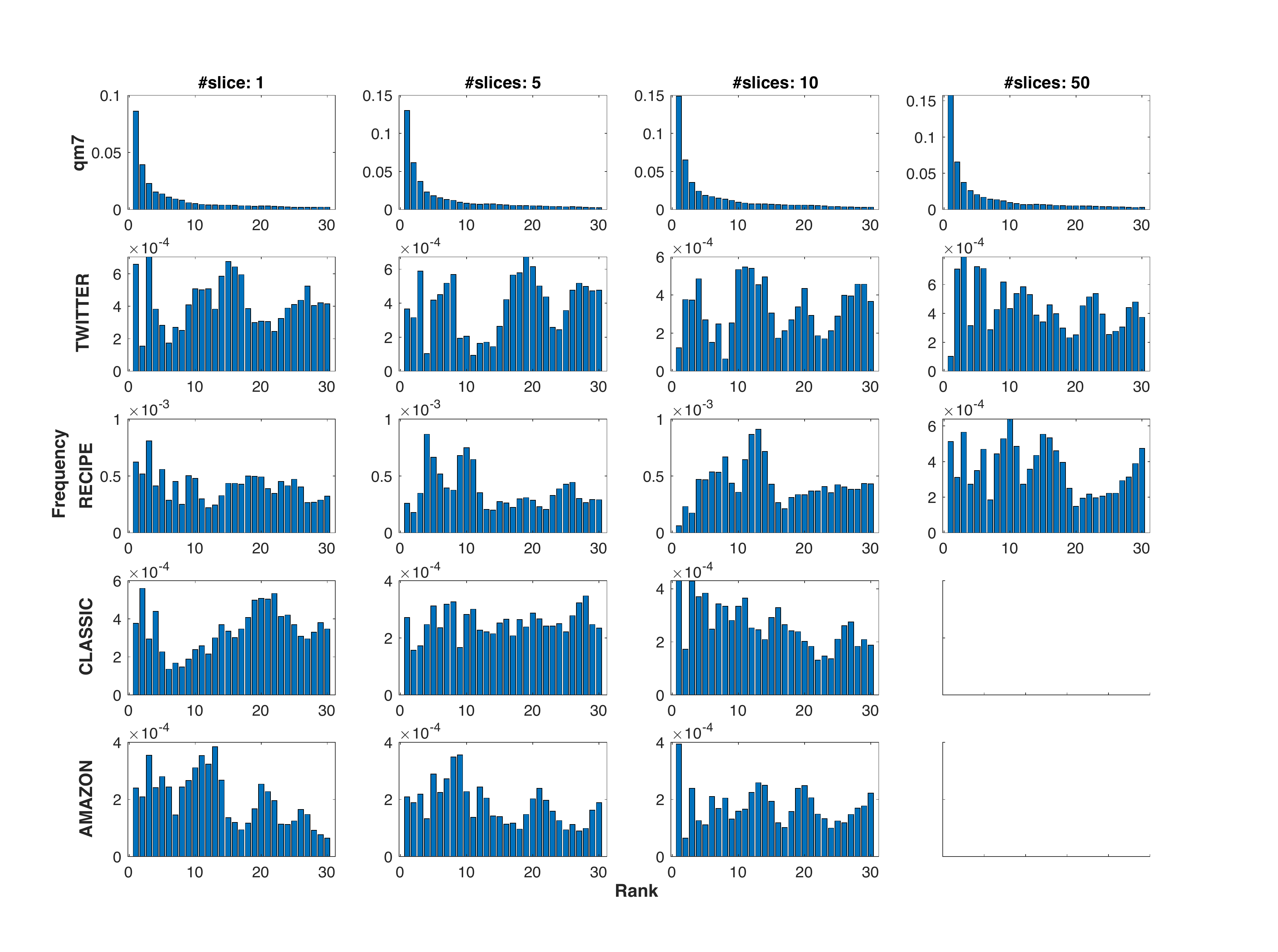}
  \end{center}
  \caption{Empirical relation: $d_{\alpha}$ := SGW and $d_{\beta}$ := DA.We used DA with $1$ tree slice for \texttt{qm7}, $10$ tree slices for \texttt{TWITTER}, $5$ tree slices for \texttt{RECIPE}, $1$ tree slice for \texttt{CLASSIC}, and $1$ tree slice for \texttt{AMAZON}. For SGW in \texttt{CLASSIC, AMAZON} datasets, we only evaluate it until $10$ slices due to its slowness with a larger number of slices.}
  \label{fg:rank_fixDTGW_SGW}
\end{figure}

\subsubsection{Empirical relation between FA/SGW and EGW$_0$/EGW}

For those experiments, the entropic regularization for EGW$_0$/EGW is set $5$ for \texttt{qm7} and \texttt{TWITTER} datasets, and $10$ for \texttt{RECIPE, CLASSIC, AMAZON} datasets. For SGW in \texttt{CLASSIC, AMAZON} datasets, we only evaluate it until $10$ slices due to its slowness with a larger number of slices.

\paragraph{Empirical relation between FA/SGW and EGW$_0$.} We first set $d_{\alpha}$ := FA and $d_{\beta}$ := EGW$_0$. We illustrate an empirical relation between FA and EGW$_0$ in \texttt{qm7, TWITTER, RECIPE, CLASSIC, AMAZON} datasets in Figure~\ref{fg:rank_fixEGW_FTGW}.

Secondly, we set $d_{\alpha}$ := SGW and $d_{\beta}$ := EGW$_0$. We illustrate an empirical relation between SGW and EGW$_0$ in \texttt{qm7, TWITTER, RECIPE, CLASSIC, AMAZON} datasets in Figure~\ref{fg:rank_fixEGW_SGW}.

\begin{figure}
  \begin{center}
    \includegraphics[width=0.9\textwidth]{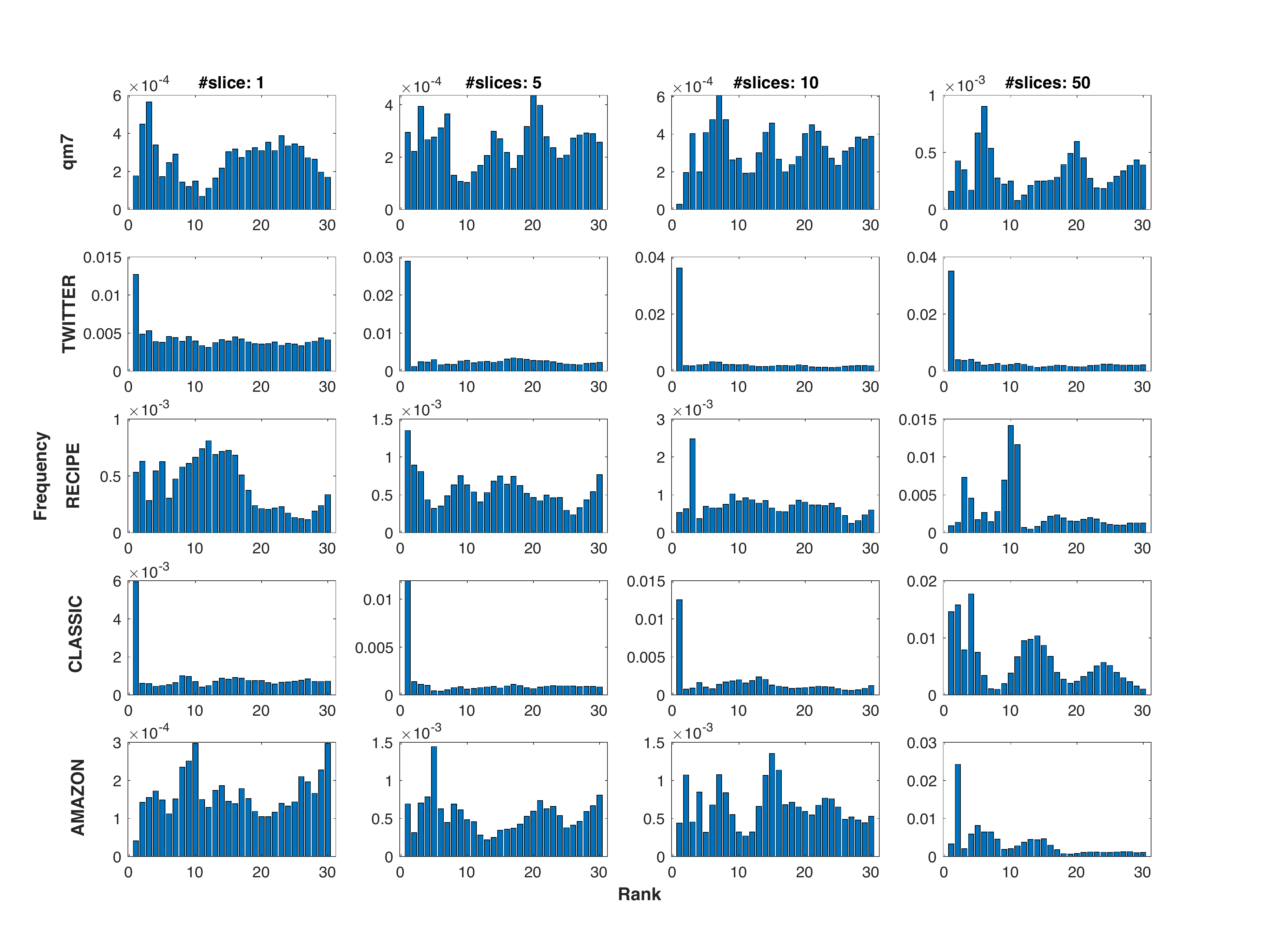}
  \end{center}
  \vspace{-10pt}
  \caption{Empirical relation: $d_{\alpha}$ := FA and $d_{\beta}$ := EGW$_0$. The entropic regularization for EGW$_0$ is set $5$ for \texttt{qm7} and \texttt{TWITTER} datasets, and $10$ for \texttt{RECIPE, CLASSIC, AMAZON} datasets.}
  \label{fg:rank_fixEGW_FTGW}
\end{figure}

\begin{figure}
  \begin{center}
    \includegraphics[width=0.9\textwidth]{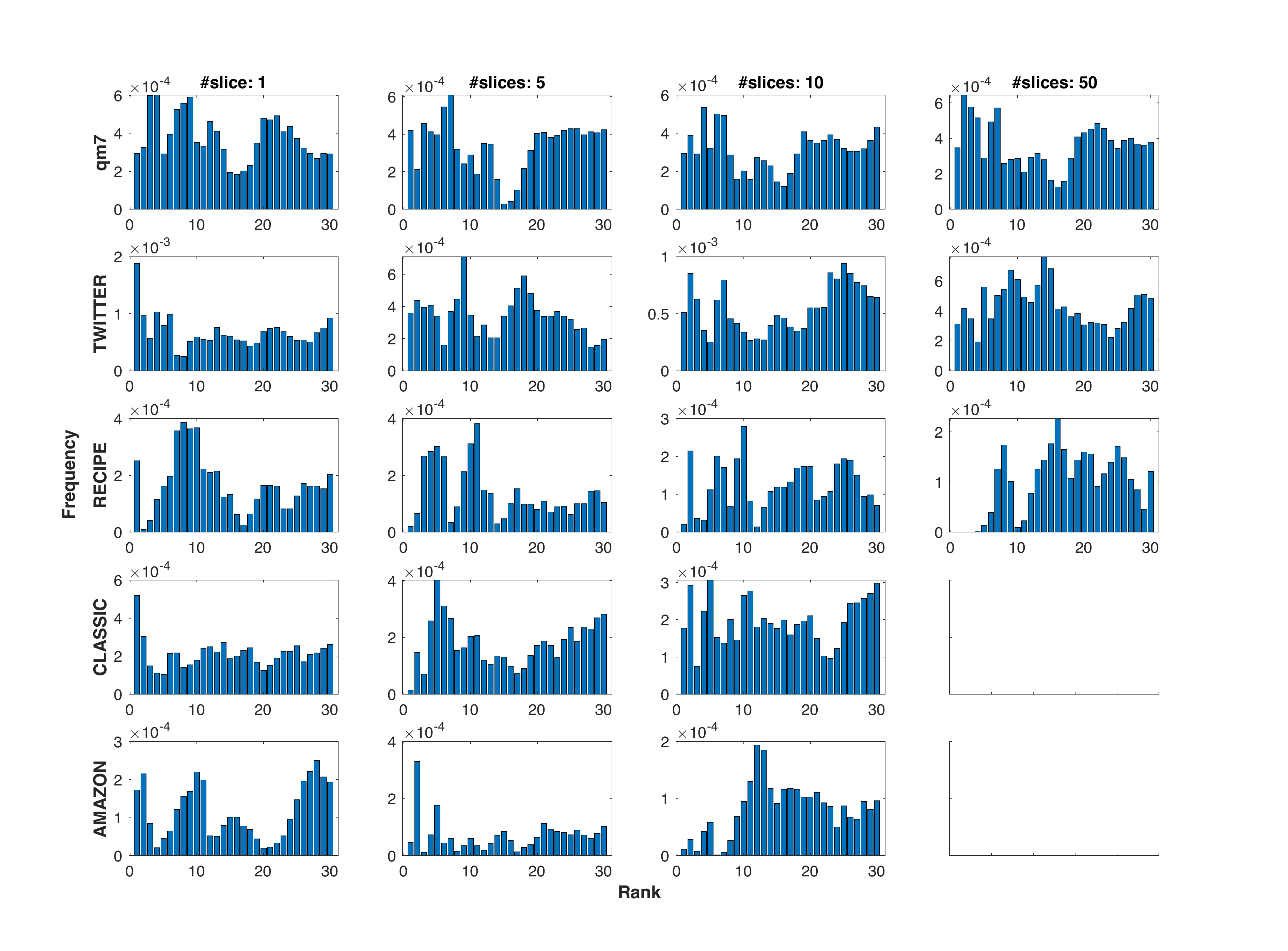}
  \end{center}
  \vspace{-10pt}
  \caption{Empirical relation: $d_{\alpha}$ := SGW and $d_{\beta}$ := EGW$_0$. The entropic regularization for EGW$_0$ is set $5$ for \texttt{qm7} and \texttt{TWITTER} datasets, and $10$ for \texttt{RECIPE, CLASSIC, AMAZON} datasets. For SGW in \texttt{CLASSIC, AMAZON} datasets, we only evaluate it until $10$ slices due to its slowness with a larger number of slices.}
  \label{fg:rank_fixEGW_SGW}
\end{figure}


\paragraph{Empirical relation between FA/SGW and EGW.} We first set $d_{\alpha}$ := FA and $d_{\beta}$ := EGW. We illustrate an empirical relation between FA and EGW in \texttt{qm7, TWITTER, RECIPE, CLASSIC, AMAZON} datasets in Figure~\ref{fg:rank_fixMyEGW_FTGW}.

Secondly, we set $d_{\alpha}$ := SGW and $d_{\beta}$ := EGW. We illustrate an empirical relation between SGW and EGW in \texttt{qm7, TWITTER, RECIPE, CLASSIC, AMAZON} datasets in Figure~\ref{fg:rank_fixMyEGW_SGW}.

\begin{figure}
  \begin{center}
    \includegraphics[width=0.9\textwidth]{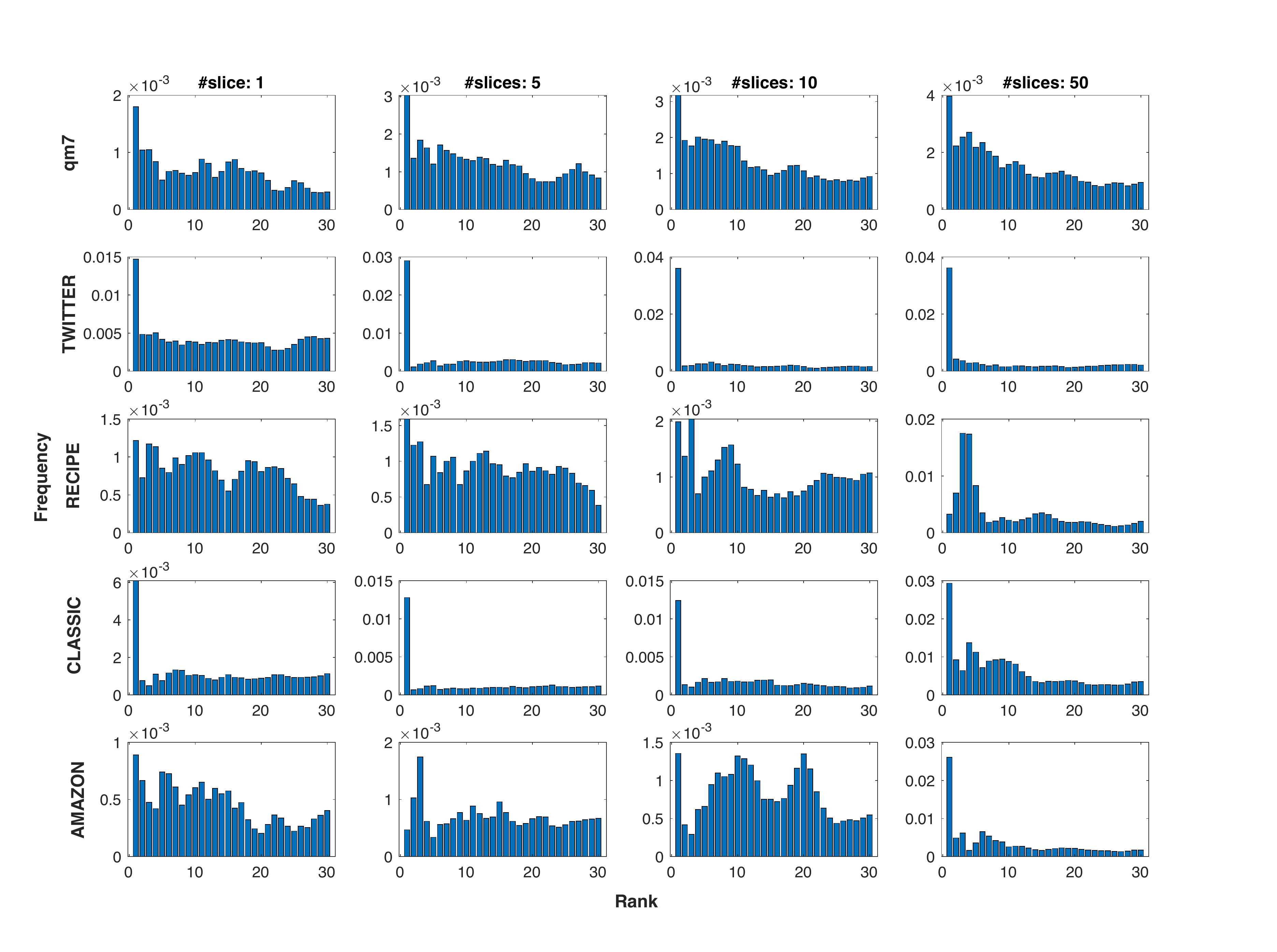}
  \end{center}
  \vspace{-10pt}
  \caption{Empirical relation: $d_{\alpha}$ := FA and $d_{\beta}$ := EGW. The entropic regularization for EGW is set $5$ for \texttt{qm7} and \texttt{TWITTER} datasets, and $10$ for \texttt{RECIPE, CLASSIC, AMAZON} datasets.}
  \label{fg:rank_fixMyEGW_FTGW}
\end{figure}

\begin{figure}
  \begin{center}
    \includegraphics[width=0.9\textwidth]{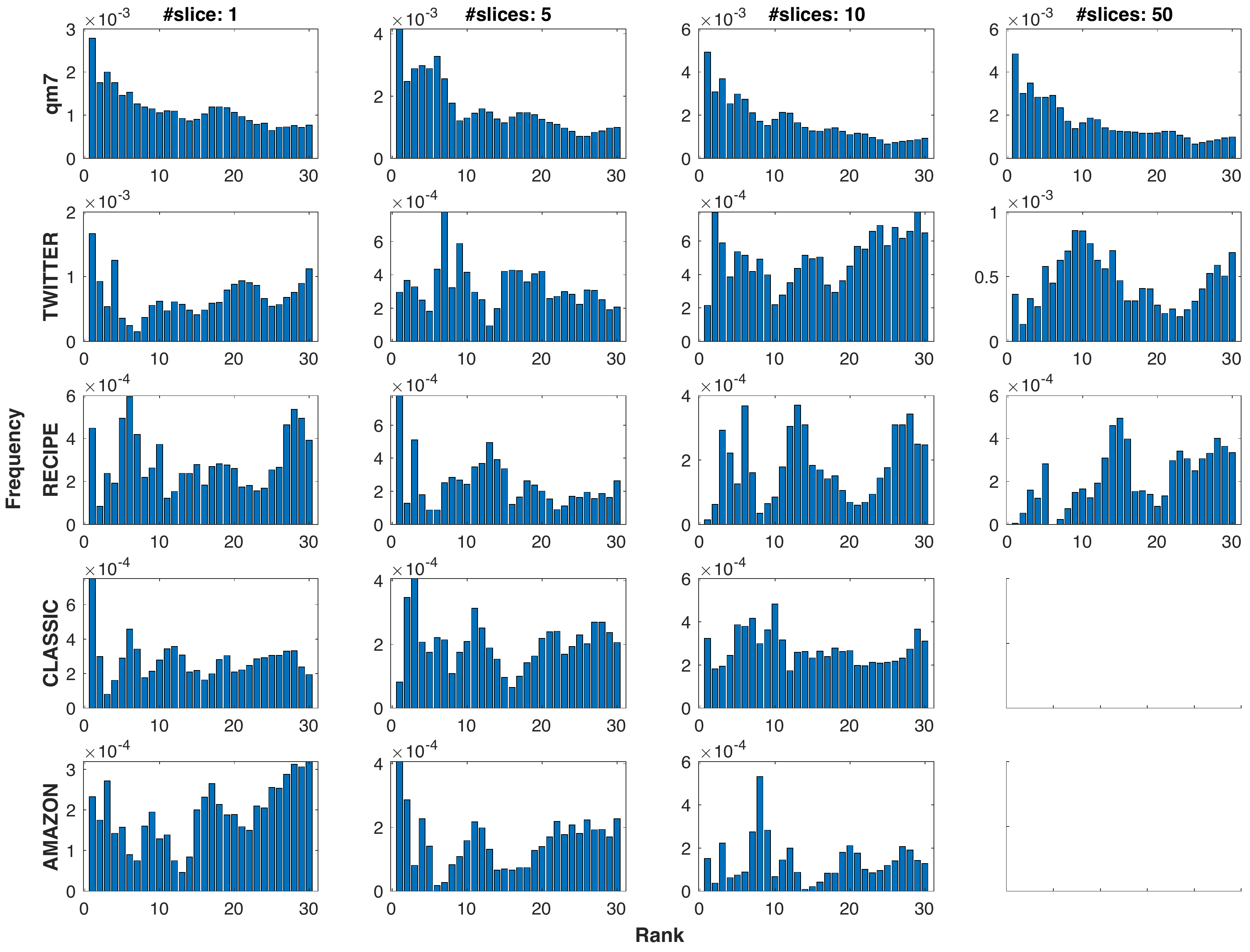}
  \end{center}
  \vspace{-10pt}
  \caption{Empirical relation: $d_{\alpha}$ := SGW and $d_{\beta}$ := EGW. The entropic regularization for EGW is set $5$ for \texttt{qm7} and \texttt{TWITTER} datasets, and $10$ for \texttt{RECIPE, CLASSIC, AMAZON} datasets. For SGW in \texttt{CLASSIC, AMAZON} datasets, we only evaluate it until $10$ slices due to its slowness with a larger number of slices.}
  \label{fg:rank_fixMyEGW_SGW}
\end{figure}

\paragraph{Discussions.} It seems that there is no much empirical relation between FA/SGW and EGW$_0$/EGW on \texttt{qm7, TWITTER, RECIPE, CLASSIC, AMAZON} datasets.


\end{document}